\documentclass[10pt,journal,compsoc]{IEEEtran}

\usepackage[utf8]{inputenc} 
\usepackage[T1]{fontenc}    
\usepackage[table,xcdraw]{xcolor}

\usepackage{url}            
\usepackage{booktabs}       
\usepackage{amsthm,amssymb,amsmath,amsfonts}    
\usepackage{mathtools}
\usepackage{nicefrac}       
\usepackage{microtype}      

\usepackage{bm}
\usepackage{mathtools}
\usepackage[ruled, lined, linesnumbered, commentsnumbered, longend]{algorithm2e}
\usepackage{multirow}
\usepackage{physics}
\usepackage{bbding} 
\usepackage{array}

\usepackage{caption}
\usepackage{fancyhdr}
\usepackage{amssymb}
\usepackage{pifont}
\usepackage{xr}
\usepackage{wrapfig}
\usepackage{setspace}
\usepackage{exscale}
\usepackage{relsize}
\usepackage{ragged2e}
\usepackage[pagebackref=true,breaklinks=true,colorlinks,bookmarks=true]{hyperref}

\newcolumntype{P}[1]{>{\centering\arraybackslash}p{#1}}
\newtheorem{theorem}{Theorem}
\newtheorem{lemma}{Lemma}

\newtheorem{assumption}{Assumption}

\newtheorem{prop}{Proposition}
\newtheorem{manualtheoreminner}{Theorem}
\newenvironment{manualtheorem}[1]{%
  \renewcommand\themanualtheoreminner{#1}%
  \manualtheoreminner
}{\endmanualtheoreminner}

\DeclareMathOperator{\E}{\mathop{\mathbb{E}}}
\DeclareMathOperator*{\argmin}{argmin}

\def\*#1{\mathbf{#1}}
\def\$#1{\mathcal{#1}}
\def\^#1{\mathbb{#1}}
\def\H#1{\Hat{#1}}

\def\innerprod#1{\left\langle#1\right\rangle}

\newcommand{\etc}{\emph{etc}}
\newcommand{\eg}{\emph{e.g.}}
\newcommand{\ie}{i.e.}

\makeatletter
\newcommand*{\addFileDependency}[1]{
  \typeout{(#1)}
  \@addtofilelist{#1}
  \IfFileExists{#1}{}{\typeout{No file #1.}}
}
\newcommand{\algorithmfootnote}[2][\footnotesize]{%
  \let\old@algocf@finish\@algocf@finish
  \def\@algocf@finish{\old@algocf@finish
    \leavevmode\rlap{\begin{minipage}{\linewidth}
    #1#2
    \end{minipage}}%
  }%
}
\makeatother

\begin{document}
\bstctlcite{IEEEexample:BSTcontrol}
\title{Adan: Adaptive Nesterov  Momentum Algorithm for  Faster Optimizing Deep Models}

\author{Xingyu Xie, Pan Zhou, Huan Li, 
       Zhouchen Lin,~\IEEEmembership{Fellow,~IEEE}
        and~Shuicheng Yan,~\IEEEmembership{Fellow,~IEEE}
\IEEEcompsocitemizethanks{\IEEEcompsocthanksitem 
X. Xie and Z. Lin are with State Key Lab of General AI, School of Intelligence Science and Technology, Peking University, China.
Z. LIN is also with Institute for Artificial Intelligence, Peking University, and Pazhou Laboratory (Huangpu), Guangzhou, China. \textbf{E-mail: zlin@pku.edu.cn}. 

\IEEEcompsocthanksitem P. Zhou is with the School of Computing and Information Systems at Singapore Management University.
\IEEEcompsocthanksitem S. Yan was with Sea AI Lab, and is now with Skywork AI, Singapore.

\IEEEcompsocthanksitem H. Li is with the Institute of Robotics and Automatic Information Systems, College of Artificial Intelligence, Nankai University, China.
\IEEEcompsocthanksitem X. Xie and P. Zhou share equal contribution.
\IEEEcompsocthanksitem Co-corresponding Author: Zhouchen Lin and Shuicheng Yan.}
}

\markboth{}%
{Shell \MakeLowercase{\textit{et al.}}: Bare Demo of IEEEtran.cls for Computer Society Journals}

\IEEEtitleabstractindextext{%
\begin{abstract}
In deep learning, different kinds of deep networks typically need different optimizers, which have to be chosen after multiple trials, making the training process inefficient. 
To relieve this issue and consistently improve the model training speed across deep networks, we propose the ADAptive Nesterov momentum algorithm, Adan for short.
Adan first reformulates the vanilla Nesterov acceleration to develop a new Nesterov momentum estimation (NME) method, which avoids the extra overhead of  computing gradient at the extrapolation point. 
Then Adan adopts NME to estimate the gradient's first- and second-order moments in adaptive gradient algorithms for convergence acceleration.
Besides, we prove that Adan finds an $\epsilon$-approximate first-order stationary point within $\order{\epsilon^{-3.5}}$ stochastic gradient complexity on the non-convex stochastic problems (\eg~deep learning problems), matching the best-known lower bound.
Extensive experimental results show that Adan consistently surpasses the corresponding SoTA optimizers on  vision, language, and RL tasks and sets new SoTAs for many popular networks and frameworks, \eg~ResNet,  ConvNext, ViT, Swin, MAE, DETR, GPT-2, Transformer-XL, and BERT. 
More surprisingly, Adan can use half of the training cost (epochs) of SoTA optimizers  to achieve higher or comparable performance on ViT, GPT-2, MAE, \etc, and also shows great tolerance to a large range of minibatch size, \eg~from 1k to 32k. 
Code is released at \url{https://github.com/sail-sg/Adan}, and has been used in multiple popular deep learning frameworks or projects.
\justifying
\end{abstract}

}

\maketitle

%
\IEEEpeerreviewmaketitle

\section{Introduction}\label{introduction}
\IEEEPARstart{D}{eep}  neural networks (DNNs) have made remarkable success in many fields, \eg~computer vision~\cite{szegedy2015going,he2016deep,dosovitskiy2020image,liu2022convnet} 
and natural language processing~\cite {sainath2013deep,abdel2014convolutional}. 
A noticeable part of such success is contributed by the stochastic gradient-based optimizers, which find satisfactory solutions with high efficiency.  
Among current deep optimizers,   SGD~\cite{robbins1951stochastic,saad1998online}
is the earliest and also the most representative stochastic optimizer, with dominant popularity for its simplicity and effectiveness. 
It adopts a single common learning rate for all gradient coordinates but often suffers unsatisfactory convergence speed on sparse data or ill-conditioned problems. 
In recent years, adaptive gradient algorithms~\cite{duchi2011adaptive,rmsprop,reddi2018convergence,chen2018convergence,luo2018adaptive,liu2019variance,zhuang2020adabelief,heo2020adamp}  
have been proposed, which adjusts the learning rate for each gradient coordinate according to the current geometry curvature of the loss objective. 
These adaptive algorithms often offer a faster convergence speed than SGD in practice across many DNN frameworks.

However, none of the above optimizers can always stay undefeated among all its competitors across different DNN architectures and applications. 
For instance, for vanilla ResNets~\cite{he2016deep}, SGD often achieves better generalization performance than adaptive gradient algorithms such as Adam~\cite{kingma2014adam}, whereas on vision transformers (ViTs)~\cite{touvron2021training,liu2021swin,yu2021metaformer}, SGD often fails, and AdamW~\cite{loshchilov2018decoupled} is the dominant optimizer with higher and more stable performance.   
Moreover, these commonly used optimizers usually fail for large-batch training, which is a default setting of the prevalent distributed training.
Although there is some performance degradation, we still tend to choose the large-batch setting for large-scale deep learning training tasks due to the unaffordable training time.
For example, training the ViT-B with the batch size of 512 usually takes several days, but when the batch size comes to 32K, we may finish the training within three hours~\cite{liu2022towards}. 
Although some methods, \eg~LARS~\cite{you2017large} and LAMB~\cite{you2019large}, have been proposed to handle large batch sizes, their performance may varies significantly across DNN architectures. 
This performance inconsistency increases the training cost and engineering burden since one has to try various optimizers for different architectures or training settings. This paper aims at relieving this issue.

When we rethink the current adaptive gradient algorithms, we find that they mainly combine the moving average idea with the heavy ball acceleration technique to estimate the first- and second-order moments of the gradient~\cite{kingma2014adam,zhuang2020adabelief,loshchilov2018decoupled,you2019large,szhou2023win}.
However, previous studies~\cite{nesterov1983method,nesterov1988approach,nesterov2003introductory} have revealed that Nesterov acceleration can theoretically achieve a faster convergence speed than heavy ball acceleration, as it uses gradient at an extrapolation point of the current solution and sees a slight ``future".
The ability to see the ``future" may help optimizers better utilize the curve information of the dynamic training trajectory and show better robustness to DNN architectures.
Moreover, recent works~\cite{nado2021large,he2021large} have shown the potential of Nesterov acceleration for large-batch training.
Thus we are inspired to consider efficiently integrating Nesterov acceleration with adaptive  algorithms. 

\begin{table*}[t]
	\centering
	\caption{
		Comparison of different adaptive gradient  algorithms on nonconvex stochastic problems.  
		``Separated Reg.'' refers to whether the $\ell_2$ regularizer (weight decay) can be separated from the loss objective like AdamW. ``Complexity" denotes stochastic gradient  complexity to find an $\epsilon$-approximate first-order stationary point.  
		Adam-type methods~\cite{guo2021novel} includes Adam,  AdaMomentum~\cite{wang2021adapting},
	and 	AdaGrad~\cite{duchi2011adaptive},  AdaBound~\cite{luo2018adaptive} 
	and AMSGrad~\cite{reddi2018convergence},
	\etc. 
	AdamW has no available convergence result. For SAM~\cite{foret2020sharpness}, A-NIGT~\cite{cutkosky2020momentum} and Adam$^+$~\cite{liu2020adam}, we compare their adaptive versions.  $d$ is  the variable dimension.  The lower bound is proven in~\cite{arjevani2020second} and please see Sec.~\ref{sec:diss} in the supplementary for the discussion on why the lower bound is $\Omega\qty(\epsilon^{-3.5})$.
 }\label{tab:cvspeed}
\setlength{\tabcolsep}{7.0pt}  
\renewcommand{\arraystretch}{3.0}
{ \fontsize{8.3}{3}\selectfont{
\begin{tabular}{P{.08\textwidth}|P{.16\textwidth}|P{.08\textwidth}|P{.08\textwidth}|p{.1\textwidth}|p{.12\textwidth}|P{.1\textwidth}}
		\toprule
  \begin{tabular}[c]{@{}c@{}}Smoothness\\ Condition\end{tabular}
		 & {Optimizer} & \begin{tabular}[c]{@{}c@{}}Separated\\ Reg. \end{tabular}    &  \begin{tabular}[c]{@{}c@{}}Batch Size\\ Condition. \end{tabular}  & {Grad. Bound} & {Complexity} & {Lower Bound}  \\ \midrule 
	    &Adam-type~\cite{guo2021novel}  & \XSolidBrush & \XSolidBrush & $\ell_{\infty}\leq c_{\infty}$ & $\order{c_{\infty}^2 d \epsilon^{-4}}$ & $\Omega\qty(\epsilon^{-4})$  \\ 
	    &  RMSProp~\cite{rmsprop,zhou2018convergence}   & \XSolidBrush & \XSolidBrush & $\ell_{\infty}\leq c_\infty$ &  $\mathcal{O}\big(\sqrt{c_{\infty}}d\epsilon^{-4}\big)$ & $\Omega\qty(\epsilon^{-4})$ \\ 
		Lipschitz &AdamW~\cite{loshchilov2018decoupled} & \CheckmarkBold  & ---   & --- & --- &  --- \\ 
	    &Adabelief~\cite{zhuang2020adabelief} & \XSolidBrush & \XSolidBrush  & $\ell_{2}\leq c_2$ & $\mathcal{O}\big(c_2^6\epsilon^{-4}\big)$ & $\Omega\qty(\epsilon^{-4})$ \\ 
	    Gradient &  Padam~\cite{chen2021closing}   & \XSolidBrush & \XSolidBrush & $\ell_{\infty}\leq c_\infty$ &  $\mathcal{O}\big(\sqrt{c_{\infty}}d\epsilon^{-4}\big)$ & $\Omega\qty(\epsilon^{-4})$ \\ 
		&LAMB~\cite{you2019large}    & \XSolidBrush & $\order{\epsilon^{-4}}$ & $\ell_{2}\leq c_2$ &  $\mathcal{O}\big(c^2_2 d\epsilon^{-4}\big)$ & $\Omega\qty(\epsilon^{-4})$ \\ 
		&\textbf{Adan \scriptsize{(ours)}}   & \CheckmarkBold & \XSolidBrush  & $\ell_{\infty}\leq c_{\infty}$  & $\order{c^{2.5}_{\infty}\epsilon^{-4}}$ & $\Omega\qty( \epsilon^{-4})$ \\  \bottomrule
		\multirow{3}{*}{\begin{tabular}[c]{@{}c@{}}Lipschitz\\ \\ Hessian\end{tabular}}&A-NIGT~\cite{cutkosky2020momentum}  & \XSolidBrush  & \XSolidBrush & $\ell_{2}\leq c_2$ & $\order{\epsilon^{-3.5}\log\frac{c_2}{\epsilon}}$ & $\Omega\qty(\epsilon^{-3.5})$ \\
		 &Adam$^+$~\cite{liu2020adam} & \XSolidBrush  & $\order{\epsilon^{-1.625}}$ & $\ell_{2}\leq c_2$ & $\order{\epsilon^{-3.625}}$ & $\Omega\qty(\epsilon^{-3.5})$ \\ 
		&\textbf{Adan \scriptsize{(ours)}}   & \CheckmarkBold & \XSolidBrush & $\ell_{\infty}\leq c_{\infty}$ & $\order{{c_{\infty}^{1.25}}\epsilon^{-3.5}}$ & $\Omega\qty(\epsilon^{-3.5})$ 
		\\ \bottomrule   
	\end{tabular}}}
 \vspace{-2mm}
\end{table*}
{\textbf{The contributions of our work include:}} \textbf{1)} We propose an efficient DNN optimizer, named Adan. 
Adan develops a Nesterov momentum estimation method to estimate stable and accurate first- and second-order moments of the gradient in adaptive gradient algorithms for acceleration. 
\textbf{2)} Moreover, Adan enjoys a provably faster convergence speed than previous adaptive gradient algorithms such as Adam. 
\textbf{3)} Empirically, Adan  shows superior performance over the SoTA deep optimizers across vision, language, and reinforcement learning (RL) tasks. 
So it is possible that the effort on trying different optimizers for different deep network architectures can be greatly reduced.
Our \textit{detailed} contributions are highlighted below.

Firstly, we propose an efficient Nesterov-acceleration-induced deep learning optimizer termed Adan. 
Given a function $f$ and the current solution $\bm{\theta}_{k}$, Nesterov acceleration~\cite{nesterov1983method,nesterov1988approach,nesterov2003introductory} estimates the gradient ${\*g}_k = \nabla f(\bm{\theta}_k') $ at the extrapolation point $\bm{\theta}_k' =\bm{\theta}_{k} - \eta \qty(1-\beta_1) \*m_{k-1}$ with the learning rate $\eta$ and momentum coefficient $\beta_1\in(0,1)$, and updates the moving gradient average as $ \*m_k = \qty(1-\beta_1)\*m_{k-1} +  {\*g}_k$. 
Then it runs a step by $\bm{\theta}_{k+1} = \bm{\theta}_{k} - {\eta} \*m_k $.
Despite its theoretical advantages, the implementation of Nesterov acceleration in practice reveals several significant challenges:
1) The process requires estimating the gradient at a point, $\bm{\theta}_k'$, which is not the current parameter set but an extrapolation. This two-step operation involves storing the original parameters $\bm{\theta}_k$ and updating them to $\bm{\theta}_k'$ solely for the purpose of gradient computation. Such a mechanism increases complexity in implementation, adds computational overhead, and escalates memory demands within deep learning frameworks. This additional step interrupts the workflow, as the optimizer cannot proceed directly to forward propagation without first completing this parameter extrapolation, thereby complicating the training process; 2) Distributed training, essential for handling large models, splits optimizer states and model weights across multiple GPUs. The requirement to manually update model weights to reflect the extrapolated position introduces significant communication burdens. Effective synchronization of these updates and their corresponding gradients across various nodes is crucial but challenging, often leading to inefficiencies and potential delays.
To resolve the above incompatibility issues, we propose an alternative Nesterov momentum estimation (NME). 
We compute the gradient ${\*g}_k = \nabla f(\bm{\theta}_{k}) $ at the current solution $\bm{\theta}_k$, and estimate the moving gradient average  as $ \*m_k = \qty(1-\beta_1)\*m_{k-1} +  \*g_k'$, where $ \*g_k'={\*g}_k + \qty(1-\beta_1)\qty({\*g}_k - {\*g}_{k-1})$. 
Our NME is provably equivalent to the vanilla one yet can avoid the extra cost. 
Then by regarding $\*g_k'$ as the current stochastic gradient in adaptive gradient algorithms, \eg~Adam, we accordingly estimate the first- and second-moments as $ \*m_k = \qty(1-\beta_1)\*m_{k-1} +\beta_1\*g_k'$ and  $\*n_k = \qty(1-\beta_2)\*n_{k-1} + \beta_2 \qty(\*g_k')^2$, respectively.
Finally, we update $\bm{\theta}_{k+1} = \bm{\theta}_{k} - {\eta}\*m_k/(\sqrt{\*n_k} + \varepsilon) $. 
In this way, Adan enjoys the merit of Nesterov  acceleration, namely faster convergence speed and tolerance to large mini-batch size~\cite{lin2020accelerated}, which is verified in our experiments in Sec.~\ref{experiments}. 

Secondly, as shown in Table~\ref{tab:cvspeed}, we  theoretically justify the advantages of Adan over previous SoTA adaptive gradient algorithms on nonconvex stochastic problems.
\begin{itemize}
    \item Given the Lipschitz gradient condition, to find an $\epsilon$-approximate first-order stationary point,  Adan has the stochastic gradient complexity  $\order{c^{2.5}_{\infty}\epsilon^{-4}}$ which accords with the   lower bound $\Omega(\epsilon^{-4})$ (up to a constant factor)~\cite{arjevani2019lower}. 
This complexity is lower than $\order{c_2^6\epsilon^{-4}}$ of Adabelief~\cite{zhuang2020adabelief} and  $\order{c_2^2 d\epsilon^{-4}}$ of LAMB, especially on over-parameterized networks.
Specifically, for the $d$-dimensional gradient, compared with its $\ell_{2}$ norm $c_2$,  its $\ell_{\infty}$ norm $c_{\infty}$ is usually much smaller, and can be $\sqrt{d}\times$ smaller for the best case.  
Moreover, Adan's results still hold when the loss and the $\ell_2$ regularizer are separated, which could significantly benefit the generalization~\cite{touvron2021training} but the convergence analysis for Adam-type optimizers remains open.
\item Given the Lipschitz Hessian condition, Adan has a complexity $\order{c_{\infty}^{1.25}\epsilon^{-3.5}}$ which also matches the lower bound $\Omega(\epsilon^{-3.5})$ in~\cite{arjevani2020second}. 
This complexity is superior to $\order{\epsilon^{-3.5}\log\frac{c_2}{\epsilon}}$ of A-NIGT~\cite{cutkosky2020momentum} and also $\order{\epsilon^{-3.625}}$ of Adam$^+$~\cite{liu2020adam}. 
Indeed, Adam$^+ $ needs the minibatch size of order $\order{\epsilon^{-1.625}}$ which is prohibitive in practice.  
For other optimizers, \eg~Adam, their convergence has not been provided yet under the Lipschitz Hessian condition.
\end{itemize}

Finally, Adan simultaneously surpasses the corresponding SoTA optimizers across vision, language, and RL tasks, and establishes  new SoTAs for many networks and settings, \eg~ResNet,  ConvNext~\cite{liu2022convnet}, ViT~\cite{touvron2021training}, Swin~\cite{liu2021swin}, MAE~\cite{he2022masked}, LSTM~\cite{schmidhuber1997long}, Transformer-XL~\cite{dai2019transformer} and BERT~\cite{devlin2018bert}. 
More importantly, with half of the training cost (epochs) of SoTA optimizers, Adan can achieve higher or comparable performance on ViT, Swin, ResNet, \etc. 
Besides, Adan works well in a large range of minibatch sizes, \eg~from 1k to 32k on ViTs. 
Due to the consistent improvement for various architectures and settings, Adan is supported in several popular deep learning frameworks or projects including Timm~\cite{rw2019timm} (a library containing SoTA CV models), Optax from DeepMind~\cite{deepmind2020jax} and MMClassification~\cite{2020mmclassification} from  OpenMMLab, see our Github repo for more projects.

\section{Related Work}  
Current DNN optimizers can be grouped into two families: SGD and its accelerated variants, and  adaptive gradient algorithms.  
SGD computes stochastic gradient and updates the variable along the gradient direction.
Later, heavy-ball acceleration~\cite{polyak1964some}  movingly averages stochastic gradient in SGD for faster convergence. 
Nesterov acceleration~\cite{nesterov2003introductory}  runs a step along the moving gradient average and then computes the gradient at the new point to look ahead for correction. 
Typically, Nesterov acceleration converges faster both empirically and theoretically at least on convex problems, and also has superior generalization results on  DNNs~\cite{foret2020sharpness,kwon2021asam}.

Unlike SGD, adaptive gradient algorithms, \eg~AdaGrad~\cite{duchi2011adaptive}, RMSProp~\cite{rmsprop} and Adam, view the second momentum of gradient as a precontioner and also use moving gradient average to update the variable.  
Later, many variants have been proposed to estimate more accurate and stable first momentum of gradient or its second momentum, \eg~AMSGrad~\cite{reddi2018convergence}, 
Adabound~\cite{luo2018adaptive}, and  Adabelief~\cite{zhuang2020adabelief}. 
To avoid gradient collapse,  AdamP~\cite{heo2020adamp} proposes to clip gradient adaptively.  Radam~\cite{liu2019variance} reduces gradient variance  to stabilize training.
To improve generalization, AdamW~\cite{loshchilov2018decoupled} splits the objective and trivial regularization, and its effectiveness is validated  across many applications;  SAM~\cite{foret2020sharpness} and its variants~\cite{kwon2021asam,du2021efficient,liu2022towards} 
aim to find flat minima but need forward and backward twice per iteration. 
LARS~\cite{you2017large} and LAMB~\cite{you2019large} train DNNs with a large batch but suffer unsatisfactory performance on small batch.
\cite{xie2022adaptive} reveal the generalization and convergence
gap between Adam and SGD from the perspective of diffusion theory and propose the optimizers,
Adai, which accelerates the training and provably favors flat minima.
Padam~\cite{chen2021closing} provides a simple but effective way to improve the generalization performance of Adam by adjusting the second-order moment in Adam. 
The most related work to ours is NAdam~\cite{dozat2016incorporating}. 
It simplifies Nesterov acceleration to estimate the first moment of the gradient in Adam.
But its acceleration does not use any gradient from the extrapolation points and thus does not look ahead for correction. 
Moreover, there is no theoretical result to ensure its convergence. See more difference discussion in Sec.~\ref{algorithm}, especially for  Eqn.~\eqref{comparison}.

In addition to the optimization techniques that form the core focus of our work, it is pertinent to acknowledge the breadth of research dedicated to enhancing training efficiency across various domains. Notable among these is the domain of data augmentation, where techniques such as mixup have been proposed, which performs the training on convex combinations of pairs of examples and their labels~\cite{zhang2018mixup,yun2019cutmix}. This approach significantly enriches the training data without the need for additional data collection. Furthermore, innovative training strategies play a crucial role in the efficient training of compact deep neural networks. For instance, the concept of multi-way BP offers a more efficient gradient calculation mechanism~\cite{guo2020multi}. Additionally, the design of loss functions, as explored in works like Sphere Loss~\cite{wang2021sphere}, introduces novel approaches to learning discriminative features. Each of these areas contributes to the overarching goal of training efficiency, offering complementary avenues to optimization techniques. 

\section{Methodology}
In this work, we study the following  regularized nonconvex optimization problem: 
\begin{equation}\label{problem}
\min\nolimits_{\bm{\theta}}   F(\bm{\theta}) \coloneqq  \E_{\bm{\zeta}\sim \$D}\qty[f(\bm{\theta},\bm{\zeta})]  +  \frac{\lambda}{2} \norm{\bm{\theta}}_2^2 ,  
\end{equation}
where loss $f(\cdot,\cdot)$ is differentiable and possibly nonconvex, data $\bm{\zeta}$ is drawn from an unknown distribution $\$D$, $\bm{\theta}$ is learnable parameters, and $\norm{\cdot}$ is the classical $\ell_2$ norm.  
Here we consider the $\ell_2$ regularizer as it can improve generalization performance and is widely used in practice~\cite{loshchilov2018decoupled}.  
The formulation~\eqref{problem} encapsulates a large body of machine learning problems, \eg~network training problems, and least square regression.  
Below, we first introduce the key motivation of  Adan in Sec.~\ref{motivation}, and then give detailed algorithmic steps in Sec.~\ref{algorithm}.

\subsection{Preliminaries}\label{motivation} %
Adaptive gradient algorithms, Adam~\cite{kingma2014adam} and AdamW~\cite{loshchilov2018decoupled}, have become the default choice to train CNNs and ViTs.    
Unlike SGD which uses one learning rate for all gradient coordinates, adaptive algorithms adjust the learning rate for each gradient coordinate according to the current geometry curvature of the objective function, and thus converge faster.
Take RMSProp~\cite{rmsprop} and Adam~\cite{kingma2014adam} as examples. 
Given stochastic gradient estimator ${\*g}_{k}\coloneqq   \E_{\bm{\zeta}\sim \$D}[\nabla f(\bm{\theta}_k,\bm{\zeta})] + \bm{\xi}_k$, \eg~minibatch gradient, where $\bm{\xi}_k$ is the gradient noise, RMSProp updates the variable $\bm{\theta}$ as follows:
\[
\text{RMSProp:}
\left\{
\begin{aligned}
	& \*n_k = \qty(1-\beta)\*n_{k-1} + \beta {\*g}_k^2 \\
	&\bm{\theta}_{k+1} = \bm{\theta}_{k} - \bm{\eta}_k \circ {\*g}_k ,
\end{aligned}
\right. 
\]
where $\bm{\eta}_k \coloneqq  {\eta}/\qty(\sqrt{\*n_k} + \varepsilon ), \*m_0=\*g_0$,  $\*n_0=\*g^2_0$, the scalar  $\eta$ is the base learning rate, $\circ$ denotes the element-wise product, and the vector square and the vector-to-vector or scalar-to-vector root in this paper are both element-wise.

Based on RMSProp, Adam (for presentation convenience, we omit the de-bias term in adaptive gradient methods), as follows, replaces the estimated gradient ${\*g}_{k}$ with a moving average $\*m_k$ of all previous gradient  ${\*g}_{k}$.
\[
\text{Adam:}
\left\{
\begin{aligned}
	& \*m_k = \qty(1-{\color{orange}\beta_1})\*m_{k-1} +  {\color{orange}\beta_1} {\*g}_k \\
	& \*n_k = \qty(1-\beta_2)\*n_{k-1} + \beta_2 {\*g}_k^2 \\
	&\bm{\theta}_{k+1} = \bm{\theta}_{k} - \bm{\eta}_k \circ \*m_k ,
\end{aligned}
\right.
\]
By inspection, one can easily observe that the moving average idea in Adam is similar to the classical (stochastic) heavy-ball acceleration (HBA) technique~\cite{polyak1964some}:
\[
\text{HBA:}
\left\{
\begin{aligned}
	& {\*g}_k = \nabla f(\bm{\theta}_{k}) + \bm{\xi}_k \\
	& \*m_k = \qty(1-{\color{orange}\beta_1})\*m_{k-1} +  {\*g}_k\\
	&\bm{\theta}_{k+1} = \bm{\theta}_{k} - {\eta} \*m_k ,
\end{aligned}
\right.
\]
Both Adam and HBA share the spirit of moving gradient average, though HBA does not have the factor $\beta_1$ on the gradient $ {\*g}_k$.
That is, given one gradient coordinate, if its gradient directions  are more consistent along the optimization trajectory, Adam/HBA accumulates a larger gradient value in this direction and thus goes ahead for a bigger gradient step, which accelerates convergence.  

In addition to HBA,  Nesterov’s accelerated (stochastic) gradient descent (AGD)~\cite{nesterov1983method,nesterov1988approach,nesterov2003introductory} is another popular acceleration technique in the optimization community: 
\begin{equation}\label{eq:AGD-I}
\text{AGD:}\\
\left\{
\begin{aligned}
	& {\*g}_k = \nabla f(\bm{\theta}_{k} - \eta \qty(1-{\color{orange}\beta_1}) \*m_{k-1}) + \bm{\xi}_k \\
	& \*m_k = \qty(1-{\color{orange}\beta_1})\*m_{k-1} +  {\*g}_k \\
	&\bm{\theta}_{k+1} = \bm{\theta}_{k} - {\eta} \*m_k
\end{aligned}
\right.
\end{equation}
Unlike HBA, AGD uses the gradient at the extrapolation point
$\bm{\theta}_{k}'= \bm{\theta}_{k} - \eta \qty(1-{\color{orange}\beta_1})\*m_{k-1}$.
Hence when the adjacent iterates share consistent gradient directions, AGD sees a slight future to converge faster. 
Indeed, AGD theoretically converges faster than HBA and achieves optimal convergence rate among first-order optimization methods 
on the general smooth convex  problems~\cite{nesterov2003introductory}. 
It also relaxes the convergence conditions of HBA on the strongly convex problems~\cite{nesterov1988approach}.   
Meanwhile, since the over-parameterized DNNs have been observed/proved to have many convex-alike local basins~\cite{xie2017diverse,li2017convergence,charles2018stability,zhou2021LA,nguyen2021tight,nguyen2020global,xie2022optimization,xie2024sscnet},
AGD seems to be more suitable than HBA for DNNs. For large-batch training, \cite{nado2021large} showed that AGD has the potential to achieve comparable performance to some specifically designed optimizers, \eg~LARS and LAMB. 
With its  advantage and potential 
in convergence and large-batch training, we consider applying AGD to improve adaptive algorithms.

\subsection{Adaptive Nesterov  Momentum Algorithm}\label{algorithm} 

{\textbf{Main Iteration.}}  
We temporarily set $\lambda = 0$ in Eqn.~\eqref{problem}. 
As aforementioned,  AGD computes gradient at an extrapolation point $\bm{\theta}_k'$ instead of the current iterate $\bm{\theta}_k$, which however brings  extra computation and memory overhead for computing  $\bm{\theta}_k'$  and preserving both $\bm{\theta}_k$ and $\bm{\theta}_k'$.   
To solve the issue, Lemma~\ref{lem:equivalence} with proof in supplementary Sec.~\ref{sec:AGDII} reformulates AGD~\eqref{eq:AGD-I} into its equivalent but more DNN-efficient version. 

\begin{lemma}\label{lem:equivalence}
Assume $\E(\bm{\xi}_k) = \bm{0}$, $\operatorname{Cov}(\bm{\xi}_i, \bm{\xi}_j) =0$ for any $k,i,j>0$, $\Bar{\bm{\theta}}_k$ and $\Bar{\*m}_k$ be the iterate and momentum of the vanilla  AGD in Eqn.~\eqref{eq:AGD-I}, respectively. 
Let $\bm{\theta}_{k+1}\coloneqq \Bar{\bm{\theta}}_{k+1} - \eta \qty(1-{\color{orange}\beta_1}) \Bar{\*m}_{k}$ and ${\*m}_{k} \coloneqq \qty(1-{\color{orange}\beta_1})^2 \Bar{\*m}_{k-1} + \qty(2-{\color{orange}\beta_1})\qty(\nabla f(\bm{\theta}_{k})+\bm{\xi}_k)$.
The  vanilla AGD in Eqn.~\eqref{eq:AGD-I} becomes AGD-II:
\[
\text{AGD II:}
\left\{
\begin{aligned}
	& {\*g}_k = \E_{\bm{\zeta}\sim \$D}[\nabla f(\bm{\theta}_k,\bm{\zeta})]  + \bm{\xi}_k  \\
	& \*m_k = \qty(1-{\color{orange}\beta_1})\*m_{k-1} + \*g_k'\\
	&\bm{\theta}_{k+1} = \bm{\theta}_{k} - {\eta} \*m_k 
\end{aligned}
\right. ,
\]
where $\*g_k'\coloneqq{\*g}_k + \qty(1-{\color{orange}\beta_1})\qty( {\*g}_k -  {\*g}_{k-1})$. 
Moreover, if vanilla AGD in Eqn.~\eqref{eq:AGD-I} converges, so does AGD-II: $\E(\bm{\theta}_{\infty}) = \E(\Bar{\bm{\theta}}_{\infty})$.
\end{lemma}
The main idea in Lemma~\ref{lem:equivalence} is that we maintain $\qty(\bm{\theta}_{k} - \eta \qty(1-{\beta_1}) \*m_{k-1})$ rather than $\bm{\theta}_{k}$ in vanilla AGD at each iteration since there is no difference between them when the algorithm converges.
Like other adaptive optimizers, by regarding $\*g_k'$ as the current stochastic gradient and movingly averaging $\*g_k'$ to estimate the first- and second-moments of gradient, we obtain:
\[
\text{Vanilla Adan:}
\left\{
\begin{aligned}
	& \*m_k = \qty(1-{\color{orange}\beta_1})\*m_{k-1} +  {\color{orange}\beta_1}\*g_k'\\
	& \*n_k = \qty(1-\beta_3)\*n_{k-1} + \beta_3  \qty(\*g_k')^2 \\
	&\bm{\theta}_{k+1} = \bm{\theta}_{k} - \bm{\eta}_k \circ \*m_k  
\end{aligned}
\right. ,
\]
where $\*g_k'\coloneqq{\*g}_k + \qty(1-{\color{orange}\beta_1})\qty( {\*g}_k -  {\*g}_{k-1})$ and the vector square in the second line is element-wisely.
The main difference of Adan with Adam-type methods and Nadam~\cite{dozat2016incorporating} is that, as compared in Eqn.~\eqref{comparison}, the first-order moment $\*m_k$ of Adan is the average of $\{{\*g}_t + \qty(1-{\color{orange}\beta_1})\qty({\*g}_t - {\*g}_{t-1})\}_{t=1}^k$ while those of Adam-type and Nadam are the average of $\{{\*g}_t\}_{t=1}^k$. 
So is their second-order term $\*n_k$,
\begin{equation}\label{comparison}
\*m_k \!=\! \begin{cases}  \sum_{t = 0}^k {\color{orange}c_{k,t}} \qty[{\*g}_t + \qty(1-{\color{orange}\beta_1})\qty({\*g}_t - {\*g}_{t-1}) ], & \text{Adan}, \\  \sum_{t = 0}^k {\color{orange}c_{k,t}} {\*g}_t,   & \text{Adam},
		\\  \frac{\mu_{k+1}}{\mu_{k+1}'} \big(\sum_{t = 0}^k {\color{orange}c_{k,t}}{\*g}_t\big) + \frac{1-\mu_{k}}{\mu_k'} {\*g}_k,   & \text{Nadam},
\end{cases} 
\end{equation}
where ${\color{orange}c_{k,t}} =  {\color{orange}\beta_1} \qty(1-{\color{orange}\beta_1})^{k-t}$ for $t>0$ and ${\color{orange}c_{k,t}} = \qty(1-{\color{orange}\beta_1})^{k}$ for $t = 0$.
$\qty{\mu_t}_{t=1}^\infty$ is a predefined exponentially decaying sequence,    $\mu_k'=1-\prod_{t=1}^{k}\mu_t$. 
Nadam is more like Adam than Adan, as their $\*m_k$  averages the historical gradients instead of gradient differences in Adan.  
For the large $k$ (\ie~small $\mu_k$),  $\*m_k$ in Nadam and Adam are almost the same.
 
As shown in Eqn.~\eqref{comparison}, the moment $\*m_k$  in Adan 
consists of two terms, \ie~gradient term ${\*g}_t $ and gradient difference  term  $({\*g}_t - {\*g}_{t-1})$, which actually have different physic meanings. 
So here we decouple them for greater flexibility and also better trade-off between them. 
Specifically, we estimate:
\begin{equation}\label{mknew}
\begin{aligned}
\qty(\bm{\theta}_{k+1} \!-\! \bm{\theta}_{k})/\bm{\eta}_k  & = \sum_{t = 0}^k \qty[{\color{orange}c_{k,t}} {\*g}_t \!+\! \qty(1\!-\!{\color{blue}\beta_2}){\color{blue}c'_{k,t}} \qty({\*g}_t \!-\! {\*g}_{t-1}) ]\\
&= \*m_k+\qty(1-{\color{blue}\beta_2})\*v_k,
\end{aligned}
\end{equation}
where
 ${\color{blue}c'_{k,t}} = {\color{blue}\beta_2} \qty(1-{\color{blue}\beta_2})^{k-t}$ for $t>0$, ${\color{blue}c'_{k,t}} = \qty(1-{\color{blue}\beta_2})^{k}$ for $t=0$,  and, with a little abuse of notation on $\*m_k$, we let $\*m_k $ and $\*v_k $ be:
\[
\left\{
\begin{aligned}
& {\*m}_k = \qty(1-{\color{orange}\beta_1}){\*m}_{k-1} +  {\color{orange}\beta_1} {\*g}_k\\
& \*v_k = \qty(1-{\color{blue}\beta_2})\*v_{k-1} + {\color{blue}\beta_2}\qty({\*g}_k - {\*g}_{k-1})
\end{aligned}
\right. .
\]
This change for a flexible estimation does not impair convergence speed.  
As shown in Theorem~\ref{theorem1}, Adan's convergence complexity still matches the  best-known lower bound. 
We do not separate the gradients and their difference in the second-order moment  $\*n_k$, since $\E(\*n_k)$ contains the correlation term $\operatorname{Cov}(\*g_k,\*g_{k-1})\neq 0$ which may have statistical significance.

{\textbf{Decay Weight by Proximation.}}   
As observed in AdamW, 
decoupling the optimization objective and simple-type regularization (\eg~$\ell_2$ regularizer) can largely improve the generalization performance.
Here we follow this idea
but from a rigorous optimization perspective. 
Intuitively, at each iteration $\bm{\theta}_{k+1}  = \bm{\theta}_{k} - \bm{\eta}_k \circ \bar{\*m}_k$, we minimize the first-order approximation of $F(\cdot)$ at the point $\bm{\theta}_k$:
\[
  \bm{\theta}_{k+1} \! = \! \argmin_{\bm{\theta}} \qty(F(\bm{\theta}_{k})\! + \!\innerprod{\bar{\*m}_k , \bm{\theta}-\bm{\theta}_k} \! + \! \frac{1}{2\eta}\norm{\bm{\theta}-\bm{\theta}_k}_{\sqrt{\*n_k}}^2), 
\]
where $\norm{\*x}^2_{\sqrt{\*n_k}} \coloneqq \innerprod{\*x, \qty(\sqrt{\*n_k} + \varepsilon) \circ \*x}$ and $\bar{\*m}_k \coloneqq  \*m_k+\qty(1-{\color{blue}\beta_2})\*v_k$ is the first-order derivative of $F(\cdot)$ in some sense. 
Follow the idea of  proximal gradient descent~\cite{parikh2014proximal,zhuang2022understanding}, we decouple the $\ell_2$ regularizer from  $F(\cdot)$ and only linearize the loss function $f(\cdot)$: 
\begin{equation}\label{eq:wd-update}
\begin{aligned}
   \! \bm{\theta}_{k+1} & \!=\! \argmin_{\bm{\theta}}  \Big(F_k'(\bm{\theta})
	\! + \! \innerprod{\bar{\*m}_k, \bm{\theta}\! - \!\bm{\theta}_k}\! + \! \frac{1}{2\eta}\norm{\bm{\theta}\! - \!\bm{\theta}_k}_{\sqrt{\*n_k}}^2\Big)\\
 & = \frac{\bm{\theta}_{k} - \bm{\eta}_k \circ \bar{\*m}_k}{ 1+{\lambda_k} \eta},
\end{aligned}
\end{equation}
where $F_k'(\bm{\theta}) \coloneqq  \E_{\bm{\zeta}\sim \$D}\qty[f(\bm{\theta}_k,\bm{\zeta})]   +  \frac{\lambda_k}{2}\norm{\bm{\theta}}_{\sqrt{\*n_k}}^2$,
and $\lambda_k>0$ is the weight decay at the $k$-th iteration.
Interestingly, we can easily reveal the updating rule $\bm{\theta}_{k+1} = (1-\lambda \eta) \bm{\theta}_{k} - \bm{\eta}_k \circ \bar{\*m}_k$ of AdamW by using the first-order approximation of Eqn.~\eqref{eq:wd-update} around $\eta = 0$: 
 1) $(1+\lambda \eta )^{-1} = (1-\lambda \eta) + \order{\eta^2}$; 
 2) $\lambda \eta \bm{\eta}_k = \order{\eta^2}/\qty(\sqrt{\*n_k} + \varepsilon )$. 

One can find that the optimization objective of Separated Regularization 
at the $k$-th iteration is changed from the vanilla ``static" function $F(\cdot)$ in Eqn.~\eqref{problem} to a ``dynamic" function $F_k(\cdot)$ in Eqn.~\eqref{decoupleproblem}, which adaptively regularizes the coordinates with larger gradient  more:
\begin{equation}\label{decoupleproblem}
 F_k(\bm{\theta}) \coloneqq  \E_{\bm{\zeta}\sim \$D}\qty[f(\bm{\theta},\bm{\zeta})]   +  \frac{\lambda_k}{2}\norm{\bm{\theta}}_{\sqrt{\*n_k}}^2. 
\end{equation}
We summarize our Adan in Algorithm \ref{alg:Name}.
We reset the momentum term properly by the restart condition, a common trick to stabilize optimization and benefit convergence~\cite{li2022restarted,jin2018accelerated}. But to make Adan simple, in all experiments except Table~\ref{tab:restart}, we do not use this restart strategy although it can improve performance as shown in Table~\ref{tab:restart}.  

\begin{algorithm}[t]
	\SetAlgoLined
	\KwIn{ initialization $\bm{\theta}_0$, step size $\eta$,   
	momentum
	$({\color{orange}\beta_1}, {\color{blue}\beta_2}, \beta_3)\in [0,1]^3$, 
 stable parameter $\varepsilon>0$, 
	weight decay $\lambda_k>0$,  restart condition.}
	\KwOut{some average of $\qty{\bm{\theta}_{k}}_{k=1}^{K}$.}
	\While{$k<K$}{
		estimate the stochastic gradient  ${\*g}_{k}$ at $\bm{\theta}_{k}$\;
		$\*m_k = \qty(1-{\color{orange}\beta_1})\*m_{k-1} +  {\color{orange}\beta_1} {\*g}_k$\;
		$\*v_k = \qty(1-{\color{blue}\beta_2})\*v_{k-1} + {\color{blue}\beta_2}\qty({\*g}_k - {\*g}_{k-1})$\;
		$\*n_k = \qty(1\!-\!\beta_3)\*n_{k-1} \!+\! \beta_3  \qty[{\*g}_k \!+\! \qty(1\!-\!{\color{blue}\beta_2})\qty({\*g}_k \!-\! {\*g}_{k-1})]^2 $\;
		$\bm{\eta}_k =  {\eta}/\qty(\sqrt{\*n_k} + \varepsilon)$\;
  $\bm{\theta}_{k+1} \!=\! \qty( 1\!+\!\lambda_k \eta)^{-1}\qty[\bm{\theta}_{k} \!-\! \bm{\eta}_k \circ \qty(\*m_k+\qty(1-{\color{blue}\beta_2})\*v_k)]$\;
		\If{restart condition holds}{
			estimate stochastic gradient  ${\*g}_{0}$ at $\bm{\theta}_{k+1}$\;
		set $k=1$ and update $\bm{\theta}_{1}$  by Line 7;}
	}
	\caption{\textbf{Adan} (Adaptive Nesterov  Momentum Algorithm)}\label{alg:Name}
 \algorithmfootnote{we set $\*m_0 = {\*g}_{0}$ , $\*v_0=\bm{0}$, $\*v_1 = {\*g}_1 - {\*g}_{0}$, and $\*n_0 = {\*g}_0^2$.}
\end{algorithm}

\begin{table*}[th!]
	\caption{{Top-1 Acc. (\%) of ResNet  and ConvNext on ImageNet under the official settings.} 
		$*$ and $\diamond$ are from~\cite{wightman2021resnet,liu2022convnet}.
	}
	\label{tab:CNNs}
	\begin{minipage}[c]{.65\linewidth}
		\setlength{\tabcolsep}{8.0pt} 
		\renewcommand{\arraystretch}{3.0}
		{ \fontsize{8.3}{3}\selectfont{
				\begin{tabular}{l|ccc|ccc}
					\toprule
     & \multicolumn{3}{c|}{ResNet-50}
     & \multicolumn{3}{c}{ResNet-101}              \\ 
    Epoch       & \multicolumn{1}{c}{100}    & \multicolumn{1}{c}{200}                          &300    & \multicolumn{1}{c}{100}    & \multicolumn{1}{c}{200}   & 300 \\ \midrule
    
    SAM~\cite{foret2020sharpness}       & \multicolumn{1}{c}{77.3}& \multicolumn{1}{c}{{78.7}}&
    \multicolumn{1}{c|}{{79.4}}&
    \multicolumn{1}{c}{79.5}&
    \multicolumn{1}{c}{81.1 }&
    \multicolumn{1}{c}{81.6}\\  	
    SGD-M~\cite{nesterov1983method,nesterov1988approach,nesterov2003introductory}  & \multicolumn{1}{c}{{77.0}}& \multicolumn{1}{c}{{78.6}}&
					\multicolumn{1}{c|}{{79.3}}&
					\multicolumn{1}{c}{79.3}&
					\multicolumn{1}{c}{{81.0}}&
					\multicolumn{1}{c}{{81.4}}\\
     Adam~\cite{kingma2014adam}       & \multicolumn{1}{c}{{76.9}}& \multicolumn{1}{c}{{78.4}}&
        \multicolumn{1}{c|}{{78.8}}& \multicolumn{1}{c}{{78.4}}& \multicolumn{1}{c}{{80.2}}&
        \multicolumn{1}{c}{{80.6}}\\ 
    AdamW~\cite{loshchilov2018decoupled}  & \multicolumn{1}{c}{{77.0}}& \multicolumn{1}{c}{{78.9}}&
        \multicolumn{1}{c|}{{79.3}}&
        \multicolumn{1}{c}{78.9}&
        \multicolumn{1}{c}{{79.9}}&
        \multicolumn{1}{c}{{80.4}}\\  	
        LAMB~\cite{you2019large,wightman2021resnet}         & \multicolumn{1}{c}{{77.0}}
        & \multicolumn{1}{c}{79.2}
        & \multicolumn{1}{c|}{\;\;$79.8^*$}
        & \multicolumn{1}{c}{{79.4}}                       & \multicolumn{1}{c}{81.1}                       & \multicolumn{1}{c}{{\;\;$81.3^*$ }}                                           \\ 
        \textbf{Adan \scriptsize{(ours)}}         & \textbf{78.1}    & \textbf{79.7}      & \textbf{80.2}     & \textbf{79.9}                 &  \textbf{81.6}               & \multicolumn{1}{c}{\textbf{81.8}}       \\ \bottomrule
    \end{tabular}
		}}
	\end{minipage}%
	\begin{minipage}[c]{.25\linewidth}
		\setlength{\tabcolsep}{10pt} 
		\renewcommand{\arraystretch}{2.7}
		{ \fontsize{8.3}{3}\selectfont{
				\begin{tabular}{l|cc}
					\toprule
					& \multicolumn{2}{c}{ConvNext Tiny}        \\ 
					Epoch       & 150   & 300     \\ \midrule
					AdamW~\cite{loshchilov2018decoupled,liu2022convnet}      &  81.2 & \;\;82.1$^{\diamond}$   \\
					\textbf{Adan \scriptsize{(ours)}}  & \textbf{81.7} & \textbf{82.4}  \\ \bottomrule
					&  \multicolumn{2}{c}{ConvNext Small}                                                                                                                    \\ 
					Epoch       & 150   & 300  \\ \midrule
					AdamW~\cite{loshchilov2018decoupled,liu2022convnet}       & 82.2  & \;\;83.1$^{\diamond}$  \\
					\textbf{Adan \scriptsize{(ours)}}   & \textbf{82.5} & \textbf{83.3} \\ \bottomrule
				\end{tabular}		
		}}
	\end{minipage}
\end{table*}

\begin{table*}[t!]
	\begin{center}
		\caption{Top-1 Acc. (\%) of ResNet-18 under the official  setting in~\cite{he2016deep}.  $*$ are  reported in~\cite{zhuang2020adabelief}.}
		\label{tab:res-18}
		\setlength{\tabcolsep}{3.5pt} 
		\renewcommand{\arraystretch}{3.5}
		{\fontsize{8.3}{3}\selectfont{
    \begin{tabular}{c|ccccccccc}
        \toprule
        \textbf{Adan}            & SGD~\cite{robbins1951stochastic}  & Nadam~\cite{dozat2016incorporating}   &   AdaBound~\cite{luo2018adaptive}  & Adam~\cite{kingma2014adam}    &  Radam~\cite{liu2019variance}   &
        Padam~\cite{chen2021closing}     &LAMB~\cite{you2019large}    & AdamW~\cite{loshchilov2018decoupled}  &   AdaBlief~\cite{zhuang2020adabelief}                    \\ \midrule
        \textbf{70.90}  & 70.23$^*$ & 68.82  & 68.13$^*$   & 63.79$^*$  & 67.62$^*$ & 70.07 & 68.46 & 67.93$^*$ & 70.08$^*$  \\ \bottomrule
    \end{tabular}
		}}
	\end{center}
\end{table*} 

\begin{table*}[t!]
	\begin{center}
		\caption{Top-1 Acc. (\%) of ResNet-34 under the  setting from AdaBlief~\cite{zhuang2020adabelief} on CIFAR-10 dataset.}
		\label{tab:cifar}
		\setlength{\tabcolsep}{4pt} 
		\renewcommand{\arraystretch}{3.5}
		{\fontsize{8.3}{3}\selectfont{
    \begin{tabular}{c|ccccccccc}
        \toprule
        \textbf{Adan}            & SGD~\cite{robbins1951stochastic}  & Nadam~\cite{dozat2016incorporating}   &   AdaBound~\cite{luo2018adaptive}  & Adam~\cite{kingma2014adam}    &  Radam~\cite{liu2019variance}   & LAMB~\cite{you2019large}    & AdamW~\cite{loshchilov2018decoupled}  &   AdaBlief~\cite{zhuang2020adabelief}   & Yogi~\cite{zaheer2018adaptive}                   \\ \midrule
        \textbf{95.07}  & 94.65 & 92.98  & \textbf{94.69}   & 93.17  & 94.39 & 94.01 & 94.28 & 94.11 & 94.52  \\ \bottomrule
    \end{tabular}
		}}
	\end{center}
\end{table*}

\section{Convergence Analysis}\label{sec:convergence}
For analysis, we make several mild assumptions used in many works, \eg~\cite{guo2021novel,zhou2018convergence, chen2021closing, foret2020sharpness,cutkosky2020momentum,liu2020adam,kwon2021asam,wang2021adapting,zhou2024towards}
{
\begin{assumption}[$L$-smoothness]\label{asm:Lsmooth}
The function $f(\cdot, \cdot)$ is $L$-smooth w.r.t. the parameters. Denote $F(\*x) \coloneqq \E_{\bm{\zeta}}[f(\*x,\bm{\zeta})]$. We have:
	\[
	\norm{\nabla F(\*x)  - \nabla F(\*y) } \leq L \norm{\*x - \*y},\qquad \forall \*x, ~\*y
	\]
\end{assumption}

\begin{assumption}[Unbiased and bounded gradient oracle]\label{asm:boundVar}
The stochastic gradient oracle ${\*g}_{k} = \E_{\bm{\zeta}}[\nabla f(\bm{\theta}_k,\bm{\zeta})] + \bm{\xi}_k$ is unbiased, i.e., $\E\qty(\bm{\xi}_k) = \bm{0}$, and its magnitude and variance are bounded  with probability $1$:
\[
  \norm{{\*g}_{k}}_{\infty} \leq c_\infty/3, \quad 	\^E\qty(\norm{\bm{\xi}_k}^2) \leq \sigma^2, 
\quad \forall k \in [T].
\]
\end{assumption}
}
\begin{assumption}[$\rho$-Lipschitz continuous Hessian]\label{asm:rhosmooth}
	The function $f(\cdot,\cdot)$ has $\rho$-Lipschitz Hessian w.r.t. the parameters. :
	\[
	\norm{\nabla^2  F(\*x) - \nabla^2  F(\*y)} \leq \rho \norm{\*x - \*y},\qquad \forall \*x, ~\*y,
	\]
	where $F(\*x) \coloneqq \E_{\bm{\zeta}}[f(\*x,\bm{\zeta})]$, $\norm{\cdot}$ is matrix spectral norm for matrix and $\ell_2$ norm for vector.
\end{assumption}
For a general nonconvex problem, if Assumptions \ref{asm:Lsmooth} and~\ref{asm:boundVar} hold, the lower bound of the stochastic gradient complexity (a.k.a. IFO complexity) 
to find an  $\epsilon$-approximate first-order stationary point ($\epsilon$-ASP) is $\Omega(\epsilon^{-4})$~\cite{arjevani2019lower}. 
Moreover, if Assumption \ref{asm:rhosmooth} further holds, the lower complexity bound becomes $\Omega(\epsilon^{-3.5})$ for a non-variance-reduction algorithm~\cite{arjevani2020second}.  

{\textbf{Lipschitz Gradient.}}
Theorem~\ref{theorem1} with proof in Supplementary Sec.~\ref{sec:first-order} proves the  convergence of  Adan on problem~\eqref{decoupleproblem} with Lipschitz gradient condition. 
\begin{manualtheorem}{1}\label{theorem1}
Suppose that Assumptions \ref{asm:Lsmooth} and \ref{asm:boundVar} hold.
	Let  $\max\qty{\beta_1,\beta_2} = \order{\epsilon^2}$, $\mu\coloneqq {\sqrt{2\beta_3} c_\infty}/{\varepsilon} \ll 1$, $\eta = \order{\epsilon^2}$, and {$\lambda_k = \lambda \qty(1-\mu)^k$}. Algorithm~\ref{alg:Name} runs at most $K = \Omega\qty(c^{2.5}_\infty \epsilon^{-4})$ iterations to achieve:
	\[
	\frac{1}{K+1} \sum\nolimits_{k=0}^K \^E\qty(\norm{\nabla F_k(\bm{\theta}_k)}^2) \leq 4\epsilon^2.
	\]
That is, to find an $\epsilon$-ASP, the stochastic gradient complexity  of Adan on problem~\eqref{decoupleproblem} is $\order{c^{2.5}_{\infty} \epsilon^{-4}}$. 
\end{manualtheorem}
Theorem~\ref{theorem1} shows that under  Assumptions \ref{asm:Lsmooth} and \ref{asm:boundVar}, Adan  can converge to an $\epsilon$-ASP of a nonconvex stochastic problem with stochastic gradient complexity  $\order{c^{2.5}_{\infty} \epsilon^{-4}}$ which accords with the lower bound $\Omega(\epsilon^{-4})$ in~\cite{arjevani2019lower}.  
For this convergence, Adan has no requirement on minibatch size and only assumes gradient estimation to be unbiased and bounded. 
Moreover, as shown in Table~\ref{tab:cvspeed} in Sec.~\ref{introduction}, the complexity of Adan is superior to those of previous adaptive gradient algorithms. 
For Adabelief and LAMB,  Adan always has lower complexity and respectively enjoys $d^3\times$ and $d^2\times$ lower complexity for the worst case. 
Adam-type optimizers (\eg~Adam and AMSGrad) enjoy the same complexity as Adan. 
But they cannot separate the $\ell_2$ regularizer with the objective like AdamW and Adan.
Namely, they always solve a static loss $F(\cdot)$ rather than a dynamic loss $F_k(\cdot)$. 
The regularizer separation can boost generalization performance~\cite{touvron2021training,liu2021swin} and already helps  AdamW dominate training of ViT-alike architectures. 
Besides, some previous analyses~\cite{luo2018adaptive,zaheer2018adaptive,liu2019variance,shi2020rmsprop} need the momentum coefficient (\ie~$\beta$s) to be close or increased to one, which contradicts with the practice that $\beta$s are close to zero. 
In contrast, Theorem~\ref{theorem1} assumes that all $\beta$s are very small, which is more consistent with the practice. 
Note that when $\mu = c/T$, we have $\lambda_k/\lambda \in [(1-c),1]$ during training. Hence we could choose the $\lambda_k$ as a fixed constant in the experiment for convenience.

{\textbf{Lipschitz Hessian.}}
To further improve the theoretical convergence speed, we introduce Assumption \ref{asm:rhosmooth}, and set a proper restart condition to reset the momentum during  training.
Consider an extension point $\*y_{k+1}:=\bm{\theta}_{k+1} + \bm{\eta}_k\circ \qty[\*m_k + \qty(1-\beta_2)\*v_k - \beta_1{\*g}_k]$, and the restart condition is:
\begin{equation}\label{eq:restartcond}
	(k+1)\sum\nolimits_{t=0}^{k} \norm{\*y_{t+1} - \*y_t}^2_{\sqrt{\*n_{t}}} > R^2,
\end{equation}
where the constant $R$  controls the restart frequency.
Intuitively, when the parameters have accumulated enough updates, the iterate may reach a new local basin. 
Resetting the momentum at this moment helps Adan to better use the local geometric information.
Besides, we change $\bm{\eta}_k$ from ${\eta}/\qty(\sqrt{\*n_{k}} + \varepsilon)$ to ${\eta}/\qty(\sqrt{\*n_{k-1}} + \varepsilon) $  to ensure $\bm{\eta}_k$ to be independent of noise $\bm{\zeta}_k$.
See its proof in Supplementary~\ref{sec:second-order}.
 
\begin{manualtheorem}{2}
\label{theorem2}
	Suppose that  Assumptions \ref{asm:Lsmooth}-\ref{asm:rhosmooth} hold. Let $R = \order{\epsilon^{0.5}}$, $\max\qty{\beta_1,\beta_2}  = \order{\epsilon^2}$, $\beta_3  = \order{\epsilon^4}$, $\eta =  \order{\epsilon^{1.5}}$, $K = \order{\epsilon^{-2}}$, $\lambda = 0$. Then Algorithm \ref{alg:Name} with restart condition Eqn.~\eqref{eq:restartcond} satisfies:
	\[
	\E\qty(\norm{\nabla F_k(\Bar{\bm{\theta}})}) = \order{c_\infty^{0.5} \epsilon}, \quad \text{where } \Bar{\bm{\theta}} \coloneqq\frac{1}{K_0} \sum_{k=1}^{K_0}\bm{\theta}_k,
	\]
and
 $
K_{0}=\argmin_{\lfloor\frac{K}{2}\rfloor \leq k \leq K-1} \norm{\*y_{t+1} - \*y_t}^2_{\sqrt{\*n_{t}}}$.
	Moreover, to find an $\epsilon$-ASP, Algorithm \ref{alg:Name} restarts  at most   $\order{c_\infty^{0.5}\epsilon^{-1.5}}$ times in which each restarting cycle has at most $K = \order{\epsilon^{-2}}$ iterations, and hence needs  at most $\order{{c_{\infty}^{1.25}}\epsilon^{-3.5}}$ stochastic gradient complexity.  
\end{manualtheorem}
From Theorem~\ref{theorem2}, one can observe that with an extra Hessian condition, Assumption~\ref{asm:rhosmooth}, 
Adan improves its stochastic gradient complexity from $\order{c_{\infty}^{2.5}\epsilon^{-4}}$ to  $\order{{c_{\infty}^{1.25}}\epsilon^{-3.5}}$, which also matches the corresponding lower bound $\Omega(\epsilon^{-3.5})$~\cite{arjevani2020second}. 
This complexity is lower than  $\order{\epsilon^{-3.5}\log\frac{c_2}{\epsilon}}$ of A-NIGT~\cite{cutkosky2020momentum}  and $\order{\epsilon^{-3.625}}$  of Adam$^+$~\cite{liu2020adam}. For other DNN optimizers, \eg~Adam, their convergence under Lipschitz Hessian condition has not been proved yet.

Moreover, Theorem~\ref{theorem2} still holds for the large batch size. For example, by using minibatch size $b = \order{\epsilon^{-1.5}}$, our results still hold when $R = \order{\epsilon^{0.5}}$, $\max\qty{\beta_1,\beta_2}  = \order{\epsilon^{0.5}}$, $\beta_3  = \order{\epsilon}$, $\eta =  \order{1}$, $K = \order{\epsilon^{-0.5}}$, and $\lambda=0$. In this case, our  $\eta$ is of the order $\order{1}$, and  is much larger than $\order{\operatorname{ploy}(\epsilon)}$ of other  optimizers (\eg, LAMB~\cite{you2019large} and Adam$^{+}$) for handling large minibatch. This large step size often boosts convergence speed in practice, which is actually desired.

\section{Experimental Results}\label{experiments}
\setlength{\floatsep}{3.5pt}
\setlength{\textfloatsep}{5.5pt}
\setlength{\intextsep}{3.5pt}%
\setlength{\dbltextfloatsep}{5.0pt} 
\setlength{\dblfloatsep}{5.0pt} 
We evaluate Adan on vision tasks, natural language processing (NLP) tasks and  reinforcement learning (RL) tasks.  For vision classification tasks, we test Adan on several representative SoTA backbones  under the conventional supervised settings,  including 1) CNN-type architectures (ResNets~\cite{he2016deep} and ConvNexts~\cite{liu2022convnet}) and 2) ViTs (ViTs~\cite{dosovitskiy2020image,touvron2021training} and Swins~\cite{liu2021swin}). We also investigate  Adan via the self-supervised pretraining  by using it to train MAE-ViT~\cite{he2022masked}.  
Moreover, we test Adan on the vision object detection and instance segmentation tasks with two frameworks Deformable-DETR~\cite{zhu2021deformable} and Mask-RCNN\cite{he2017mask} (choosing ConvNext~\cite{liu2022convnet} as the backbone).
For NLP tasks,  we train LSTM~\cite{schmidhuber1997long}, Transformer-XL~\cite{dai2019transformer}, and BERT~\cite{devlin2018bert}  for sequence modeling.
We also provide the Adan's results on large language models, like GPT-2~\cite{radford2019language}, on the code generation tasks. For more results on general large-language model can be found in Appendix.
On RL  tasks, we evaluate Adan on four games in MuJoCo~\cite{todorov2012mujoco}.
We also conduct the experiments on GNNs.

\textbf{Due to space limitation, we defer the ablation study, additional experimental results, and implementation details into supplementary materials.} 
We compare Adan with the model's default/SoTA optimizer in all the experiments but may miss some representative optimizers, \eg, Adai, Padam, and AdaBlief in some cases. This is because they report few results for larger-scale experiments. For instance, Adablief only tests ResNet-18 performance on ImageNet and actually does not test any other networks. So it is really hard for us to compare them on ViTs, Swins, ConvNext, MAEs, etc, due to the challenges for hyper-parameter tuning and limited GPU resources. 
The other reason is that some optimizers may fail or achieve poor performance on transformers. For example, SGD and Adam achieve much lower accuracy than AdamW. See Table~\ref{tab:vit-s} in supplementary materials.

\begin{figure*}[t]
\begin{minipage}[c]{.49\linewidth}
\centering
\includegraphics[width=\textwidth]{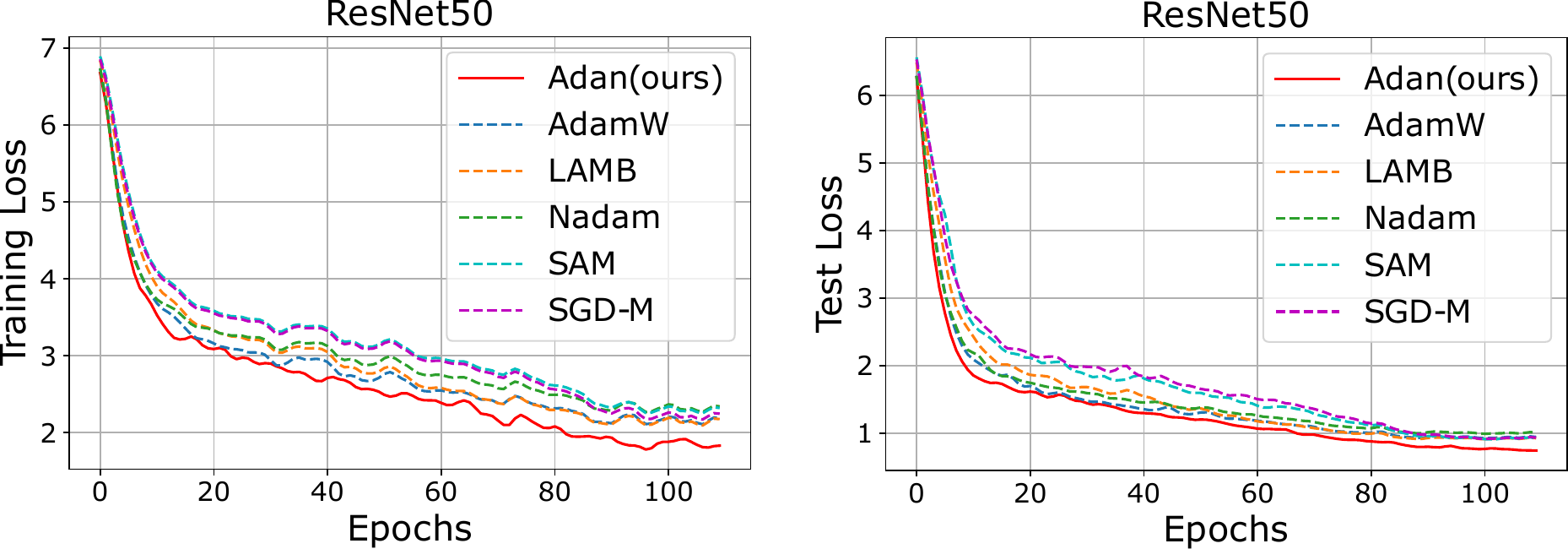} \\
	(a) Training and test curves  on ResNet-50. 
\end{minipage}
\hspace{0.05cm}
\begin{minipage}[c]{.49\linewidth}
\centering
\includegraphics[width=\textwidth]{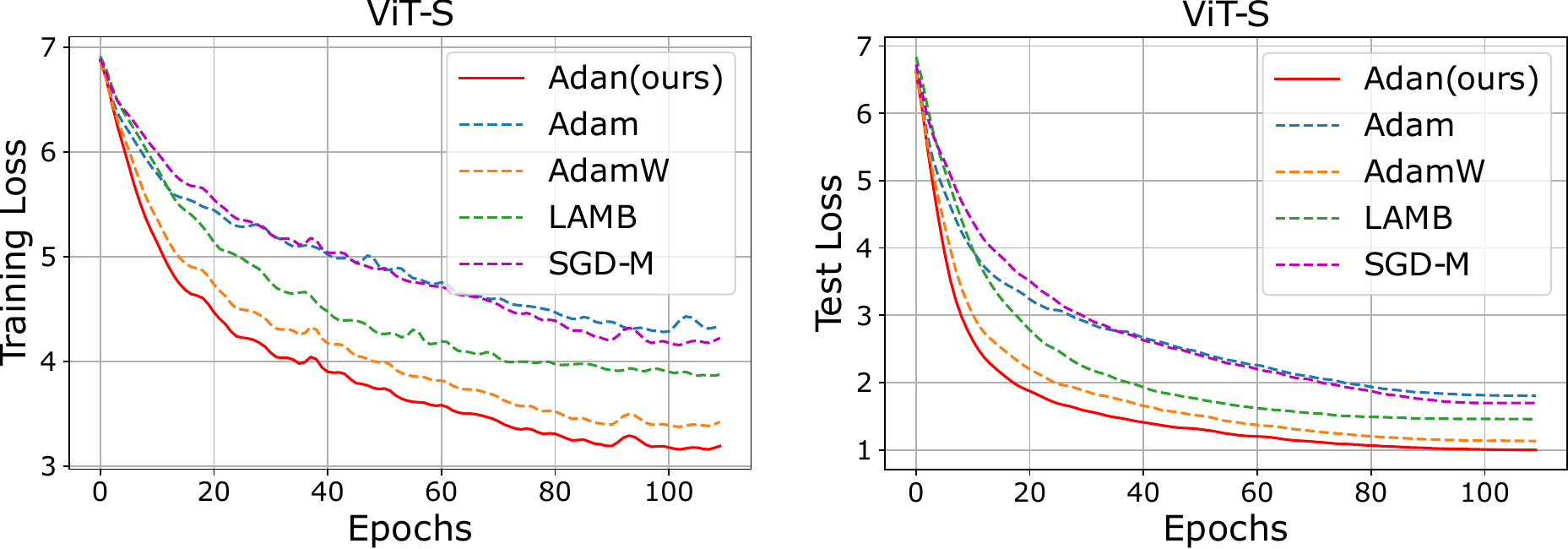}\\
		(b) Training and test curves on ViT-S. 
\end{minipage}
		\caption{Training and test curves of various optimizers on ImageNet.  
	The different magnitude of training and test loss is due to  data argumentation.
 \textbf{Best viewed in 2$\times$-sized color pdf file.}}\label{fig:res-vit}
\end{figure*}

\begin{table*}[t!]
	\caption{{Top-1 ACC. (\%) of ViT and Swin on ImageNet.} 
		We  use their official Training Setting II to  train them. $*$ and $\diamond$ are respectively reported in~\cite{touvron2021training,liu2021swin}}
	\label{tab:ViT}
	\centering
	\setlength{\tabcolsep}{10pt} 
	\renewcommand{\arraystretch}{3.0}
	{ \fontsize{8.3}{3}\selectfont{
			\begin{tabular}{l|cc|cc|cc| cc| cc}
				\toprule
				& \multicolumn{2}{c|}{ViT Small}    & \multicolumn{2}{c|}{ViT Base}   & \multicolumn{2}{c|}{Swin Tiny}  & \multicolumn{2}{c|}{Swin small} &  \multicolumn{2}{c}{Swin Base}    \\ 
				Epoch       & 150   & 300  & 150   & 300 & 150   & 300 & 150   & 300 & 150   & 300    \\ \midrule
				AdamW~\cite{touvron2021training,liu2021swin,loshchilov2018decoupled}   &  78.3 &\;79.9$^{*}$ & 79.5 & \;\;81.8$^{*}$ & 79.9  &\;\;81.2$^{\diamond}$  & 82.1 & \;\;83.2$^{\diamond}$ & 82.6 & \;\;83.5$^{\diamond}$   \\
				\textbf{Adan \scriptsize{(ours)}}  & \textbf{79.6} & \textbf{80.9} & \textbf{81.7} & \textbf{82.6} & \textbf{81.3} & \textbf{81.6} &  \textbf{82.9} &  \textbf{83.7} & \textbf{83.3} & \textbf{83.8}  \\ \bottomrule
			\end{tabular}
	}}
\end{table*}

\subsection{Experiments for Vision Classification Tasks}
\subsubsection{Training Setting}
Besides the  vanilla supervised training setting used in ResNets~\cite{he2016deep}, we further consider the following two prevalent training settings on ImageNet~\cite{deng2009imagenet}.

{\textbf{Training Setting I.}}  The recently proposed ``A2 training recipe'' in \cite{wightman2021resnet} has pushed the performance limits of many SoTA CNN-type architectures by using stronger data augmentation and more training iterations. For example, on ResNet50, it sets new SoTA $80.4\%$, and improves the accuracy $76.1\%$ under vanilla setting in~\cite{he2016deep}. 
Specifically, for data augmentation, this setting uses  random crop,  horizontal flipping, Mixup (0.1)~\cite{zhang2018mixup}/CutMix (1.0)~\cite{yun2019cutmix} with probability $0.5$, and RandAugment~\cite{cubuk2020randaugment} with $M =7, N=2$ and MSTD $= 0.5$. It sets stochastic depth $(0.05)$~\cite{huang2016deep}, and adopts cosine learning rate decay and  binary cross-entropy (BCE) loss. 
For Adan, we  use  batch size 2048 for ResNet and ViT. 

{\textbf{Training Setting II.}}  We follow the same official training procedure of ViT/Swin/ConvNext.
For this setting,  data augmentation includes random crop, horizontal flipping, Mixup (0.8), CutMix (1.0),  RandAugment ($M =9$, MSTD $= 0.5$) and Random Erasing ($p =0.25$).  We use CE loss, the cosine decay for base learning rate,  the stochastic depth (with official parameters), and weight decay. 
For Adan, we set batch size 2048 for Swin/ViT/ConvNext and 4096 for MAE. 
We follow MAE and tune $\beta_3$ as 0.1. 

\begin{table*}[t!]
\begin{minipage}[c]{.49\linewidth}
	\caption{Top-1 Acc. (\%) of ViT-B and ViT-L trained by   MAE under the official Training Setting II. $*$ and $\diamond$ are respectively reported in~\cite{chen2022context,he2022masked}.}
\label{tab:MAE}
 		\centering
 		\setlength{\tabcolsep}{3.6pt} 
 		\renewcommand{\arraystretch}{3.0}
 	{ \fontsize{8.3}{3}\selectfont{
 	\begin{tabular}{l|ccc|cc}
 	\toprule
 	& \multicolumn{3}{c|}{MAE-ViT-B} & \multicolumn{2}{c}{MAE-ViT-L} \\
 	Epoch & 300      & 800      & 1600    & 800           & 1600          \\ \midrule
 	AdamW~\cite{loshchilov2018decoupled,he2022masked}    &\;\;82.9$^*$& ---      & 83.6$^{\diamond}$  &\;\;85.4$^{\diamond}$        & 85.9$^{\diamond}$       \\
 	\textbf{Adan \scriptsize{(ours)}}  &\textbf{83.4}  & \textbf{83.8}   & ---     & \textbf{85.9}       & ---   \\
 	\bottomrule
\end{tabular}
}}
\end{minipage}
\hspace{0.2cm}
	\begin{minipage}[c]{.49\linewidth}
 \caption{Top-1 Acc. (\%) of ViT-S  on ImageNet trainined by Adam and LAMB under the Training Setting I with different batch sizes. 
	}
	\label{tab:batch-size}
	\centering
	\setlength{\tabcolsep}{5.0pt} 
	\renewcommand{\arraystretch}{4.0}
	{ \fontsize{8.3}{3}\selectfont{
			\begin{tabular}{l|cccccc}
				\toprule 
				Batch Size       & 1k& 2k &  4k   & 8k & 16k & 32k    \\ \midrule
				LAMB~\cite{you2019large,he2021large}     &78.9 & 79.2 & 79.8   & 79.7    &79.5   & 78.4   \\ 
				\textbf{Adan \scriptsize{(ours)}}  & \textbf{80.9}   & \textbf{81.1} & \textbf{81.1}   & \textbf{80.8}   & \textbf{80.5}  & \textbf{80.2}  \\ \bottomrule
			\end{tabular}
	}}
 \end{minipage}
 \end{table*}

\subsubsection{Results on CNN-type Architectures}
To train ResNet and ConvNext, we respectively use their official Training Setting I and II. 
For ResNet/ConvNext, its  default official  optimizer is LAMB/AdamW. From Table~\ref{tab:CNNs},   one can observe that on ResNet, 1) in most cases, Adan only running 200 epochs can achieve higher or comparable top-1 accuracy on ImageNet~\cite{deng2009imagenet} compared with the official SoTA result trained by LAMB with 300 epochs; 2) Adan gets more improvements over other optimizers, when training is insufficient, \eg~100 epochs.  
The possible reason for observation 1) is the regularizer separation, which can dynamically adjust the weight decay for each coordinate instead of sharing a common one like LAMB.
For observation 2),  this can be explained by the faster convergence speed of Adan than other optimizers.
As shown in Table~\ref{tab:cvspeed}, Adan converges faster than many adaptive gradient optimizers.  
This faster speed partially comes from its large learning rate guaranteed by Theorem \ref{theorem2}, almost $3\times$ larger than that of LAMB, since the same as Nestrov acceleration, Adan also looks ahead for possible correction.  Note, we have tried to adjust learning rate and warmup-epoch for Adam and LAMB, but observed unstable  training  behaviors. 
On ConvNext (tiny and small), one can observe similar comparison results on ResNet. 

Since some well-known deep optimizers test ResNet-18 for 90 epochs under the official vanilla training setting~\cite{he2016deep},  we also run Adan  90 epochs under this setting for more comparison.   Table~\ref{tab:res-18} shows that  Adan consistently outperforms SGD and all compared adaptive optimizers. 
Note for this setting, it is not easy for adaptive optimizers to surpass SGD due to the absence of heavy-tailed noise,  which is the crucial factor helping adaptive optimizers beat AGD~\cite{zhang2020adaptive}.

Additionally, we have extended our experiments to include smaller datasets, specifically running tests on the CIFAR-10~\cite{krizhevsky2009learning} dataset using a ResNet-34 model to evaluate Adan against nine other optimizers. These experiments were conducted using AdaBelief's codebase~\cite{zhuang2020adabelief} as a benchmark for settings and hyperparameters, ensuring consistency and comparability. The results, now included in Table~\ref{tab:cifar}, reveal that Adan not only maintains its superior performance in comparison with other optimizers but also confirms its efficacy on smaller datasets. This evidence underlines Adan's robust performance across various dataset sizes and its capability to adapt to diverse training conditions.

\subsubsection{Results on ViTs}
{\textbf{Supervised Training.}} 
We train ViT and Swin under their official training setting, \ie~Training Setting II.  Table \ref{tab:ViT} shows that across different model sizes of ViT and Swin, Adan outperforms the official AdamW optimizer by a large margin. For ViTs, their gradient per iteration differs much from the previous one due to the much sharper loss landscape than CNNs~\cite{chen2021vision} and the strong random augmentations for training.
So it is hard to train ViTs to converge within a few epochs. Thanks to its faster convergence, as shown in Figure~\ref{fig:res-vit}, Adan is very suitable for this situation.
Moreover, the direction correction term from the gradient difference $\*v_k$ of Adan can also better correct the  first- and second-order moments. 
One piece of evidence is that the first-order moment  decay coefficient $\beta_1=0.02$ of Adan is much smaller than $0.1$ used in other deep optimizers. 
Besides AdamW, we also compare Adan with several other popular optimizers, including Adam, SGD-M, and LAMB, on ViT-S, please see Table \ref{tab:vit-s} in supplementary materials.

{\textbf{Self-supervised MAE Training (pre-train + finetune).}}  
We follow the MAE training framework to pre-train and finetune ViT-B on ImageNet, \ie~300/800 pretraining epochs and 100 fine-tuning epochs.  
Table \ref{tab:MAE} shows that 1) with 300 pre-training epochs, Adan makes $0.5\%$ improvement over AdamW; 2) Adan pre-trained 800 epochs surpasses AdamW pre-trained 1,600 epochs by non-trial  $0.2\%$.  All these results show the superior convergence and generalization performance of Adan. 

{\textbf{Large-Batch Training.}}   
Although large batch size can increase computation parallelism to reduce training time and is heavily desired, optimizers often suffer performance degradation, or even fail.
For instance,  AdamW   fails to train ViTs when batch size is beyond  4,096.  How to solve the problem remains open~\cite{he2021large}. At present, LAMB is  the most effective optimizer for large batch size.  
Table~\ref{tab:batch-size}  reveals that Adan is robust to batch sizes from 2k to 32k, and shows higher performance and robustness than LAMB.

\subsubsection{Comparison of Convergence Speed}\label{vis}

In  Figure~\ref{fig:res-vit} (a), we plot the curve of training and test loss along with the training epochs on ResNet50. One can observe that Adan converges faster than the compared baselines and enjoys the smallest training and test losses. This demonstrates its fast convergence property and good generalization ability.   To sufficiently investigate  the fast convergence of Adan, we further plot  the curve of training and test loss  on the ViT-Small in Figure~\ref{fig:res-vit} (b). From the results, we can see that Adan consistently shows  faster convergence behaviors than other baselines in terms of both training loss and test loss.  This also partly explains the good performance of Adan.

\begin{table*}[t]
\begin{minipage}[c]{.49\linewidth}
\centering
\setlength{\tabcolsep}{3.0pt} 
\renewcommand{\arraystretch}{3.0}
\caption{Detection box-AP of Deformable-DETR~\cite{zhu2021deformable} on COCO. $*$ and $\diamond$ are respectively reported in official setting~\cite{zhu2021deformable} and MMdection's improved settings~\cite{mmdetection}. The official optimizer is AdamW and the training epoch is 50.}\label{tab:detr}
{\fontsize{8.3}{3}\selectfont{
\begin{tabular}{cc|ccc}
\toprule 
\begin{tabular}[c]{@{}c@{}}Method\\(Backbone)\end{tabular}              & Optimizer & {AP$^b$}   & {AP$^b_{50}$}   & {AP$^b_{75}$}    \\ 
\midrule
\multirow{3}{*}{\begin{tabular}[c]{@{}c@{}}Deformable-DETR\\ (ResNet-50)\end{tabular}} &
AdamW     &~~43.8$^*$ & ~~62.6$^*$ & ~~47.7$^*$  \\
                  &              
AdamW     &~~44.5$^{\diamond}$ & ~~63.2$^{\diamond}$ & ~~48.9$^{\diamond}$ \\
& 
\textbf{Adan}      & \textbf{45.3} & \textbf{64.4} & \textbf{49.3} 
\\
\bottomrule
\end{tabular}
}}
\end{minipage}
\hspace{0.1cm}
\begin{minipage}[c]{.49\linewidth}
\centering
\setlength{\tabcolsep}{1.5pt} 
\renewcommand{\arraystretch}{3.6}
\caption{Instance segmentation box/mask-AP of Mask-RCNN~\cite{he2017mask}, choosing ConvNext-T as the backbone, on COCO. $\diamond$ is from~\cite{mmdetection}. The official optimizer of these settings is AdamW and the training epoch is 36.}\label{tab:rcnn}
{\fontsize{8.3}{3}\selectfont{
\begin{tabular}{cc|ccc|ccc}
\toprule 
\begin{tabular}[c]{@{}c@{}}Method\\ (Backbone)\end{tabular}      &{Optimizer} & {AP$^b$}   & {AP$^b_{50}$}   & {AP$^b_{75}$}   & {AP$^m$}   & {AP$^m_{50}$}   & {AP$^m_{75}$}   \\ 
\midrule
\multirow{2}{*}{\begin{tabular}[c]{@{}c@{}}Mask R-CNN\\ (ConvNeXt-T)\end{tabular}}   & 
AdamW     &~~46.2$^{\diamond}$ & ~~68.1$^{\diamond}$ & ~~50.8$^{\diamond}$ & ~~41.7$^{\diamond}$ & ~~65.0$^{\diamond}$ & ~~44.9$^{\diamond}$ \\
 & \textbf{Adan}      & \textbf{46.7} & \textbf{68.5} & \textbf{51.0} & \textbf{42.2} & \textbf{65.5} & \textbf{45.3}
\\ \bottomrule
\end{tabular}
}}
\end{minipage}
\end{table*}

\begin{table*}[t]
\centering
\setlength{\tabcolsep}{3.0pt} 
\renewcommand{\arraystretch}{3.0}
\caption{ Test perplexity (the lower, the better) on Penn Treebank for one-, two- and three-layered LSTMs. 
All results except Adan and Padam in the table are reported by  AdaBelief~\cite{zhuang2020adabelief}.}\label{tab:lstm}
{\fontsize{8.3}{3}\selectfont{
\begin{tabular}{c|cccccccccc}
\toprule
LSTM    & \textbf{Adan} & AdaBelief~\cite{zhuang2020adabelief}  & SGD~\cite{robbins1951stochastic}  & AdaBound~\cite{luo2018adaptive}  & Adam~\cite{kingma2014adam} & AdamW~\cite{loshchilov2018decoupled} 
& Padam~\cite{chen2021closing} & RAdam~\cite{liu2019variance} & Yogi~\cite{zaheer2018adaptive}  \\ \midrule
1 layer  & \textbf{83.6}       & 84.2      & 85.0 
& 84.3 & 85.9  & 84.7 & 84.2       & 86.5  & 86.5  \\
2 layers  & \textbf{65.2}       & 66.3      & 67.4 & 67.5     & 67.3 & 72.8  
& 67.2       & 72.3  & 71.3   \\
3 layers &  \textbf{59.8}       & 61.2      & 63.7 & 63.6     & 64.3 & 69.9  
& 63.2       & 70.0  & 67.5 \\
\bottomrule
\end{tabular}
}}
\end{table*}

\subsubsection{Experiments for Detection and Segmentation Tasks}
In this experiment, we test Adan on the detection and segmentation tasks via the COCO dataset~\cite{lin2014microsoft} which is a large-scale dataset for detection, segmentation and captioning tasks. 
We accomplish the experiments with Deformable-DETR~\cite{zhu2021deformable} and Mask R-CNN~\cite{he2017mask} (with ConvNext~\cite{liu2022convnet} as the backbone) to compare Adan and their official  optimizer AdamW.

Table~\ref{tab:detr} reports the box Average Precision (AP) of objection detection by Deformable-DETR. For AdamW, its results  on Deformable-DETR are quoted  from the reported results under the official setting~\cite{zhu2021deformable} and improved setting from MMdection~\cite{mmdetection}. For fairness, we also follow the setting in MMdection to test Adan.   
The results in Table~\ref{tab:detr} show that Adan improves the box AP by $1.6\%\sim 1.8\%$ compared to the official optimizer AdamW. 
Meanwhile, Table~\ref{tab:rcnn} reports both the box AP and mask AP of instance segmentation by Mask R-CNN with ConvNext backbone.
Adan achieves $0.5\% \sim 1.2\%$ mask/box AP improvement over the official optimizer AdamW. 
All These results show the effectiveness of the proposed Adan.

\subsection{Experiments for Language Processing Tasks}
\subsubsection{Results on LSTM} 
To begin with, we test our Adan on LSTM~\cite{schmidhuber1997long}  by using the Penn TreeBank dataset~\cite{marcinkiewicz1994building}, and report the perplexity (the lower, the better) on the test set in Table~\ref{tab:lstm}. We follow the exact experimental setting in   Adablief~\cite{zhuang2020adabelief}. Indeed, all our implementations are also based on  the code provided by Adablief~\cite{zhuang2020adabelief}\footnote{ The reported results in~\cite{zhuang2020adabelief}   slightly differ from the those in~\cite{chen2021closing} because of  their different settings for LSTM and training hyper-parameters.}.   
We use the default  setting for all the hyper-parameters provide by Adablief, since it provides more baselines for fair comparison.  
For Adan, we utilize its default weight decay ($0.02$) and $\beta$s ($\beta_1 = 0.02, \beta_2 = 0.08$, and $\beta_3 = 0.01$).
We choose learning rate as $0.01$ for Adan.

Table~\ref{tab:lstm} shows that on the three LSTM models, Adan always achieves the lowest perplexity, making about 1.0  overall average perplexity improvement over the runner-up. Moreover, when  the LSTM depth increases, the advantage of Adan becomes more remarkable.

\begin{figure*}[t]	\includegraphics[width=\textwidth]{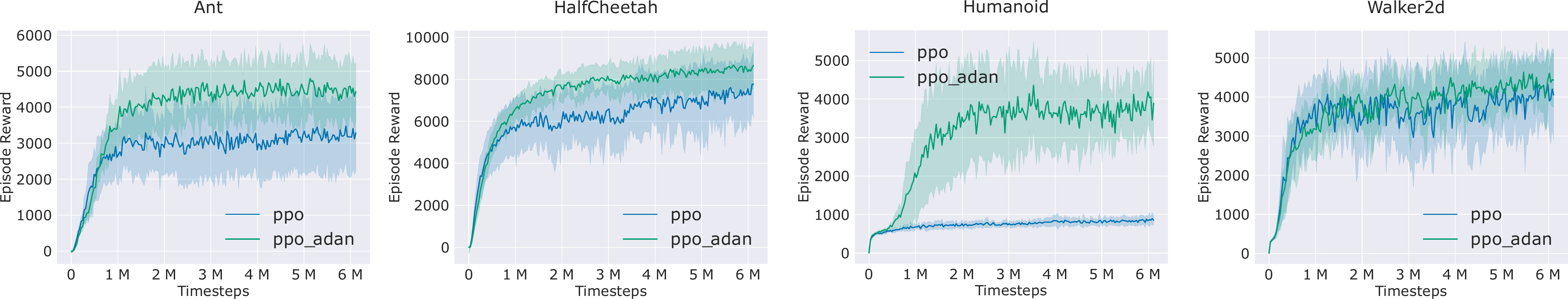}
	\centering
		\caption{Comparison of PPO and our PPO-Adan on several RL games  simulated by MuJoCo. Here PPO-Adan simply replaces the Adam optimizer in PPO with our Adan and does not change others.  \textbf{Best viewed in 2$\times$-sized color pdf file.}
}\label{fig:rl}
\end{figure*}

\begin{table*}[t!]
\centering
\setlength{\tabcolsep}{10.0pt} 
\renewcommand{\arraystretch}{3.2}
\caption{ Correlation or ACC. (\%) (the higher, the better) of BERT-base model on the development set of GLUE.  }\label{tab:bert}
{\fontsize{8.3}{3}\selectfont{
\begin{tabular}{l|cccccccc}
\toprule
BERT-base & MNLI & QNLI & QQP & RTE & SST-2  & CoLA & STS-B & \textbf{Average}   \\ \midrule
 Adam~\cite{kingma2014adam} (from~\cite{huggingface}) & 83.7/84.8 & 89.3 & 90.8 & 71.4 & 91.7  & 48.9 & 91.3 & 81.5 \\
Adam~\cite{kingma2014adam} (reproduced) & 84.9/84.9 & 90.8 & 90.9 & 69.3 & 92.6  & 58.5 & 88.7 & 82.5  \\
\textbf{Adan (ours)}  & \textbf{85.7/85.6} & \textbf{91.3} & \textbf{91.2} & \textbf{73.3} & \textbf{93.2}  & \textbf{64.6} & \textbf{89.3} & \textbf{84.3 (+1.8)}  \\
\bottomrule
\end{tabular}
}}
\end{table*}

\begin{table*}[t!]
\begin{minipage}[c]{.47\linewidth}
\centering
\setlength{\tabcolsep}{7.0pt} 
\renewcommand{\arraystretch}{3.5}
\caption{Pass@k metric (the higher, the better), evaluating functional correctness, for GPT-2 (345M) model on the HumanEval dataset pre-trained with different steps.}\label{tab:gpt}
{\fontsize{8.3}{3}\selectfont{
\begin{tabular}{c|c|ccc}
\toprule
GPT-2 (345m) & Steps & pass@1 & pass@10 & pass@100 \\
\midrule
Adam         & 300k  & 0.0840 & 0.209   & 0.360    \\
\textbf{Adan}        & 150k  & \textbf{0.0843} & \textbf{0.221}   & \textbf{0.377}  \\ \bottomrule
\end{tabular}
}}
\end{minipage}
\hspace{0.2cm}
\begin{minipage}[c]{.47\linewidth}
\centering
\setlength{\tabcolsep}{10.0pt} 
\renewcommand{\arraystretch}{2.5}
\caption{Test PPL (the lower, the better) for Transformer-XL-base model on the WikiText-103 dataset with different training steps. * is reported in the official implementation.}\label{tab:XLNet}
{\fontsize{8.3}{3}\selectfont{
\begin{tabular}{l|ccc}
\toprule
\multicolumn{1}{c|}{\multirow{2}{*}{Transformer-XL-base}} & \multicolumn{3}{c}{Training Steps}                                            \\
\multicolumn{1}{c|}{}                                     & \multicolumn{1}{c}{50k} & \multicolumn{1}{c}{100k} & \multicolumn{1}{c}{200k} \\ \midrule
Adam~\cite{kingma2014adam}                                                      & 28.5                    & 25.5                     &~~24.2$^*$                    \\
\textbf{Adan (ours)}                                               & \textbf{26.2}                    & \textbf{24.2}                     &  \textbf{23.5}                    \\ \bottomrule
\end{tabular}
}}
 \end{minipage}
 \end{table*}

\subsubsection{Results on BERT}
Similar to the pretraining experiments of MAE which is also a  self-supervised learning framework on vision tasks, we utilize Adan to train BERT~\cite{devlin2018bert} from scratch, which is one of the most widely used pretraining models/frameworks for NLP tasks. 
We employ the exact BERT training setting in the widely used codebase---Fairseq~\cite{ott2019fairseq}. 
We replace the default Adam optimizer in BERT with our Adan for both pretraining and fune-tuning. 
Specifically, we first pretrain BERT-base on the Bookcorpus and Wikipedia datasets, and then finetune BERT-base separately for each GLUE task on the corresponding training data.  Note, GLUE is a collection of 9 tasks/datasets to  evaluate natural language understanding systems, in which the tasks are organized as either single-sentence classification or sentence-pair classification. 

Here we simply replace the Adam optimizer in BERT with our Adan and do not make other changes, \eg~random seed, warmup steps and learning rate decay strategy, dropout probability, \etc.
For pretraining, we use Adan with its default weight decay ($0.02$) and $\beta$s ($\beta_1 = 0.02, \beta_2 = 0.08$, and $\beta_3 = 0.01$), and choose learning rate as $0.001$.
For fine-tuning, we consider a limited hyper-parameter sweep for each task, with a batch size of 16, and learning rates $\in \{2e-5, 4e-5\}$ and use Adan with $\beta_1 = 0.02, \beta_2 = 0.01$, and $\beta_3 = 0.01$ and weight decay $0.01$. 

Following the conventional setting, we run each fine-tuning experiment three times and report the median performance in Table~\ref{tab:bert}. 
On MNLI,  we report the mismatched and matched accuracy.
And we report Matthew's Correlation and Person Correlation on the task of CoLA and STS-B, respectively.
The performance on the other tasks is measured by classification accuracy.
The performance of our reproduced one (second row) is slightly better than the vanilla  results of BERT  reported in Huggingface-transformer~\cite{huggingface} (widely used codebase for transformers in NLP), since the vanilla Bookcorpus data in~\cite{huggingface}  is not available and thus we train on the latest Bookcorpus data version.

From Table~\ref{tab:bert}, one can see that in the most commonly used BERT training experiment, Adan reveals a much better advantage over Adam. Specifically, in all GLUE tasks, on the BERT-base model,  Adan achieves higher performance than Adam and makes 1.8 average improvements on all tasks. In addition, on some tasks of Adan, the BERT-base trained by Adan can outperform some large models. e.g., BERT-large which achieves 70.4\% on RTE, 93.2\% on SST-2, and 60.6 correlation on CoLA, and XLNet-large which has 63.6 correlation on CoLA. See~\cite{liu2019roberta} for more results.

\subsubsection{Results on GPT-2}
We evaluate Adan on the 
large language models (LLMs), GPT-2~\cite{radford2019language}, for code generalization tasks, which enables the completion and synthesis of code, both from other code snippets and natural language descriptions.
LLMs  work across a wide range of domains, tasks, and programming languages, and
can, for example, assist professional and citizen developers with building new applications.
We pre-train GPT-2 on The-Stack dataset (Python only)~\cite{Kocetkov2022TheStack} from BigCode\footnote{\url{https://www.bigcode-project.org}} and evaluated on the HumanEval dataset~\cite{chen2021codex} by zero-shot learning. HumanEval is used to measure functional correctness for synthesizing programs from docstrings. It consists of 164 original programming problems, assessing language comprehension, algorithms, and simple mathematics, with some comparable to simple software interview questions. 
We set the temperature to 0.8 during the evaluation.

We report pass@k~\cite{kulal2019spoc} in Table~\ref{tab:gpt} to  evaluate the functional correctness, where $k$ code samples are generated per problem, a problem is considered solved if any sample passes the unit tests and the total fraction of problems solved is reported.  We can observe that on GPT-2, Adan surpasses its default Adam optimizer in terms of pass@k within only half of the pre-training steps, which implies that Adan has a much larger potential in training LLMs with fewer computational costs.
\textbf{For more comprehensive results on LLMs, please refer to Appendix Sec.~\ref{sec:llm}}.

\subsubsection{Results on Transformer-XL}
Here we investigate the performance of Adan  on Transformer-XL~\cite{dai2019transformer} which is often used to model long sequences.
We follow the exact official setting to train Transformer-XL-base on the WikiText-103 dataset that is the largest available word-level language modeling benchmark with long-term dependency. 
We only replace the default Adam optimizer of Transformer-XL-base by our Adan, and do not make other changes for the hyper-parameter.  
For Adan, we set $\beta_1 = 0.1, \beta_2 = 0.1$, and $\beta_3 = 0.001$, and choose learning rate as 0.001. 
 We test Adan and Adam with several training steps, including 50k, 100k, and 200k (official), and report the results in Table~\ref{tab:XLNet}. 

 From Table~\ref{tab:XLNet}, one can observe that on Transformer-XL-base, Adan surpasses its default Adam optimizer in terms of test PPL (the lower, the better) under all training steps.  Surprisingly, Adan using 100k training steps can even achieve comparable results to Adam with 200k training steps. All these results demonstrate the superiority of Adan over the default SoTA Adam optimizer in Transformer-XL.

\subsection{Results on Reinforcement Learning Tasks}\label{RL}
Here we evaluate Adan on reinforcement learning tasks.  
Specifically, we replace the default Adam optimizer in PPO~\cite{duan2016benchmarking} , which is one of the most popular policy gradient methods, without making any other changes to PPO. For brevity, we call this new PPO version ``PPO-Adan".  Then we test PPO and  PPO-Adan on several games which are actually continuous control environments simulated by the standard and widely-used engine,  MuJoCo~\cite{todorov2012mujoco}.  For these test games,  their agents receive a reward at each step.  Following standard evaluation, we run each game under 10 different and independent random seeds (\ie~1 $\sim$ 10), and test the performance for {10 episodes every 30,000 steps.}  All these experiments are based on the widely used codebase Tianshou~\cite{tianshou}.  
For fairness, we use the default hyper-parameters in Tianshou, \eg~batch size, discount, and GAE parameter.  
We use Adan with its default $\beta$'s ($\beta_1 = 0.02, \beta_2 = 0.08$, and $\beta_3 = 0.01$).
Following the default setting, we do not adopt the weight decay and choose the learning rate as 3e-4.

We report the results on four test games in Figure~\ref{fig:rl}, in which the solid line denotes the averaged episodes rewards in  evaluation and the shaded region is its 75\% confidence intervals. From  Figure~\ref{fig:rl}, one can observe that on the four test games, PPO-Adan achieves much higher rewards  than vanilla PPO which uses Adam as its optimizer.  These results demonstrate the advantages of Adan over Adam since PPO-Adan simply replaces the Adam optimizer in PPO with our Adan and does not make other changes.

\subsection{Results on Graph Neural Networks}
To further assess the effectiveness of the Adan optimizer across different network architectures, this section focuses on graph neural networks using the Open Graph Benchmark (OGB)~\cite{hu2020open}. OGB encompasses several challenging large-scale datasets. Consistent with the settings used in DeepGCN~\cite{li2019deepgcns,li2020deepergcn}, our experiments were conducted on the ogbn-proteins dataset, optimizing the node feature prediction task at the level of the optimizer. The ogbn-proteins dataset is an undirected, weighted graph, classified by species type, comprising $132,534$ nodes and $39,561,252$ edges. Each edge is associated with an 8-dimensional feature, and every node features an 8-dimensional binary vector representing the species of the corresponding protein. Given that the prediction task for ogbn-proteins in DeepGCN is multi-label, ROC-AUC was chosen as the evaluation metric. As demonstrated in Table~\ref{tab:GNN}, the Adan optimizer exhibits unique advantages in addressing the complex optimization challenges of graph convolutional networks. Particularly in the context of \emph{deep} graph neural networks, Adan efficiently and effectively manages learning rate adjustments and model parameter update directions, enabling the DeepGCN model to achieve superior performance on the test dataset. 

\begin{table}[t!]
\centering
\setlength{\tabcolsep}{10.0pt} 
\renewcommand{\arraystretch}{3.6}
\caption{Comparison of ROC-AUC metrics for the DeepGCN graph neural network on the ogbn-proteins dataset.}\label{tab:GNN}
{\fontsize{8.5}{3}\selectfont{
\begin{tabular}{l|ccc}
\toprule
\multicolumn{1}{c|}{DeepGCN~\cite{li2020deepergcn}} & \multicolumn{2}{c}{Epochs}                                            \\
\multicolumn{1}{c|}{layer=24, channel=64}                                     & \multicolumn{1}{c}{500} & \multicolumn{1}{c}{$1,000$} \\ \midrule
Adam (official)                             & 0.812                  & 0.826                             \\
\textbf{Adan}                                               & \textbf{0.828}                    & \textbf{0.831}                    \\ \bottomrule
\end{tabular}
}}
 \end{table}

\section{Conclusion}
In this paper, to relieve the plague of trying different optimizers for different deep network architectures, we propose a new deep optimizer, Adan. 
We reformulate the vanilla AGD to a more efficient version and use it to estimate the first- and second-order moments in adaptive optimization algorithms. We prove that the complexity of Adan matches the lower bounds and is superior to those of other adaptive optimizers.  
Finally, extensive experimental results demonstrate that  Adan consistently surpasses other optimizers on many popular backbones and frameworks, including ResNet, ConvNext, ViT, Swin,  MAE-ViT, LSTM, Transformer-XL, BERT, and GPT-2.

\section{Acknowledge}
Z. Lin was supported by National Key R\&D Program of China (2022ZD0160300), the NSF China (No. 62276004), and Qualcomm. Pan Zhou was supported by the Singapore Ministry of Education (MOE) Academic Research Fund (AcRF) Tier 1 grant.

\ifCLASSOPTIONcaptionsoff
  \newpage
\fi
\bibliographystyle{IEEEtran}
\bibliography{adan_bib}
\begin{IEEEbiography}[{\includegraphics[width=1in,height=1.25in,clip,keepaspectratio]{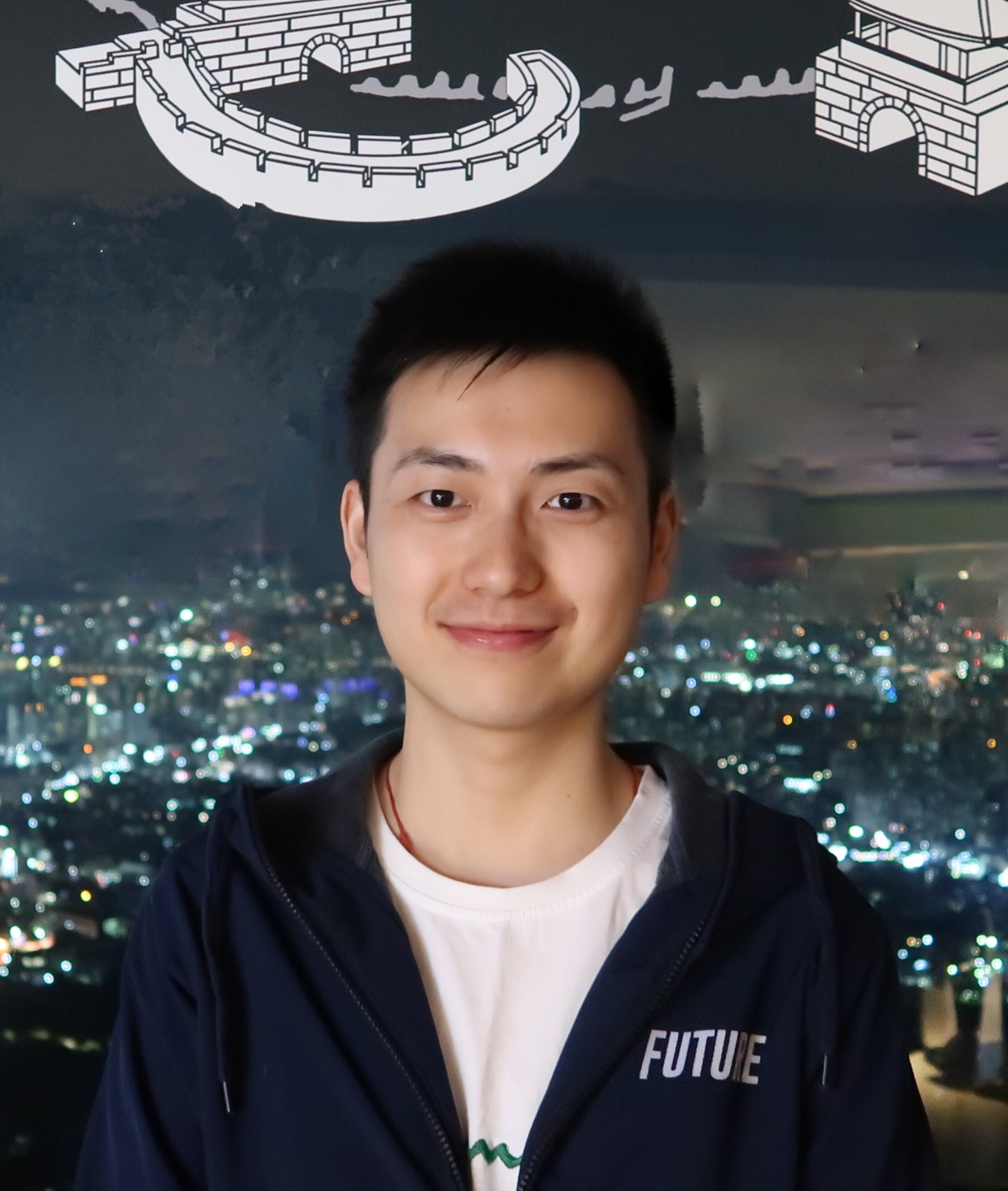}}]{Xingyu Xie}
received his Ph.D. degree from Peking University, in 2023. He is currently a Research Fellow at the Department of Mathematics, National University of Singapore. His current research interests include large-scale optimization and deep learning.
\end{IEEEbiography}

\begin{IEEEbiography}[{\includegraphics[width=1in,height=1.25in,clip,keepaspectratio]{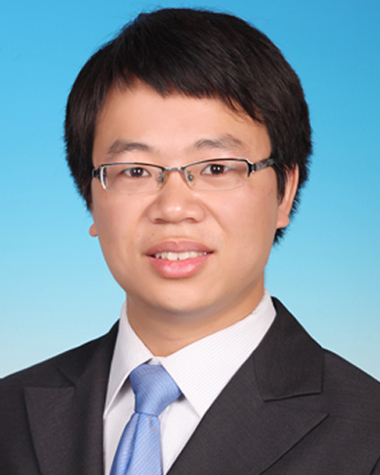}}]{Pan Zhou} received Master Degree at Peking University in 2016 and obtained Ph.D. Degree at National University of Singapore in 2019. Now
he is an assistant professor at Singapore Management University, Singapore. Before he also worked as a research scientist at Salesforce and
Sea AI Lab, Singapore. His research interests include computer vision, machine learning, and optimization. He was the winner of the Microsoft Research Asia Fellowship 2018.
\end{IEEEbiography}

\begin{IEEEbiography}[{\includegraphics[width=1in,height=1.25in,clip,keepaspectratio]{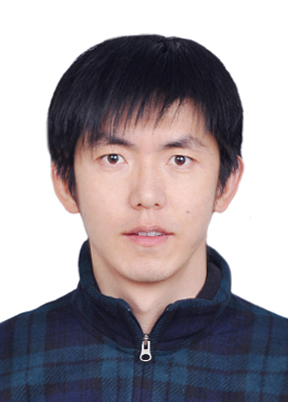}}]{Huan Li} 
received his Ph.D. degree from Peking University, in 2019. He is currently an Assistant Researcher at the Institute of Robotics and Automatic Information Systems, Nankai University. His current research interests include optimization and machine learning.
\end{IEEEbiography}

\begin{IEEEbiography}[{\includegraphics[width=1in,height=1.25in,clip,keepaspectratio]{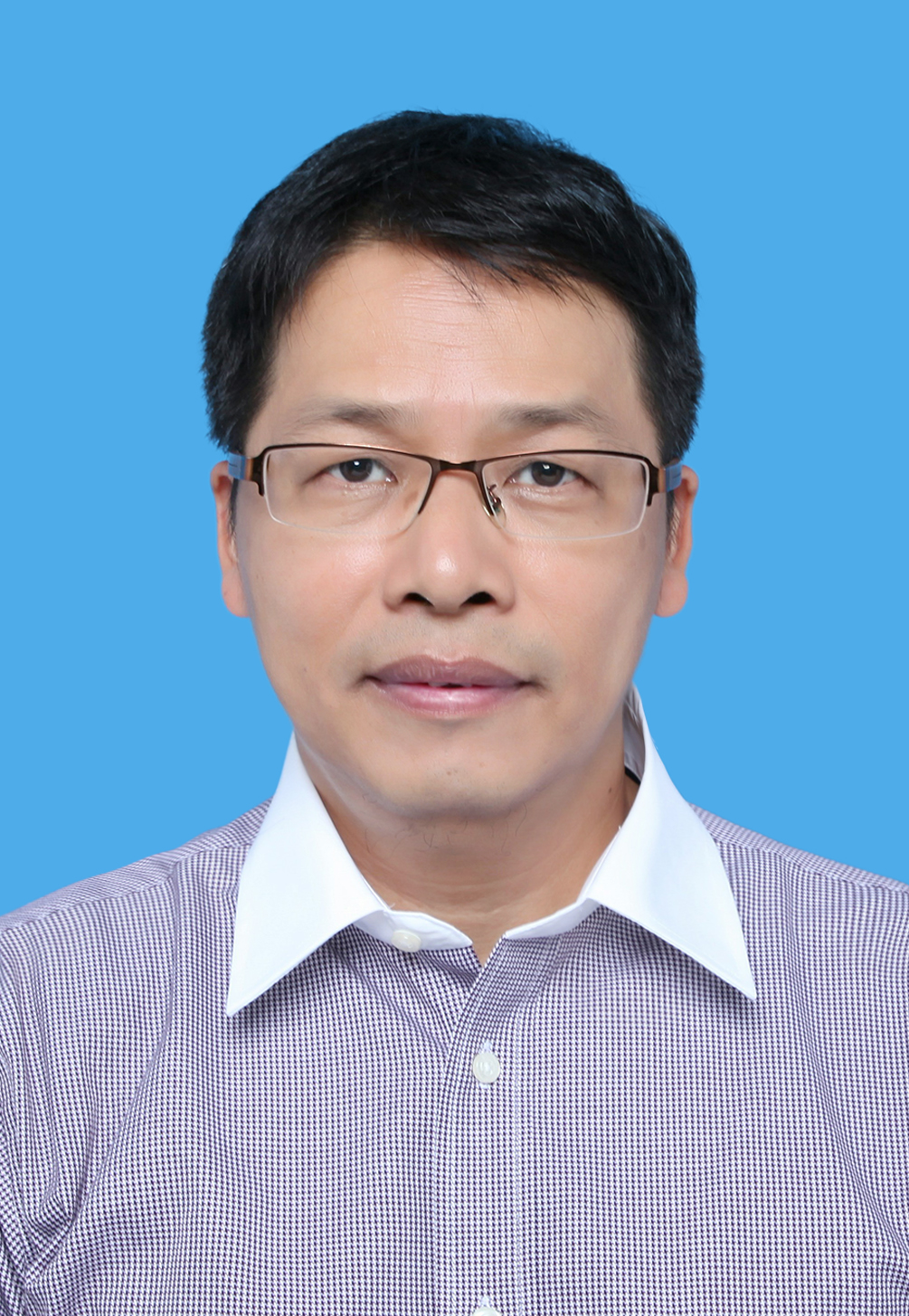}}]{Zhouchen Lin} (M’00–SM’08–F’18) received the Ph.D. degree in applied mathematics from Peking University in 2000. He is currently a Boya Special Professor with the State Key Laboratory of General Artificial Intelligence, School of Intelligence Science and Technology, Peking University. His research interests include machine learning and numerical optimization. He has published over 310 papers, collecting more than 35000 Google Scholar citations. He is a Fellow of the IAPR, the IEEE, the AAIA and the CSIG.
\end{IEEEbiography}

\begin{IEEEbiography}[{\includegraphics[width=1in,height=1.25in,clip,keepaspectratio]{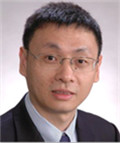}}]{Shuicheng Yan} is currently the Managing Director of Kunlun 2050 Research and Chief Scientist of Kunlun Tech \& Skywork AI, and the former Group Chief Scientist of Sea Group.
He is a Fellow of Singapore's Academy of Engineering, AAAI, ACM, IEEE, and IAPR. His research areas include computer vision, machine learning, and multimedia analysis. Till now, Prof Yan has published over 800 papers at top international journals and conferences, with an H-index of 140+. He has also been named among the annual World's Highly Cited Researchers nine times.
\end{IEEEbiography}

\clearpage
\onecolumn
\appendices
\begin{center}
\text{\huge{\textbf{Adan: Adaptive Nesterov Momentum Algorithm}}
}
\\
\vspace{1mm}
\text{\huge{\textbf{for Faster Optimizing Deep Models}}
}
\\
\vspace{1mm}
\text{\huge{(Supplementary Material)}
}
\end{center}
The supplementary contains some additional experimental results and the technical proofs of the paper entitled ``Adan: Adaptive Nesterov Momentum Algorithm for Faster Optimizing Deep Models''. It is structured as follows. 

Sec.~\ref{sec:diss} discuss why the lower bound for the convergence complexity is $\Omega(\epsilon^{-3.5})$ instead of $\Omega(\epsilon^{-3.0})$. And we also compare more about the constants in the convergence bounds of various optimizers in this section.

Sec.~\ref{sec:add-exp} includes the implementation details and additional experimental results, which contain more detailed results on ViTs for many representative optimizers in Sec.~\ref{sec:vit-more} and the ablation study in Sec.~\ref{sec:robustness}.

Sec.~\ref{sec:AGDII} provides the proof of the equivalence between AGD and reformulated AGD, \ie, the proof of Lemma \ref{lem:equivalence}.
And then, given Lipschitz gradient condition, Sec.~\ref{sec:first-order} provides the convergence analysis in Theorem \ref{theorem1}.
Next, we show Adan's faster convergence speed with Lipschitz Hessian condition in Sec.~\ref{sec:second-order}, by first reformulating our Algorithm \ref{alg:Name} and introducing some auxiliary bounds. 
Finally, we present some auxiliary lemmas in Sec.~\ref{sec:auxiliary}.
\section{Discussion on Convergence Results}\label{sec:diss}
\subsection{Discussion about Lower Bound}
For the lower bound, as proven in~\cite{arjevani2020second}, on the nonconvex problems with Lipschitz gradient and Hessian,  for stochastic gradient-based methods with {1) unbiased and 
variance-bounded stochastic gradient and 2) stochastic gradient queried on the same point per iteration}, their complexity lower bound is  $\Omega(\epsilon^{-3.5})$ to find an $\epsilon$-accurate first-order stationary point.  For condition 2), it means that per iteration, the algorithm only queries the stochastic gradient at one point (e.g., SGD, Adam, Adan) instead of multiple points (variance-reduced algorithms, e.g. SVRG~\cite{johnson2013accelerating}). Otherwise, the   complexity  lower bound  becomes  $\Omega(\epsilon^{-3.0})$~\cite{arjevani2020second}.

For the nonconvex problems with Lipschitz gradient \emph{but without} Lipschitz Hessian, the complexity lower bound  is $\Theta(\epsilon^{-4})$ as shown in~\cite{arjevani2019lower}.  
Note, the above  Lipschitz gradient and Hessian assumption are defined on the training loss w.r.t. the variable/parameter instead of w.r.t. each datum/input $\zeta$. 
We would like to clarify that our proofs are only based on the above Lipschitz gradient and Hessian  assumptions and do not require  the Lipschitz gradient and Hessian w.r.t. the input $\zeta$. 

\subsection{Discussion about Convergence Complexity}
The constant-level difference among the complexities of compared optimizers is not incremental.
Firstly, under the corresponding assumptions,  most  compared optimizers already achieve the optimal  complexity in terms of the dependence on optimization accuracy  $\epsilon$, and their complexities only differ  from their  constant factors, \eg $c_2$, $c_\infty$ and $d$.  For instance,  with Lipschitz gradient but without Lipschitz Hessian,  most optimizers have complexity $\order{\frac{x}{\epsilon^{4}}}$ which matches the lower bound $\order{\frac{1}{\epsilon^{4}}}$ in~\cite{arjevani2019lower}, where the  constant factor $x$  varies from different optimizers, \eg $x=c_{\infty}^2 d $ in Adam-type optimizer, $x=c_2^6$ in Adabelief,   $x = c_2^2d$  in LAMB, and $x=c_{\infty}^{2.5}$ in Adan.  So under the same conditions, one cannot improve  the complexity dependence on $\epsilon$ but can improve the constant factors which, as discussed below, is still significant, especially for  DNNs. 
	
Secondly, the constant-level difference may cause very different complexity whose magnitudes vary by several orders on networks.  This is because 1) the modern network is often large, e.g. 11 M parameters in the small ReNet18, leading a very large $d$; 2) for network gradient, its $\ell_2$-norm upper bound  $c_2$ is often much larger than its $\ell_\infty$-norm upper bound $c_\infty$ as observed and  proved in some work~\cite{du2018algorithmic},  because the stochastic  algorithms can probably adaptively adjust the parameter magnitude at different layers so that these parameter magnitudes are balanced.  
	
Actually, we also empirically find $c_\infty = \order{8.2}, c_2 = \order{430}, d = 2.2\times 10^{7}$ in the  ViT-small  across different optimizers, e.g., AdamW, Adam, Adan, LAMB.  In the extreme case, under the widely used  Lipschitz gradient assumption,  the complexity bound of Adan is $7.6\times 10^{6}$ smaller than the one of Adam, $3.3\times 10^{13}$ smaller than the one of AdaBlief, $2.1\times 10^{10}$ smaller than the one of LAMB, \etc.  For ResNet50, we also observe $c_\infty = \order{78}, c_2 = \order{970}, d = 2.5\times 10^{7}$ which also means a large big improvement of Adan over other optimizers.

\section{Additional Experimental Results}\label{sec:add-exp}

\subsection{Pre-training Results on LLMs}\label{sec:llm}
To investigate the efficacy of the Adan optimizer for large-scale language tasks, we conducted pre-training experiments using MoE models based on the architecture specified in a recent study~\cite{jiang2024mixtral}. Our experiments were designed as training from scratch, a method known for its significant computational costs. This approach was selected to rigorously assess the optimizer's performance under demanding conditions. The experiments utilized the RedPajama-v2 dataset~\cite{together2023redpajama} with three configurations, each consisting of 8 experts: $8\times0.1$B (totaling 0.5B trainable parameters), $8\times0.3$B (2B trainable parameters), and $8\times0.6$B (4B trainable parameters). These models were trained with sampled data comprising 10B, 30B, 100B, and 300B tokens, respectively.
In line with conventional practices for LLMs, our training protocol processed each data point exactly once. This approach, typical for evaluating optimizer performance, aligns training loss with validation loss, providing a clear measure of efficiency.

The results, as summarized in Table~\ref{tab:moe}, indicate that Adan consistently outperforms the AdamW optimizer across all configurations and data volumes. This improvement underscores Adan's capacity for efficient parameter updates and its utility in large-scale distributed training setups.

\subsection{Detailed Comparison on ViTs}\label{sec:vit-more}
Besides AdamW, we also compare Adan with several other popular optimizers, including Adam, SGD-M, and LAMB, on ViT-S.  Table \ref{tab:vit-s} shows that SGD, Adam, and LAMB perform poorly on ViT-S, which is also observed in the  works~\cite{xiao2021early, nado2021large}. 
These results demonstrate that  the decoupled weight decay in Adan and AdamW is much more effective  than 1) the vanilla weight decay, namely the commonly used $\ell_2$ regularization in SGD, and 2) the one without any  weight decay, since as shown in Eqn.~\eqref{decoupleproblem},  the decoupled weight decay is a dynamic regularization along the training trajectory and could better regularize the loss. 
Compared with AdamW, Adan's advantages mainly come from its faster convergence speed. This empirical evidence solidifies Adan as a superior choice for training ViTs, particularly when rapid convergence is essential. 

\begin{table}[t]
\caption{Comparison of training loss for MoE with different data volumes and model sizes using Adan and AdamW.} \label{tab:moe}
 		\centering
 		\setlength{\tabcolsep}{12.0pt} 
 		\renewcommand{\arraystretch}{3.2}
 	{ \fontsize{8.5}{3}\selectfont{
\begin{tabular}{l|ccc|ccc|c}
\toprule 
Model Size & \multicolumn{3}{c|}{8 $\times$ 0.1B}                       & \multicolumn{3}{c|}{8 $\times$ 0.3B}                       & 8 $\times$ 0.6B        \\ \midrule
Token Size & 10B            & 30B            & 100B           & 30B            & 100B           & \textbf{300B}           & \textbf{300B}           \\ \midrule
AdamW      & 2.722          & 2.550          & 2.427          & 2.362          & 2.218          & 2.070          & 2.023          \\ \midrule
Adan       & \textbf{2.697} & \textbf{2.513} & \textbf{2.404} & \textbf{2.349} & \textbf{2.206} & \textbf{2.045} & \textbf{2.010}\\
\bottomrule
\end{tabular}
 	}}
\end{table}

\begin{table}[t]
 	\caption{ Top-1 ACC. (\%) of different optimizers for ViT-S  on ImageNet trained  under training setting II.  * is  from~\cite{touvron2021training}.} \label{tab:vit-s}
 		\centering
 		\setlength{\tabcolsep}{15.0pt} 
 		\renewcommand{\arraystretch}{3.0}
 	{ \fontsize{8.3}{3}\selectfont{
\begin{tabular}{l|cccc}
	\toprule 
	Epoch              & 100  & 150     & 200     & 300     \\ \midrule
	AdamW~\cite{loshchilov2018decoupled} (default)   & 76.1 & 78.9    & 79.2    & 79.9$^*$    \\  
	Adam~\cite{kingma2014adam} & 62.0 & 64.0 &  64.5 & 66.7 \\
        Adai~\cite{xie2022adaptive} & 66.4 & 72.6 &  75.3 & 77.4 \\
	SGD-M~\cite{nesterov1983method,nesterov1988approach,nesterov2003introductory} & 64.3 & 68.7    & 71.4    & 73.9 \\ 
	
	LAMB~\cite{you2019large}            & 69.4 & 73.8    & 75.9    & 77.7 \\ 
	
	\textbf{Adan (ours)}               & \textbf{77.5} & \textbf{79.6}    & \textbf{80.0}    & \textbf{80.9}    \\
	\bottomrule
\end{tabular}
 	}}
\end{table}

\begin{table}[t!]
 	\caption{A comparison of peak memory and wall duration on \textbf{single NVIDIA A800 GPU} for different models. The duration time is the total time of 200 iteration steps.} \label{tab:single-time}
 		\centering
 		\setlength{\tabcolsep}{10.0pt} 
 		\renewcommand{\arraystretch}{3.5}
 	{ \fontsize{8.3}{3}\selectfont{
\begin{tabular}{l|c|cccc|cccc}
	\toprule 
 \multirow{2}{*}{Model} &  Model 	&   \multicolumn{4}{c|}{100 Steps Time (ms)}   &  \multicolumn{4}{c}{Peak Memory (GB)}\\
 \cline{3-10}
	& Size &Adan & {AdamW} & LAMB & AdaBelief & Adan &  {AdamW} & LAMB & AdaBelief \\   \midrule
   ResNet-50 & 25M & 127.6 & 127.5 & 154.2 & 130.4 & 13.8 & 13.8 & 13.8 & 13.8 \\
  ResNet-101 & 44M & 211.3 & 207.1 & 251.1 & 214.4 & 19.5 & 19.4 & 19.4& 19.4 \\
 ViT-B & 86M & 229.8 & 225.8 & 252.3 & 229.3 & 17.8 & 17.2 & 17.2 & 17.2 \\
Swin-B & 87M & 454.3 & 443.4 & 495.1 & 454.5 & 32.2 & 31.5 & 31.5 & 31.5\\
 ConvNext-B & 88M & 509.0 & 508.1 & 562.5 & 517.2 & 33.7 & 33.7 & 33.7 & 33.7 \\
 Swin-L & 196M & 706.1 & 695.9& 747.8 & 705.6 & 49.6 & 47.4 & 47.4 & 47.4 \\
ConvNext-L & 197M & 804.0 & 793.6 & 849.3 & 802.5 & 50.4 & 50.4 & 50.4 & 50.4 \\
 ViT-L & 304M & 700.1 & 684.6 & 728.6 & 691.2& 48.1 & 45.8 & 45.8 & 45.8 \\
 GPT-2 & 670M &  641.2 & 606.1 & 638.7 & 617.3 & 67.7 & 62.8 & 62.8 & 62.8\\
  GPT-2 & 1024M & 746.0 & 683.9 & 737.1 & 710.7 & 78.6 & 71.9 & 71.9 & 71.9 \\
	\bottomrule
\end{tabular}
 	}}
\end{table}

\begin{table*}[t!]
	\begin{center}
		\caption{Training speed (tokens/s on each GPU) investigation of different optimizers in prevalent Megatron-LM framework for efficient multi-node LLMs training with different model sizes and GPU number.}
		\label{tab:mul-gpu}
		\setlength{\tabcolsep}{10pt} 
		\renewcommand{\arraystretch}{3.5}
		{ \fontsize{8.3}{3.2}\selectfont{
				\begin{tabular}{l|cccc|cccc}
					\toprule
			32$\times$ NVIDIA A800 	& \multicolumn{4}{c|}{Speed (tokens/s/GPU)} & \multicolumn{4}{c}{Peak Memory (GB)}    \\ \cline{2-9} 
					Model   & Adan  & AdamW   & LAMB   & AdaBelief & Adan  & AdamW   & LAMB & AdaBelief  \\ \midrule
				MoE (8 $\times$ 0.1B)  & 58644.6 & 58369.4 &  58488.3 &  58698.5 & 19.6 & 19.5 & 19.5 & 19.5   \\
    MoE (8 $\times$ 0.3B)  & 24123.2 & 23872.2 &  24018.3 &  24007.5  & 49.3 & 49.0 & 49.0 & 49.0   \\
                \bottomrule
16$\times$ NVIDIA A800 	& \multicolumn{4}{c|}{Speed (tokens/s/GPU)} & \multicolumn{4}{c}{Peak Memory (GB)}    \\ \cline{2-9} 
					Model   & Adan  & AdamW   & LAMB   & AdaBelief & Adan  & AdamW   & LAMB & AdaBelief  \\ \midrule
				MoE (8 $\times$ 0.1B)  & 63073.8 & 62933.5 &  63024.9 &  62835.3 & 20.0 & 19.8 & 19.8 & 19.8   \\
    MoE (8 $\times$ 0.3B)  & 24953.9  & 24961.4 &  24897.7 &  24924.4 & 50.8 & 49.8 & 49.8 & 49.8  \\
                \bottomrule 
\end{tabular}		
		}}
	\end{center}
\end{table*}

\subsection{Memory and Computation Time Comparison for Single Step}
To comprehensively validate the computational efficiency and memory usage of the Adan optimizer, we conduct a detailed analysis of its performance during single and distributed multi-GPU training setups. The detailed results are presented in Table~\ref{tab:single-time} and Table~\ref{tab:mul-gpu} for single and multiple GPU setups, respectively.

\textbf{Single GPU Analysis}: 
On a single GPU, we evaluated Adan across a diverse set of over $10$ different models, including both CNNs and transformers. For this experiment, peak memory usage and computational time were recorded over $200$ training iterations. Despite Adan's slightly increased computational complexity, the time differences were negligible. This minimal impact on timing can be attributed to the highly parallel nature of GPU computations. Independent calculations, such as those required for maintaining and computing the gradient difference in Adan, are efficiently parallelized, effectively 'hiding' any added computational cost under normal GPU operation loads.

From a memory standpoint, small models showed little difference in peak memory usage. This consistency is largely due to PyTorch's memory management, which includes preemptive reservation of memory blocks to accommodate sudden demands from user codes. For instance, although AdamW and Adan might use 768 MB and 900 MB respectively, PyTorch often rounds these up to the nearest whole memory page, such as 1024 MB. This effect is more pronounced in smaller models. However, as model sizes increase to a point where single-page memory reservations are insufficient, PyTorch's dynamic memory allocation starts to work, which could lead to small observable differences in memory usage. Nonetheless, forward pass activations generally govern peak memory demands, and the additional memory required by Adan does not significantly exacerbate these peak demands.

\textbf{Multi-GPU Distributed Training}: 
In a more complex multi-GPU setting, where we employ an 8-expert MoE LLM architecture~\cite{jiang2024mixtral} with each expert having $0.1$ billion parameters (totaling 1.3 billion parameters), we observed small differences in both time and memory across GPUs. This can be attributed to the distribution of optimizer states across multiple GPUs, which minimizes the impact of any single GPU’s additional memory load. Furthermore, the slight increase in computation due to Adan's operations is marginal compared to the substantial computations involved in forward and backward propagation, as well as communication overheads between GPUs.

Overall, regardless of model size, the additional overhead Adan introduced by Adan is minimal regardless of single or multi-GPU settings. However, the performance enhancements it provides are significant and cannot be overlooked. 

\subsection{Implementation Details of Adan}
For fairness, in all experiments, we only replace the optimizer with Adan and tune the  step size, warm-up epochs, and weight decay while fixing the other hyper-parameters, \eg~data augmentation, $\epsilon$ for adaptive optimizers, and  model parameters. Moreover, to make Adan simple, in all experiments except Table~\ref{tab:restart} in Sec.~\ref{restart}, we do not use the restart strategy.  
For the large-batch training experiment, we use the sqrt rule  to scale the learning rate:  $\text{lr} \!=\! \sqrt{\frac{\text{batch size}}{256} }\times 6.25$e-3, and respectively set warmup epochs $\{20,40,60,100,160,200\}$ for batch size $\text{bs}=\{1k,2k,4k,8k,16k,32k\}$.
For other remaining experiments, we use the hyper-parameters:  learning rate $1.5$e-2 for ViT/Swin/ResNet/ConvNext and MAE fine-tuning,  and $2.0$e-3 for MAE pre-training according to the official settings. 
We set $\beta_1 = 0.02, \beta_2 = 0.08$ and $\beta_3 = 0.01$, and let weight decay be $0.02$ unless noted otherwise. 
We clip the global gradient norm to $5$ for ResNet and do not clip the gradient for ViT, Swin, ConvNext, and {MAE}.
We utilize the de-bias strategy for Adan to keep consistent with Adam-type optimizers.

\begin{figure*}[t]
\centering
	\includegraphics[width=0.8\textwidth]{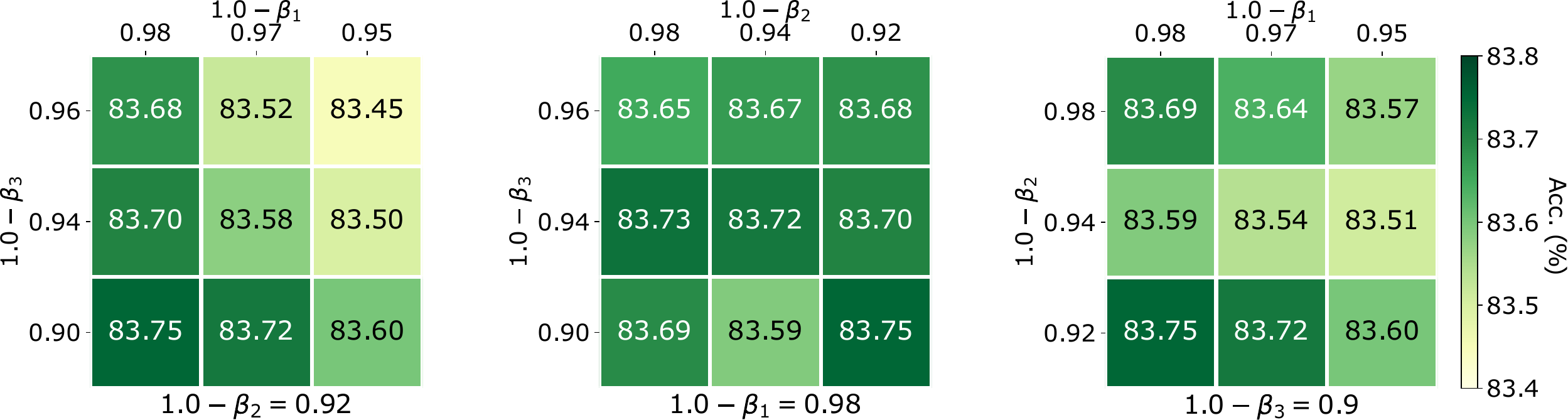}
		\caption{Effects of momentum coefficients $(\beta_1,\beta_2,\beta_3)$ to top-1 accuracy (\%) of Adan on ViT-B under  MAE  training framework (800 pretraining and  100  fine-tuning epochs on ImageNet). 
}\label{fig:MAE}
\end{figure*}

\begin{table}[t]
 	\caption{Ablation study examining the training loss (measured by next token prediction loss where lower is better) across different moments' parameters in an MoE architecture (8 $\times$ 0.1B) using 3B tokens from RedPajama dataset. For both the Adan and AdamW optimizers, the first-order momentum is represented as $1-\beta_1$. The second-order momentum is represented as $1-\beta_3$ for Adan and $1-\beta_2$ for AdamW.} \label{tab:betas}
 		\centering
 		\setlength{\tabcolsep}{15.0pt} 
 		\renewcommand{\arraystretch}{3}
 	{ \fontsize{8.3}{3}\selectfont{
\begin{tabular}{l|cc|c|c}
	\toprule 
	Optimizer              & \begin{tabular}[c]{@{}c@{}}First-order \\ Momentum\end{tabular}   &  \begin{tabular}[c]{@{}c@{}}Second-order \\ Momentum\end{tabular}     & Training Loss   & Variance    \\ \midrule
 AdamW & 0.90&0.95 & 2.859 & \multirow{6}{*}{1.73e-3}\\
 AdamW & 0.94&0.95 & \textbf{2.847}\\
 AdamW & 0.98&0.95 & 2.972 \\
 AdamW & 0.90&0.99 & 2.888\\
 AdamW & 0.94&0.99 & 2.870\\
 AdamW & 0.98&0.99 & 2.911\\
 \midrule
 Adan & 0.90&0.95 & 2.845 & \multirow{6}{*}{1.29e-3}\\
 Adan & 0.94&0.95 & \textbf{2.835}\\
 Adan & 0.98&0.95 & 2.857 \\
 Adan & 0.90&0.99 & 2.861\\
 Adan & 0.94&0.99 & 2.870\\
 Adan & 0.98&0.99 & 2.858\\
	\bottomrule
\end{tabular}
 	}}
\end{table}

\begin{table}[t!]
 	\caption{Ablation study investigating the components of the Adan optimizer. We report the training loss, measured by next token prediction where a lower score indicates better performance, for an MoE architecture (8 $\times$ 0.1B) trained using 10B tokens from the RedPajama dataset.} \label{tab:component}
 		\centering
 		\setlength{\tabcolsep}{10.0pt} 
 		\renewcommand{\arraystretch}{3}
 	{ \fontsize{8.3}{3}\selectfont{
\begin{tabular}{l|cccc|c|c}
	\toprule 
	Optimizer              & \begin{tabular}[c]{@{}c@{}}Heavy-ball \\ Acceleration\end{tabular}
  &  \begin{tabular}[c]{@{}c@{}}Nesterov \\ Acceleration\end{tabular}     & \begin{tabular}[c]{@{}c@{}}Weight Decay\\ by Proximation\end{tabular}       & Restart &  \begin{tabular}[c]{@{}c@{}}Training \\ Loss\end{tabular} & Improvement   \\ \midrule
	AdamW   & \CheckmarkBold & \XSolidBrush   & \XSolidBrush    & \XSolidBrush & 2.646  & ---   \\  
	Adan   & \XSolidBrush & \CheckmarkBold   & \XSolidBrush    & \XSolidBrush & 2.630  & 0.016  \\
        Adan   & \XSolidBrush & \CheckmarkBold   & \CheckmarkBold    & \XSolidBrush & 2.628  & 0.018  \\
	 Adan   & \XSolidBrush & \CheckmarkBold   & \CheckmarkBold    & \CheckmarkBold & 2.622  & 0.024  \\
	\bottomrule
\end{tabular}
 	}}
\end{table}

\subsection{Ablation Study}\label{sec:robustness}

\begin{table*}[t!]
	\begin{minipage}[c]{.48\linewidth}
	\caption{Top-1 accuracy (\%) of ViT-S  on ImageNet trained under Training Setting I and II. $*$   is  reported in~\cite{touvron2021training}.} \label{tab:setting}
		\centering
		\setlength{\tabcolsep}{4.2pt} 
		\renewcommand{\arraystretch}{3.5}
		{ \fontsize{8.3}{3}\selectfont{
\begin{tabular}{c|cc|cc}
	\toprule 
Training 	&   \multicolumn{2}{c|}{Training Setting I}   &  \multicolumn{2}{c}{Training Setting II} \\
epochs	& AdamW~\cite{loshchilov2018decoupled} & {Adan} & AdamW~\cite{loshchilov2018decoupled} &  {Adan} \\ \midrule 
	150  & 76.4 &  \textbf{80.2} & 78.3 & \textbf{79.6} \\
	300  & 77.9 & \textbf{81.1} &~79.9$^*$ & \textbf{80.7}  \\
	\bottomrule
\end{tabular}
		}}
	\end{minipage}
	\hspace{1.3cm}
	\begin{minipage}[c]{.42\linewidth}
\caption{Top-1 accuracy (\%) of ViT-S and ConvNext-T  on ImageNet under  Training Setting II trained by 300 epochs.
} \label{tab:restart}
		\centering
		\setlength{\tabcolsep}{3.2pt} 
		\renewcommand{\arraystretch}{3.9}
		{ \fontsize{8.3}{3}\selectfont{
\begin{tabular}{l|c|c}
	\toprule 
	& \multicolumn{1}{l|}{ViT Small {}} & \multicolumn{1}{l}{ConvNext Tiny } \\
	\midrule
	Adan w/o  restart   & 80.71                                             & 81.38                                      \\
	Adan w/ restart & \textbf{80.87}                                             & \textbf{81.62}   \\                          
	\bottomrule
\end{tabular}
		}}
	\end{minipage}%
\end{table*}

\subsubsection{Robustness to in momentum coefficients}
Here we choose MAE to investigate the effects of the momentum coefficients ($\beta$s) to Adan, since as shown in MAE, its pre-training is actually  sensitive to momentum coefficients of AdamW.  To this end,  following MAE, we  pretrain and fine tune ViT-B on ImageNet for 800 pretraining and 100  fine-tuning epochs.  We also fix one of $(\beta_1,\beta_2,\beta_3)$ and tune others.   
Figure~\ref{fig:MAE} shows that by only pretraining 800 epochs, Adan  achieves $83.7\%+$ in most cases and outperforms the official accuracy $83.6\%$ obtained by AdamW with 1600 pretraining epochs, indicating the robustness of Adan to $\beta$s. We also  observe  1) Adan is not sensitive to  $\beta_2$;  2)  $\beta_1$ has a certain impact on Adan, namely the smaller the $(1.0-\beta_1)$, the worse the accuracy; 
3) similar to findings of MAE, a small second-order coefficient $(1.0-\beta_3)$ can improve the accuracy.  The smaller the $(1.0-\beta_3)$, the more current landscape information the optimizer would utilize to adjust the coordinate-wise learning rate.  Maybe the complex pre-training task of MAE is preferred over local geometric information. 

In addition, we also conduct an ablation study on large-language models to assess the robustness of the Adan optimizer to variations in momentum coefficients. This study focus on examining the training loss, specifically next token prediction loss where a lower value indicates better performance, across different momentum parameters in a Mixture of Experts (MoE) architecture~\cite{jiang2024mixtral}. The architecture employed is an 8-expert head, each head with $0.1$ billion parameters, and the dataset used comprises 3 billion tokens from the RedPajama dataset~\cite{together2023redpajama}.
For this study, the first-order momentum coefficient for both Adan and AdamW optimizers was denoted as $1-\beta_1$, while the second-order momentum was represented as $1-\beta_3$ for Adan and $1-\beta_2$ for AdamW. 

The results are shown in Table~\ref{tab:betas}. Adan consistently achieved lower training loss compared to the default optimizer, AdamW, in nearly all cases tested. This indicates not only superior performance but also a lower sensitivity to fluctuations in the $\beta$ parameters. Notably, since the epoch is set to 1 for the language model under consideration, the training loss effectively represents the validation loss. The smaller variance in loss across different $\beta$'s settings with Adan further underscores its robustness to changes in these parameters, highlighting its suitability for large model training.

\subsubsection{Ablation Study on Adan's Components}
In efforts to understand the individual contributions of the components within the Adan optimizer, we conducted an ablation study focused on an LLM with a Mixture of Experts (MoE) architecture~\cite{jiang2024mixtral}. This study employed an 8-expert network, each with 0.1B parameters, trained using 10B tokens from the RedPajama dataset~\cite{together2023redpajama}. The objective was to measure the training loss, utilizing next token prediction as the metric, where a lower score signifies improved performance.

The results of this study are presented in Table~\ref{tab:component}. Notably, the most significant performance improvement is observed with our proposed reformulated Nesterov acceleration in Lemma~\ref{lem:equivalence}, which outperformed the heavy-ball acceleration employed by AdamW. The reduction in training loss with our reformulated Nesterov acceleration was $0.016$, a substantial enhancement compared to other components. The implementation of weight decay by proximation and the restart strategy yielded improvements of $0.002$ and $0.006$ in training loss, respectively.
It is important to note that, by default, Adan does not employ the restart strategy. This observation allows us to conclude that the primary contribution to Adan's performance enhancement stems from the use of the improved Nesterov acceleration. This finding further validates the significance of the Nesterov momentum component that we have introduced in our optimizer design.
\subsubsection{Robustness to Training Settings}
Many works~\cite{liu2021swin,liu2022convnet,touvron2022deit,wightman2021resnet,touvron2021training} often preferably chose LAMB/Adam/SGD for Training Setting I and  AdamW for Training Setting II. Table \ref{tab:setting} investigates Adan under both settings and shows its consistent improvement.  Moreover, one can also observe that  Adan under Setting I largely improves the accuracy  of Adan under Setting II. It actually surpasses the best-known accuracy $80.4\%$ on ViT-small in \cite{touvron2022deit} trained by advanced layer scale strategy  and stronger data augmentation.

\subsubsection{Discussion on Restart Strategy}\label{restart}
Here we investigate the performance Adan with and without restart strategy on ViT and ConvNext under 300 training epochs.  From the results in Table~\ref{tab:restart}, one can observe that  restart strategy slightly improves the test performance of Adan. Thus, to make our Adan simple and avoid hyper-parameter tuning of the restart strategy (e.g., restart frequency), in all experiments except  Table~\ref{tab:component} and Table~\ref{tab:restart}, we do not use this restart strategy.

\section{Technical Proofs}
We provide some notations that are frequently used throughout the paper.
The scale $c$ is in normal font. And the vector is in bold lowercase.
Give two vectors $\*x$ and $\*y$, $\*x\geq \*y$ means that $\qty(\*x-\*y)$ is a non-negative vector.
$\*x/\*y$ or $\frac{\*x}{\*y}$ represents the element-wise vector division.
$\*x \circ \*y$ means the element-wise multiplication, and
$\qty(\*x)^2 = \*x \circ \*x$.
$\innerprod{\cdot,\cdot}$ is the inner product.
Given a non-negative vector $\*n\geq 0$, we let $\norm{\*x}^2_{\sqrt{\*n}} \coloneqq \innerprod{\*x, \qty(\sqrt{\*n} + \varepsilon) \circ \*x}$.
Unless otherwise specified, $\norm{\*x}$ is the vector $\ell_2$ norm.
Note that $\E(\*x)$ is the expectation of random vector $\*x$.
For the functions $f(\cdot)$ and $g(\cdot)$, 
the notation $f(\epsilon)=\order{g(\epsilon)}$ means that $\exists a>0$, such that $ \frac{f(\epsilon)}{g(\epsilon)}\leq a, \forall \epsilon>0$.
The notation $f(\epsilon)=\Omega(g(\epsilon))$ means that $\exists a>0$, such that $\frac{f(\epsilon)}{g(\epsilon)}\geq a, \forall \epsilon>0$.
And $f(\epsilon)=\Theta(g(\epsilon))$ means that $\exists b\geq a>0$, such that $  a\leq \frac{f(\epsilon)}{g(\epsilon)}\leq b, \forall \epsilon>0$.

\subsection{Proof of Lemma \ref{lem:equivalence}: equivalence between the AGD and AGD II}\label{sec:AGDII}
In this section, we show how to get AGD II from AGD.
For convenience, we omit the noise term $\bm{\zeta}_k$.
Note that, let $\alpha \coloneqq 1-{\color{orange}\beta_1} $:
\[
\text{AGD:}
\left\{
\begin{aligned}
     & {\*g}_k = \nabla f(\bm{\theta}_{k} - \eta \alpha \*m_{k-1}) \\
     & \*m_k = \alpha \*m_{k-1} +  {\*g}_k \\
     &\bm{\theta}_{k+1} = \bm{\theta}_{k} - {\eta} \*m_k 
\end{aligned}
\right. .
\]
We can get:
\begin{equation}\label{eq:agd-reformulate}
    \begin{aligned}
     \bm{\theta}_{k+1} - \eta \alpha \*m_{k}  = &
     \bm{\theta}_{k} - {\eta} \*m_k - \eta \alpha \*m_{k}
      = \bm{\theta}_{k} - \eta \qty(1+\alpha)\qty(\alpha \*m_{k-1} +\nabla f(\bm{\theta}_{k} - \eta \alpha \*m_{k-1}))\\
      = & \bm{\theta}_{k} - \eta \alpha \*m_{k-1} - \eta \alpha^2 \*m_{k-1} - \eta\qty(1+\alpha)\qty(\nabla f(\bm{\theta}_{k} - \eta \alpha \*m_{k-1})).
\end{aligned}
\end{equation}
Let 
\[
\left\{
\begin{aligned}
& \Bar{\bm{\theta}}_{k+1} \coloneqq \bm{\theta}_{k+1} - \eta \alpha \*m_{k},\\
& \Bar{\*m}_{k} \coloneqq \alpha^2 \*m_{k-1} + (1+\alpha )\nabla f(\bm{\theta}_{k} - \eta \alpha \*m_{k-1}) = 
\alpha^2 \*m_{k-1} + (1+\alpha )\nabla f(\Bar{\bm{\theta}}_{k})\\
\end{aligned}
\right.
\]
Then, by Eq.\eqref{eq:agd-reformulate}, we have:
\begin{equation}\label{eq:AGDII-2}
  \Bar{\bm{\theta}}_{k+1} = \Bar{\bm{\theta}}_{k} - \eta \Bar{\*m}_{k}.  
\end{equation}
On the other hand, we have $\Bar{\*m}_{k-1} = 
\alpha^2 \*m_{k-2} + (1+\alpha )\nabla f(\Bar{\bm{\theta}}_{k-1})$ and :
\begin{equation}\label{eq:AGDII-1}
\begin{aligned}
\Bar{\*m}_{k} - \alpha \Bar{\*m}_{k-1}&  = 
\alpha^2 \*m_{k-1} + (1+\alpha )\nabla f(\Bar{\bm{\theta}}_{k})-\alpha \Bar{\*m}_{k-1}\\
& = (1+\alpha )\nabla f(\Bar{\bm{\theta}}_{k}) + \alpha^2 \qty(\alpha \*m_{k-2} + \nabla f(\Bar{\bm{\theta}}_{k-1}))-\alpha \Bar{\*m}_{k-1}\\
& = (1+\alpha )\nabla f(\Bar{\bm{\theta}}_{k}) + \alpha \qty(\alpha^2 \*m_{k-2} + \alpha \nabla f(\Bar{\bm{\theta}}_{k-1}) - \Bar{\*m}_{k-1}) \\
& = (1+\alpha )\nabla f(\Bar{\bm{\theta}}_{k}) + \alpha \qty(\alpha^2 \*m_{k-2} + \alpha \nabla f(\Bar{\bm{\theta}}_{k-1}))-\alpha \Bar{\*m}_{k-1}\\
& = (1+\alpha )\nabla f(\Bar{\bm{\theta}}_{k}) - \alpha \nabla f(\Bar{\bm{\theta}}_{k-1}) \\
& = \nabla f(\Bar{\bm{\theta}}_{k}) + \alpha\qty( \nabla f(\Bar{\bm{\theta}}_{k}) - \nabla f(\Bar{\bm{\theta}}_{k-1})). 
\end{aligned}
\end{equation}
Finally, due to Eq.\eqref{eq:AGDII-2} and Eq.\eqref{eq:AGDII-1}, we have:
\[
\left\{
\begin{aligned}
& \Bar{\*m}_{k} = \alpha \Bar{\*m}_{k-1} + \qty\Big(\nabla f(\Bar{\bm{\theta}}_{k}) + \alpha\qty( \nabla f(\Bar{\bm{\theta}}_{k}) - \nabla f(\Bar{\bm{\theta}}_{k-1})))\\
&\Bar{\bm{\theta}}_{k+1} = \Bar{\bm{\theta}}_{k} - \eta \Bar{\*m}_{k}
\end{aligned}
\right.
\]

\subsection{Convergence Analysis with Lipschitz Gradient}\label{sec:first-order}
We first provide several notations. 
Let ${F}_k(\bm{\theta})\coloneqq E_{\bm{\zeta}}[ f(\bm{\theta},\bm{\zeta})] + \frac{\lambda_k}{2}\norm{\bm{\theta}}_{\sqrt{\*n_k}}^2$, $f(\bm{\theta})\coloneqq E_{\bm{\zeta}}[ f(\bm{\theta},\bm{\zeta})] $, and $\mu\coloneqq {\sqrt{2\beta_3} c_\infty}/{\varepsilon}$,
\[
\norm{\*x}^2_{\sqrt{\*n_k}} \coloneqq \innerprod{\*x, \qty(\sqrt{\*n_k}+\varepsilon)\circ \*x}, \quad \lambda_k =  \lambda\qty(1-\mu)^k, \quad
\Tilde{\bm{\theta}}_k \coloneqq \qty(\sqrt{\*n_k}+\varepsilon) \circ {\bm{\theta}}_k.
\] 
{
\begin{lemma}\label{lem:prox}
Assume that $f(\bm{\theta})\coloneqq E_{\bm{\zeta}}[ f(\bm{\theta},\bm{\zeta})]$ is $L$-smooth.
For
\[
\bm{\theta}_{k+1} = \argmin_{\bm{\theta}} \qty(\frac{\lambda_k}{2}\norm{\bm{\theta}}_{\sqrt{\*n_k}}^2 + f(\bm{\theta}_k)+ \innerprod{\*u_k , \bm{\theta}-\bm{\theta}_k} + \frac{1}{2\eta}\norm{\bm{\theta}-\bm{\theta}_k}_{\sqrt{\*n_k}}^2).
\]
With $\eta \leq \min\{ \frac{\varepsilon}{3L},\frac{1}{10\lambda}\}$, define $\*g_k^{full} \coloneqq  \nabla f(\bm{\theta}_{k})$, then we have:
\[
{F}_{k+1}(\bm{\theta}_{k+1}) \leq {F}_k(\bm{\theta}_{k}) - \frac{\eta }{4c_\infty}\norm{\*u_k + \lambda_k \Tilde{\bm{\theta}}_k}^2 +  \frac{\eta}{2{\varepsilon}}\norm{\*g^{full}_k - \*u_k}^2.
\]
\end{lemma}
\begin{proof}
We denote $\*p_k \coloneqq \*u_k/\qty(\sqrt{\*n_k}+\varepsilon)$. 
By the optimality condition of $\bm{\theta}_{k+1}$, we have
\begin{equation}\label{eq:eta-diff}
\lambda_k \bm{\theta}_{k} + \*p_k =
\frac{\lambda_k \Tilde{\bm{\theta}}_k + \*u_k}{\sqrt{\*n_k}+\varepsilon}
=  \frac{1+\eta \lambda_k}{{\eta}} \qty(\bm{\theta}_{k} - \bm{\theta}_{k+1}).
\end{equation}
Then for $\eta \leq \frac{\varepsilon}{3L}$, we have:
\[
\begin{aligned}
& {F}_{k+1}(\bm{\theta}_{k+1}) \leq f(\bm{\theta}_{k} ) + \innerprod{\nabla f(\bm{\theta}_{k}), \bm{\theta}_{k+1}-\bm{\theta}_{k}} + \frac{L}{2}\norm{\bm{\theta}_{k+1}-\bm{\theta}_{k}}^2 + \frac{\lambda_{k+1}}{2}\norm{\bm{\theta}_{k+1}}_{\sqrt{\*n_{k+1}}}^2 \\
\overset{(a)}{\leq} & f(\bm{\theta}_{k} ) + \innerprod{\nabla f(\bm{\theta}_{k}), \bm{\theta}_{k+1}-\bm{\theta}_{k}} + \frac{L}{2}\norm{\bm{\theta}_{k+1}-\bm{\theta}_{k}}^2 +  \frac{\lambda_{k}}{2}\norm{\bm{\theta}_{k+1}}_{\sqrt{\*n_{k}}}^2\\
\overset{(b)}{\leq} & {F}_k(\bm{\theta}_{k}) + \innerprod{\bm{\theta}_{k+1}-\bm{\theta}_{k},\lambda_k\bm{\theta}_{k} + \frac{\*g_k^{full}}{\sqrt{\*n_k}+\varepsilon}}_{\sqrt{\*n_{k}}} + \frac{L/{\varepsilon}+\lambda_k}{2} \norm{\bm{\theta}_{k+1}-\bm{\theta}_{k}}_{\sqrt{\*n_{k}}}^2\\
 = & {F}_k(\bm{\theta}_{k}) + \frac{L/{\varepsilon}+\lambda_k}{2} \norm{\bm{\theta}_{k+1}-\bm{\theta}_{k}}_{\sqrt{\*n_{k}}}^2 +
 \innerprod{\bm{\theta}_{k+1}-\bm{\theta}_{k}, \lambda_k \bm{\theta}_{k} + \*p_k + \frac{\*g_k^{full}-\*u_k}{\sqrt{\*n_k}+\varepsilon}}_{\sqrt{\*n_{k}}}
\\
\overset{(c)}{=} & {F}_k(\bm{\theta}_{k}) + \qty(\frac{L/{\varepsilon}+\lambda_k}{2}-\frac{1+\eta \lambda_k}{{\eta}}) \norm{\bm{\theta}_{k+1}-\bm{\theta}_{k}}_{\sqrt{\*n_{k}}}^2 +
 \innerprod{\bm{\theta}_{k+1}-\bm{\theta}_{k},  \frac{\*g_k^{full}-\*u_k}{\sqrt{\*n_k}+\varepsilon}}_{\sqrt{\*n_{k}}}
\\
  \overset{(d)}{\leq} & {F}_k(\bm{\theta}_{k}) + \qty(\frac{L/{\varepsilon}}{2} - \frac{1}{\eta}) \norm{\bm{\theta}_{k+1}-\bm{\theta}_{k}}_{\sqrt{\*n_{k}}}^2 + \frac{1}{2\eta}\norm{\bm{\theta}_{k+1}-\bm{\theta}_{k}}_{\sqrt{\*n_{k}}}^2  + \frac{\eta}{2{\varepsilon}}  \norm{ \*g_k^{full} - \*u_k}^2\\ 
\leq & {F}_k(\bm{\theta}_{k}) - \frac{1}{3\eta}\norm{\bm{\theta}_{k+1}-\bm{\theta}_{k}}_{\sqrt{\*n_{k}}}^2  + \frac{\eta}{2{\varepsilon}}  \norm{ \*g_k^{full} - \*u_k}^2 \\
\leq & {F}_k(\bm{\theta}_{k}) - \frac{\eta }{4c_\infty}\norm{\*u_k + \lambda_k \Tilde{\bm{\theta}}_k}^2 +  \frac{\eta}{2{\varepsilon}}\norm{\*g_k^{full} - \*u_k}^2 ,
\end{aligned}
\]
where (a) comes from the fact $\lambda_{k+1}(1-\mu)^{-1} = \lambda_{k}$ and Proposition \ref{prop:eta_diff}:
$
\qty(\frac{\sqrt{\*n_{k}} + \varepsilon }{\sqrt{\*n_{k+1}} + \varepsilon})_i \geq 1-\mu
$, which implies:
\[
\lambda_{k+1}\norm{\bm{\theta}_{k+1}}_{\sqrt{\*n_{k+1}}}^2 \leq \frac{\lambda_{k+1}}{1-\mu} \norm{\bm{\theta}_{k+1}}_{\sqrt{\*n_{k}}}^2 = \lambda_{k} \norm{\bm{\theta}_{k+1}}_{\sqrt{\*n_{k}}}^2,
\]
and (b) is from:
\[
\norm{\bm{\theta}_{k+1}}_{\sqrt{\*n_{k}}}^2 = \qty(\norm{\bm{\theta}_{k}}_{\sqrt{\*n_{k}}}^2 + 2 \innerprod{\bm{\theta}_{k+1}-\bm{\theta}_{k},\bm{\theta}_{k}}_{\sqrt{\*n_{k}}} + \norm{\bm{\theta}_{k+1}-\bm{\theta}_{k}}_{\sqrt{\*n_{k}}}^2),
\]
(c) is due to Eqn.~\eqref{eq:eta-diff}, and
for (d), we utilize:
\[
\innerprod{\bm{\theta}_{k+1}-\bm{\theta}_{k},  \frac{\*g_k^{full}-\*u_k}{\sqrt{\*n_k}+\varepsilon}}_{\sqrt{\*n_{k}}} \leq \frac{1}{2\eta}\norm{\bm{\theta}_{k+1}-\bm{\theta}_{k}}_{\sqrt{\*n_{k}}}^2 + \frac{\eta}{2{\varepsilon}}  \norm{ \*g_k^{full} - \*u_k}^2,
\]
the last inequality comes from the fact in Eqn.~\eqref{eq:eta-diff} and $\eta \leq \frac{1}{10\lambda}$, such that:
\[
\frac{1}{3\eta}\norm{\qty(\bm{\theta}_{k+1}-\bm{\theta}_{k})}_{\sqrt{\*n_{k}}}^2 = \frac{\eta}{3 \qty(1+\eta\lambda_k)^2}\innerprod{\*u_k + \lambda_k \Tilde{\bm{\theta}}_k,  \frac{\*u_k + \lambda_k \Tilde{\bm{\theta}}_k}{\sqrt{\*n_k}+\varepsilon}} \geq \frac{\eta}{4c_\infty}\norm{\*u_k + \lambda_k \Tilde{\bm{\theta}}_k}^2.
\]
\end{proof}
}
\begin{theorem}
Suppose Assumptions \ref{asm:Lsmooth} and \ref{asm:boundVar} hold. 
Let $c_l \coloneqq \frac{1}{c_\infty}$ and $c_u \coloneqq \frac{1}{{\varepsilon}}$.
With ${\sqrt{2\beta_3} c_\infty}/{\varepsilon}\ll 1$,
\[
 \eta^2 \leq \frac{c_l\beta_1^2}{8c_u^3 L^2},  \quad \max\qty{\beta_1,\beta_2} \leq \frac{ c_l \epsilon^2}{96 c_u \sigma^2},\quad T \geq \max\qty{\frac{24 \Delta_0}{\eta c_l \epsilon^2},\frac{24 c_u \sigma^2}{\beta_1 c_l \epsilon^2}},
\]
where $\Delta_0 \coloneqq F(\bm{\theta}_{0}) - f^*$ and $f^* \coloneqq \min_{\bm{\theta}}\E_{\bm{\zeta}}[\nabla f(\bm{\theta},\bm{\zeta})]$, then we let $\*u_k \coloneqq \*m_k + \qty(1-\beta_1)\*v_k$ and have:
\[
\frac{1}{T+1} \sum_{k=0}^T \^E\qty(\norm{\*u_k + \lambda_k \Tilde{\bm{\theta}}_k}^2) \leq \epsilon^2, 
\]
and 
\[
\frac{1}{T+1}\sum_{k=0}^T \^E \qty(\norm{\*m_k - \*g^{full}_k}^2) \leq \frac{\epsilon^2}{4}, \quad \frac{1}{T+1}\sum_{k=0}^T \^E \qty(\norm{\*v_k}^2) \leq \frac{\epsilon^2}{4},
\]
where ${\*g}^{full}_{k} \coloneqq \E_{\bm{\zeta}}[\nabla f(\bm{\theta}_k,\bm{\zeta})]$. Hence, we have:
\[
\frac{1}{T+1} \sum_{k=0}^T \^E\qty(\norm{\nabla_{\bm{\theta}_k} \qty(\frac{\lambda_k}{2}\norm{\bm{\theta}}_{\sqrt{\*n_k}}^2 + \E_{\bm{\zeta}}[\nabla f(\bm{\theta},\bm{\zeta})])}^2) \leq 4\epsilon^2.
\]
\end{theorem}
\begin{proof}
We have:
\[
\norm{\*u_k - \*g^{full}_k}^2 \leq 2 \norm{\*m_k - \*g^{full}_k}^2 + 2 \qty(1-\beta_1)^2 \norm{\*v_k}^2.
\]
By Lemma \ref{lem:prox}, Lemma~\ref{lem:mk}, and Lemma \ref{lem:vk},
we already have:
\begin{align}
&{F}_{k+1}(\bm{\theta}_{k+1}) \leq  {F}_{k}(\bm{\theta}_{k}) - \frac{\eta c_l}{4}\norm{\*u_k + \lambda_k \Tilde{\bm{\theta}}_k}^2 +{\eta c_u} \norm{\*g^{full}_k - \*m_k}^2 + {\eta c_u}\qty(1-\beta_1)^2 \norm{\*v_k}^2 , \label{eq:thm1:obj}\\
& \^E\qty(\norm{\*m_{k+1} - \*g^{full}_{k+1}}^2) \leq \qty(1-\beta_1)\^E\qty(\norm{\*m_{k} - \*g^{full}_{k}}^2) + \frac{\qty(1-\beta_1)^2L^2}{\beta_1} \^E\qty(\norm{\bm{\theta}_{k+1} - \bm{\theta}_{k}}^2) + {\beta_1^2 \sigma^2}\label{eq:thm1:mk}\\
& \^E\qty(\norm{\*v_{k+1}}^2) \leq \qty(1-\beta_2)\^E\qty(\norm{\*v_{k}}^2) + 2\beta_2 \^E\qty(\norm{{\*g}^{full}_{k+1} - {\*g}^{full}_{k}}^2)
+  {3\beta_2^2 \sigma^2} \label{eq:thm1:vk}
\end{align}
Then by adding Eq.\eqref{eq:thm1:obj} with $\frac{\eta c_u}{\beta_1} \times$ Eq.\eqref{eq:thm1:mk} and $\frac{\eta c_u \qty(1-\beta_1)^2}{\beta_2} \times$ Eq.\eqref{eq:thm1:vk}, we can get:
\[
\begin{aligned}
&\^E\qty(\Phi_{k+1}) \leq 
\^E\qty( \Phi_{k}  - \frac{\eta c_l}{4} \norm{\*u_k + \lambda_k \Tilde{\bm{\theta}}_k}^2 + \frac{\eta c_u}{\beta_1}\qty(\frac{\qty(1-\beta_1)^2L^2}{\beta_1} \norm{\bm{\theta}_{k+1} - \bm{\theta}_{k}}^2 + {\beta_1^2 \sigma^2}) )\\
 & + \frac{\eta c_u \qty(1-\beta_1)^2}{\beta_2} \qty(2\beta_2 L^2 \norm{\bm{\theta}_{k+1} - \bm{\theta}_{k}}^2 + {3\beta_2^2 \sigma^2})\\
\leq & \^E\qty( \Phi_{k}  - \frac{\eta c_l}{4} \norm{\*u_k + \lambda_k \Tilde{\bm{\theta}}_k}^2 + \eta c_u L^2 \qty(\frac{\qty(1-\beta_1)^2}{\beta_1^2}+2\qty(1-\beta_1)^2) \norm{\bm{\theta}_{k+1} - \bm{\theta}_{k}}^2) + \qty(\beta_1 + 3\beta_2){\eta c_u \sigma^2}\\
\overset{(a)}{\leq} &  \^E\qty( \Phi_{k}  - \frac{\eta c_l}{4} \norm{\*u_k + \lambda_k \Tilde{\bm{\theta}}_k}^2 +  \frac{\eta c_u L^2}{\beta_1^2} \norm{\bm{\theta}_{k+1} - \bm{\theta}_{k}}^2) + { 4 \beta_m\eta c_u \sigma^2}\\
\overset{(b)}{\leq} &  \^E\qty( \Phi_{k} +\qty( \frac{(\eta c_u)^3L^2}{\beta_1^2}- \frac{\eta c_l}{4})  \norm{\*u_k + \lambda_k \Tilde{\bm{\theta}}_k}^2 ) + { 4 \beta_m\eta c_u \sigma^2}\\
\leq &  \^E\qty( \Phi_{k} - \frac{\eta c_l}{8}  \norm{\*u_k + \lambda_k \Tilde{\bm{\theta}}_k}^2 ) + { 4 \beta_m\eta c_u \sigma^2},
\end{aligned}
\]
where we let:
\[
\begin{aligned}
&\Phi_k \coloneqq {F}_k(\bm{\theta}_{k}) - f^* + \frac{\eta c_u}{\beta_1} \norm{\*m_{k} - \*g^{full}_{k}}^2 + \frac{\eta c_u\qty(1-\beta_1)^2}{\beta_2}\norm{\*v_k}^2,\\
& \beta_m = \max\qty{
\beta_1, \beta_2} \leq \frac{2}{3}, \quad 
\eta^2 \leq \frac{c_l\beta_1^2}{8c_u^3 L^2},
\end{aligned}
\]
and for (a), when $\beta_1 \leq \frac{2}{3}$, we have:
\[
\frac{\qty(1-\beta_1)^2}{\beta_1^2}+2\qty(1-\beta_1)^2 < \frac{1}{\beta_1^2},
\]
and (b) is due to Eq.\eqref{eq:eta-diff} from Lemma \ref{lem:prox}.
And hence, we have:
\[
\sum_{k=0}^T\^E\qty(\Phi_{k+1}) \leq \sum_{k=0}^T\^E\qty(\Phi_k)  -\frac{\eta c_l}{8} \sum_{k=0}^T\norm{\*u_k + \lambda_k \Tilde{\bm{\theta}}_k}^2 + \qty(T+1){4\eta  c_u \beta_m \sigma^2}.
\]
Hence, we can get:
\[
\frac{1}{T+1} \sum_{k=0}^T \^E\qty(\norm{\*u_k + \lambda_k \Tilde{\bm{\theta}}_k}^2) \leq \frac{8 \Phi_0}{\eta c_l T} + \frac{32 c_u \beta \sigma^2}{c_l} = \frac{8 \Delta_0}{\eta c_l T} + \frac{8 c_u \sigma^2}{\beta_1 c_l T} + \frac{32 c_u \beta_m \sigma^2}{c_l} \leq \epsilon^2,
\]
where
\[
\Delta_0 \coloneqq F(\bm{\theta}_{0}) - f^*, \quad \beta_m \leq \frac{c_l \epsilon^2}{96 c_u \sigma^2},\quad T \geq \max\qty{\frac{24 \Delta_0}{\eta c_l \epsilon^2},\frac{24 c_u \sigma^2}{\beta_1 c_l \epsilon^2}}.
\]
We finish the first part of the theorem. From Eq.\eqref{eq:thm1:mk}, we can conclude that:
\[
\frac{1}{T+1}\sum_{k=0}^T \^E \qty(\norm{\*m_k - \*g^{full}_k}^2) \leq \frac{\sigma^2}{ \beta T} + \frac{L^2\eta^2 c_u^2  \epsilon^2}{\beta_1^2} + {\beta_1 \sigma^2} < \frac{\epsilon^2}{4}.
\]
From Eq.\eqref{eq:thm1:vk}, we can conclude that:
\[
\frac{1}{T+1}\sum_{k=0}^T \^E \qty(\norm{\*v_k}^2) \leq  2L^2\eta^2 c_u^2  \epsilon^2 + {3\beta_2 \sigma^2} < \frac{\epsilon^2}{4}.
\]
Finally we have:
\[
\begin{aligned}
& \frac{1}{T+1} \sum_{k=0}^T \^E\qty(\norm{\nabla_{\bm{\theta}_k} \qty(\frac{\lambda_k}{2}\norm{\bm{\theta}}_{\sqrt{\*n_k}}^2 + \E_{\bm{\zeta}}[ f(\bm{\theta}_k,\bm{\zeta})])}^2)\\
\leq & \frac{1}{T+1} \qty(\sum_{k=0}^T \^E\qty(2\norm{\*u_k + \lambda_k \Tilde{\bm{\theta}_k}}^2 + 4\norm{\*m_k - \*g^{full}_k}^2 + 4\norm{\*v_k}^2)) \leq 4\epsilon^2.
\end{aligned}
\]
Now, we have finished the proof.
\end{proof}

\subsection{Faster Convergence with  Lipschitz Hessian}\label{sec:second-order}
For convenience, we let $\lambda=0$, $\beta_1 = \beta_2 = \beta$ and $\beta_3 = \beta^2$ in the following proof. To consider the weight decay term in the proof, we refer to the previous section for more details.
For ease of notation, we denote $\*x$ instead of $\bm{\theta}$ the variable needed to be optimized in the proof, and abbreviate $E_{\bm{\zeta}}[ f(\bm{\theta}_k,\bm{\zeta})]$ as $f(\bm{\theta}_k)$.
\subsubsection{Reformulation}
\begin{algorithm}
\SetAlgoLined
\KwIn{initial point $\bm{\theta}_0$, stepsize $\eta$, average coefficients $\beta$, and $\varepsilon$.}
\Begin{
 \While{$k<K$}{
    get stochastic gradient estimator ${\*g}_{k}$ at $\*x_k$\;
   $\hat{\*m}_k = \qty(1-\beta)\hat{\*m}_{k-1} +  \beta \qty({\*g}_k + \qty(1-\beta) \qty({\*g}_k - {\*g}_{k-1})) $\;
   $\*n_k = \qty(1-\beta^2)\*n_{k-1} + \beta^2 \qty({\*g}_{k-1}+ \qty(1-\beta)\qty({\*g}_{k-1} - {\*g}_{k-2}))^2 $\;
   $\bm{\eta}_k =  {\eta}/\qty(\sqrt{\*n_k} + \varepsilon)$\;
   $\*y_{k+1} = \*x_k - \bm{\eta}_k \beta  {\*g}_{k}$\;
   $\*x_{k+1} = \*y_{k+1} + (1-\beta)\qty[ \qty(\*y_{k+1} - \*y_{k}) + \qty(\bm{\eta}_{k-1} - \bm{\eta}_{k})  \qty(\hat{\*m}_{k-1} - \beta {\*g}_{k-1})]$\;
 \If{$(k+1)\sum_{t=0}^{k} \norm{\qty(\sqrt{\*n_{t}}+ \varepsilon )^{1/2}  \circ \qty(\*y_{t+1} - \*y_t)}^2 \geq R^2$}{
  get stochastic gradient estimator ${\*g}_{0}$ at $\*x_{k+1}$\;
  $\hat{\*m}_0 =  {\*g}_0, \quad \*n_0 = {\*g}_0^2, \quad \*x_0= \*y_0 = \*x_{k+1}, \quad \*x_1=\*y_1  = \*x_{0} - \eta \frac{\hat{\*m}_0}{\sqrt{\*n_0}+ \varepsilon} , \quad k = 1$\;
    }
 }
 $K_{0}=\argmin_{\left\lfloor\frac{K}{2}\right\rfloor \leq k \leq K-1}\norm{\qty({\sqrt{\*n_{k}}+ \varepsilon} )^{1/2}  \circ \qty(\*y_{k+1} - \*y_k)}$\;}
 \KwOut{$\Bar{\*x} \coloneqq\frac{1}{K_0} \sum_{k=1}^{K_0}\*x_k$}
 \caption{ Nesterov Adaptive Momentum Estimation Reformulation}\label{alg:Name_ref}
\end{algorithm}

We first prove the equivalent form between Algorithm \ref{alg:Name} and Algorithm \ref{alg:Name_ref}. The main iteration in Algorithm \ref{alg:Name} is:
\[
    \left\{
\begin{aligned}
    & \*m_k = \qty(1-\beta)\*m_{k-1} +  \beta {\*g}_k , \\
    & \*v_k = \qty(1-\beta)\*v_{k-1} +  \beta \qty( \qty({\*g}_k - {\*g}_{k-1})),\\
    & \*x_{k+1} = \*x_{k} - \bm{\eta}_k \circ \qty(\*m_k+\qty(1-\beta)\*v_k).
\end{aligned}
\right.
\]
Let $\hat{\*m}_k \coloneqq  \*m_k+\qty(1-\beta)\*v_k$, we can simplify the variable:
\[
    \left\{
\begin{aligned}
    & \hat{\*m}_k = \qty(1-\beta)\hat{\*m}_{k-1}+  \beta \qty({\*g}_k + (1-\beta) \qty({\*g}_k - {\*g}_{k-1})) , \\
    & \*x_{k+1} = \*x_{k} - \bm{\eta}_k \circ \hat{\*m}_k.
\end{aligned}
\right.
\]
We let $\*y_{k+1} \coloneqq \*x_{k+1} + \bm{\eta}_k \qty(\hat{\*m}_k - \beta {\*g}_k)$, then we can get:
\[
\*y_{k+1} = \*x_{k+1} + \bm{\eta}_k \hat{\*m}_k - \beta \bm{\eta}_k {\*g}_k = \*x_{k+1} + \*x_k - \*x_{k+1} - \beta \bm{\eta}_k {\*g}_k =  \*x_k - \beta \bm{\eta}_k {\*g}_k.
\]
On one hand, we have:
$
\*x_{k+1} = \*x_{k} - \bm{\eta}_k  \hat{\*m}_k = \*y_{k+1} - \bm{\eta}_k \qty(\hat{\*m}_k- \beta {\*g}_k).
$
On the other hand:
\[
\begin{aligned}
 \bm{\eta}_k\qty(\hat{\*m}_k - \beta {\*g}_k) = & \qty(1-\beta) \bm{\eta}_k \qty(\hat{\*m}_{k-1} + \beta\qty({\*g}_k - {\*g}_{k-1} )) 
=    \qty(1-\beta) \bm{\eta}_k \qty(\hat{\*m}_{k-1} + \beta\qty({\*g}_k - {\*g}_{k-1} )) \\
= & \qty(1-\beta) \bm{\eta}_k \qty(\frac{\*x_{k-1} - \*x_k}{\bm{\eta}_{k-1}}  + \beta\qty({\*g}_k - {\*g}_{k-1} )) \\
= & \qty(1-\beta) \frac{\bm{\eta}_k}{\bm{\eta}_{k-1}} \qty(\*x_{k-1} - \*x_k  + \beta\bm{\eta}_{k-1}\qty({\*g}_k - {\*g}_{k-1} )) \\ 
= & \qty(1-\beta) \frac{\bm{\eta}_k}{\bm{\eta}_{k-1}} \qty(\*y_{k} - \*x_k  + \beta\bm{\eta}_{k-1}{\*g}_k) \\
= & \qty(1-\beta) \qty[ \frac{\bm{\eta}_k}{\bm{\eta}_{k-1}} \qty(\*y_{k} - \*y_{k+1} - \beta\qty( \bm{\eta}_{k} -  \bm{\eta}_{k-1}){\*g}_k)] \\ 
= & \qty(1-\beta)\qty[ \qty(\*y_{k} - \*y_{k+1}) +  \frac{\bm{\eta}_k - \bm{\eta}_{k-1}}{\bm{\eta}_{k-1}} \qty(\*y_{k} - \*y_{k+1} - \beta\bm{\eta}_k{\*g}_k)]\\
= & \qty(1-\beta)\qty[ \qty(\*y_{k} - \*y_{k+1}) +  \frac{\bm{\eta}_k - \bm{\eta}_{k-1}}{\bm{\eta}_{k-1}} \qty(\*y_{k} - \*x_k)]\\
= & \qty(1-\beta)\qty[ \qty(\*y_{k} - \*y_{k+1}) +  \qty(\bm{\eta}_k - \bm{\eta}_{k-1})\qty(\*m_{k-1} - \beta {\*g}_{k-1})].
\end{aligned}
\]
Hence, we can conclude that:
\[
\*x_{k+1} = \*y_{k+1} + \qty(1-\beta)\qty[ \qty(\*y_{k+1} - \*y_{k}) +  \qty(\bm{\eta}_{k-1} - \bm{\eta}_{k})\qty(\hat{\*m}_{k-1} - \beta {\*g}_{k-1})].
\]
The main iteration in Algorithm \ref{alg:Name} becomes:
\begin{equation}\label{eq:xy-reformulate}
     \left\{
\begin{aligned}
    & \*y_{k+1} =  \*x_k - \beta \bm{\eta}_k {\*g}_k, \\
    & \*x_{k+1} = \*y_{k+1} + \qty(1-\beta) \qty[ \qty(\*y_{k+1} - \*y_{k}) +  \frac{\bm{\eta}_{k-1} - \bm{\eta}_{k}}{\bm{\eta}_{k-1}} \qty(\*y_{k} - \*x_k)].
\end{aligned}
\right.   
\end{equation}
\subsubsection{Auxiliary Bounds}
We first show some interesting property. Define $\$K$ to be the iteration number when the 'if condition' triggers, that is,
\[
\$K \coloneqq \min_{k}\qty{k \middle| k\sum_{t=0}^{k-1} \norm{(\sqrt{\*n_t} + \varepsilon)^{1/2}  \circ \qty(\*y_{t+1} - \*y_t)}^2 > R^2}.
\]
\begin{prop}\label{prop:x-y}
Given $k\leq \$K$ and $\beta \leq {{\varepsilon}}/\qty(\sqrt{2}c_\infty+\varepsilon)$, we have:
\[
\norm{\qty(\sqrt{\*n_k}+\varepsilon)^{1/2} \circ \qty(\*x_{k} - \*y_k)} \leq R.
\]
\end{prop}
\begin{proof}
First of all, we let $\hat{\*n}_k \coloneqq \qty(\sqrt{\*n_k}+\varepsilon)^{1/2}$. 
Due to Proposition \ref{prop:eta_diff}, we have:
\[
\qty(\frac{\sqrt{\*n_{k-1}} + \varepsilon}{\sqrt{\*n_{k}} + \varepsilon})_i   \in \qty[1-\frac{\sqrt{2}\beta c_\infty}{\varepsilon}, 1+\frac{\sqrt{2}\beta  c_\infty}{\varepsilon}],
\]
then, we get: 
\[
\hat{\*n}_k \leq \qty(1-\frac{\sqrt{2} \beta c_\infty}{\varepsilon})^{-1/2} \hat{\*n}_{k-1} \leq \qty(1-\beta)^{-1/4} \hat{\*n}_{k-1},
\]
where we use the fact $\beta \leq {{\varepsilon}}/\qty(2\sqrt{2}c_\infty+\varepsilon)$.
For any $1 \leq k\leq \$K$, we have:
\[
\begin{aligned}
 & \norm{\hat{\*n}_k \circ  \qty(\*y_{k} - \*y_{k-1})}^2 \leq  \qty(1-\beta)^{-1/2}\norm{\hat{\*n}_{k-1}  \circ \qty(\*y_{k} - \*y_{k-1})}^2 
 \leq 
\qty(1-\beta)^{-1}\sum_{t=1}^{k-1} \norm{ \hat{\*n}_{t}  \circ \qty(\*y_{t+1} - \*y_t)}^2 
\leq  \frac{R^2}{k(1-\beta)},
\end{aligned}
\]
hence, we can conclude that:
\begin{equation} \label{eq:yybound}
  \norm{\hat{\*n}_k \circ \qty(\*y_{k} - \*y_{k-1})}^2 \leq \frac{R^2}{k(1-\beta)}.  
\end{equation}
On the other hand, by Eq.\eqref{eq:xy-reformulate}, we have:
\[
\*x_{k+1} - \*y_{k+1} = \qty(1-\beta) \qty[\qty(\*y_{k+1} - \*y_{k}) +  \frac{\bm{\eta}_{k} - \bm{\eta}_{k-1}}{\bm{\eta}_{k-1}}\qty(\*x_k - \*y_k)],
\]
and hence,
\[
\begin{aligned}
 \norm{\hat{\*n}_k \circ \qty(\*x_{k} - \*y_k)}  & \leq  \qty(1-\beta) \qty[\norm{\hat{\*n}_k \circ \qty(\*y_{k} - \*y_{k-1})} +  \norm{\frac{\bm{\eta}_{k-1} - \bm{\eta}_{k-2}}{\bm{\eta}_{k-2}}}_\infty \norm{\hat{\*n}_k \circ(\*x_{k-1} - \*y_{k-1})}]\\
 & \overset{(a)}{\leq}  \sqrt{1-\beta}\frac{R}{ \sqrt{k}} + \qty(1-\beta) \frac{\sqrt{2}\beta^2 c_\infty}{\varepsilon} \qty(1-\frac{\sqrt{2}\beta^2 c_\infty}{\varepsilon})^{-1/2} \norm{\hat{\*n}_{k-1} \circ(\*x_{k-1} - \*y_{k-1})} \\
  & \leq \sqrt{1-\beta}\frac{R}{ \sqrt{k}} + \beta \qty(1-\beta)^{3/4} \norm{\hat{\*n}_{k-1} \circ(\*x_{k-1} - \*y_{k-1})} \\
 & \leq  \sqrt{1-\beta}R \qty(\frac{1}{\sqrt{k}} + \frac{\beta \qty(1-\beta)^{3/4}}{\sqrt{k-1}} + \cdots + \qty(\beta \qty(1-\beta)^{3/4})^{k-1}) \\
 & \overset{(b)}{\leq} \sqrt{1-\beta}R \qty( \sum_{t=1}^{k-1} \frac{1}{t^2})^{1/4} \qty(\sum_{t=0}^k \qty(\beta \qty(1-\beta)^{3/4})^{4t/3})^{3/4} 
  \overset{(c)}{<}  R,
\end{aligned}
\]
where (a) comes from Eq.\eqref{eq:yybound} and the proposition \ref{prop:eta_diff}, (b) is the application of Hölder's inequality and (c) comes from the facts when $\beta \leq 1/2$:
\[
\sum_{t=1}^\infty \frac{1}{t^2} = \frac{\pi^2}{6}, \quad
\sqrt{1-\beta}\qty(\sum_{t=0}^k \qty(\beta \qty(1-\beta)^{3/4})^{4t/3})^{3/4} \leq   \qty(\frac{(1-\beta)^{2/3}}{1 - \beta^{4/3}(1-\beta)})^{3/4}.
\]
\end{proof}

\subsubsection{Decrease of One Restart Cycle}
\begin{lemma}\label{lem:one-cycle}
Suppose that Assumptions \ref{asm:Lsmooth}-\ref{asm:boundVar} hold. Let  $R = \order{\epsilon^{0.5}}$, $\beta  = \order{\epsilon^2} $, $\eta =  \order{\epsilon^{1.5}}$, $\$K \leq K = \order{\epsilon^{-2}}$. Then we have:
\begin{equation}\label{eq:fianl-diff}
 \E\qty(f(\*y_{\$K}) - f(\*x_0)) = - \order{\epsilon^{1.5}}.   
\end{equation}
\end{lemma}
\begin{proof}
Recall Eq.\eqref{eq:xy-reformulate} and denote ${\*g}^{full}_{k}\coloneqq \nabla f(\bm{\theta}_k)$ for convenience:
\begin{equation}
\left\{ 
\begin{aligned}
 & {\*y}_{k+1} = {\*x}_k -  \beta\bm{\eta}_k \circ  \qty( {\*g}^{full}_k +  {\bm{\xi}}_k) \\
 &  {\*x}_{k+1} -  {\*y}_{k+1} = \qty(1-\beta) \qty[ \qty( {\*y}_{k+1} -  {\*y}_{k}) +  \qty( \frac{\bm{\eta}_{k} - \bm{\eta}_{k-1}}{\bm{\eta}_{k-1}}\circ\qty(\*x_k - \*y_k))],
\end{aligned}
\right.
\end{equation}
In this proof, we let $\hat{\*n}_k \coloneqq \qty(\sqrt{\*n_k}+\varepsilon)^{1/2}$, and hence $\bm{\eta}_k = \eta/\hat{\*n}^2_k$.
On one hand, we have:
\begin{equation}\label{eq:fx-fy}
\begin{aligned}
 & \E(f(\*x_k) - f(\*y_k)) \leq \E\qty(   \innerprod{\nabla f(\*y_k),  \*{x}_{k} -  \*{y}_k} + \frac{L}{2} \norm{ \*{x}_{k} -  \*{y}_k}^2)\\
  = & \E\qty(   \innerprod{\*g_k,  \*{x}_{k} -  \*{y}_k} + \innerprod{\nabla f(\*y_k) - \nabla f(\*x_k),  \*{x}_{k} -  \*{y}_k} + \frac{L}{2} \norm{ \*{x}_{k} -  \*{y}_k}^2)\\
  \leq & \E\qty(   \innerprod{\*g_k,  \*{x}_{k} -  \*{y}_k} + \frac{1}{2L}\norm{\nabla f(\*y_k) - \nabla f(\*x_k)}^2+ \frac{L}{2} \norm{ \*{x}_{k} -  \*{y}_k}^2 + \frac{L}{2} \norm{ \*{x}_{k} -  \*{y}_k}^2)\\
  \leq & \E\qty(   \innerprod{\*g_k,  \*{x}_{k} -  \*{y}_k} +\frac{3L}{2} \norm{ \*{x}_{k} -  \*{y}_k}^2)\\
  = & \E\qty(-\innerprod{\frac{  {\*y}_{k+1} - {\*x}_k} {\beta\bm{\eta}_k} + \bm{\xi}_k, \*{x}_{k} -  \*{y}_k}  +\frac{3L}{2} \norm{ \*{x}_{k} -  \*{y}_k}^2)\\
  = & \E\qty(\frac{1}{\eta \beta}\innerprod{\hat{\*n}^2_k \circ \qty(\*y_{k+1} - \*x_k), \*y_k-\*x_k}  +\frac{3L}{2} \norm{ \*{x}_{k} -  \*{y}_k}^2)\\
  \overset{(a)}{\leq} & \E\qty(\frac{1}{2\eta \beta}\qty(\norm{\hat{\*n}_k \circ\qty(\*y_{k+1}-\*x_k)}^2 + \norm{\hat{\*n}_k \circ\qty(\*y_{k}-\*x_k)}^2 - \norm{\hat{\*n}_k \circ\qty(\*y_{k+1}-\*y_k)}^2)+\frac{3L}{2} \norm{ \*{x}_{k} -  \*{y}_k}^2)\\
  \overset{(b)}{\leq} & \E\qty(\frac{1}{2\eta \beta}\qty(\norm{\hat{\*n}_k \circ\qty(\*y_{k+1}-\*x_k)}^2  - \norm{\hat{\*n}_k \circ\qty(\*y_{k+1}-\*y_k)}^2)+\frac{1+\beta/2}{2\eta \beta} \norm{\hat{\*n}_k \circ\qty(\*y_{k}-\*x_k)}^2)
\end{aligned}
\end{equation}
where (a) comes from the following facts, and in (b), we use $3L \eta \leq \frac{{\varepsilon}}{2}$:
\[
  2\innerprod{\hat{\*n}^2_k \circ \qty(\*y_{k+1} - \*x_k), \*y_k-\*x_k} = \norm{\hat{\*n}_k \circ\qty(\*y_{k+1}-\*x_k)}^2 + \norm{\hat{\*n}_k \circ\qty(\*y_{k}-\*x_k)}^2 - \norm{\hat{\*n}_k \circ\qty(\*y_{k+1}-\*y_k)}^2.
\]
On the other hand, by the $L$-smoothness condition, for $1\leq k\leq \$K$, we have:
\begin{equation}\label{eq:fy-1-fx}
    \begin{aligned}
 \E \qty( f(\*y_{k+1}) - f(\*x_k) ) \leq &  \E \qty(\innerprod{\*g_{k}, \*y_{k+1}-\*x_k} + \frac{L}{2}\norm{\*y_{k+1}-\*x_k}^2) \\
 = &\E \qty(- \innerprod{\frac{  {\*y}_{k+1} - {\*x}_k} {\beta\bm{\eta}_k} + \bm{\xi}_k, \*{y}_{k+1} -  \*{x}_k}+ \frac{L}{2}\norm{\*y_{k+1}-\*x_k}^2) \\
 \overset{(a)}{\leq} &  \E\qty(-\frac{1}{\eta \beta}\norm{\H{\*n}_k\circ\qty(\*{y}_{k+1} - \*{x}_k  )}^2 +\frac{L}{2}\norm{\*y_{k+1}-\*x_k}^2 ) + \frac{\eta \beta \sigma^2}{{\varepsilon}}\\
  \leq &  \E\qty(-\frac{1}{\eta \beta}\norm{\H{\*n}_k\circ\qty(\*{y}_{k+1} - \*{x}_k  )}^2 +\frac{L}{2{\varepsilon}}\norm{\H{\*n}_k\circ\qty(\*{y}_{k+1} - \*{x}_k  )}^2 ) + \frac{\eta \beta \sigma^2}{{\varepsilon}}\\
   \leq &  \E\qty(-\frac{1}{2\eta \beta}\norm{\H{\*n}_k\circ\qty(\*{y}_{k+1} - \*{x}_k  )}^2 )  + \frac{\eta \beta \sigma^2}{{\varepsilon}},
\end{aligned}
\end{equation}
where (a) comes from the facts:
$
\E\qty(\innerprod{ \bm{\xi}_k,\*{y}_{k+1} - \*{x}_k}) = 
\E\qty(\innerprod{ \bm{\xi}_k,{\*x}_k -  \beta\bm{\eta}_k \circ  \qty( {\*g}_k +  {\bm{\xi}}_k)}) =\E\qty(\innerprod{ \bm{\xi}_k,\beta\bm{\eta}_k \circ   {\bm{\xi}}_k})
\leq  \frac{\eta \beta \sigma^2}{{\varepsilon}}.
$
and the last inequality is due to $L\eta \leq {\varepsilon}$. 
By combing Eq.\eqref{eq:fx-fy} and Eq.\eqref{eq:fy-1-fx}, we have:
\[
\begin{aligned}
 & \E\qty(f(\*y_{k+1}) - f(\*y_k)) \leq \E\qty(-\frac{1}{2\eta \beta} \norm{\hat{\*n}_k \circ\qty(\*y_{k+1}-\*y_k)}^2+\frac{1+\beta/2}{2\eta \beta} \norm{\hat{\*n}_k \circ\qty(\*y_{k}-\*x_k)}^2) + \frac{\eta \beta \sigma^2}{{\varepsilon}}\\
 \overset{(a)}{\leq} & \E\qty(-\frac{1}{2\eta \beta} \norm{\hat{\*n}_k \circ\qty(\*y_{k+1}-\*y_k)}^2+\frac{1-\beta/2-\beta^2/2}{2\eta \beta} \norm{\hat{\*n}_{k-1} \circ\qty(\*y_{k}-\*y_{k-1})}^2)+\frac{4\beta^2 R^2c_\infty^2}{\eta\varepsilon^2}  + \frac{\eta \beta \sigma^2}{{\varepsilon}},
\end{aligned}
\]
where (a) comes from the following fact, and note that by Proposition \ref{prop:x-y} we already have $\hat{\*n}_k \leq \qty(1-\beta)^{-1/4} \hat{\*n}_{k-1}$:
\begin{equation}\label{eq:y-y}
\begin{aligned}  
  &\norm{\hat{\*n}_k \circ \qty(\*x_{k} - \*y_k)}^2 
  \leq    \qty(1-\beta)^2 \qty[ (1+\alpha) \norm{\hat{\*n}_k \circ \qty(\*y_{k} - \*y_{k-1})}^2 +  (1+\frac{1}{\alpha})\hat{\beta}^2 \norm{\hat{\*n}_k \circ(\*x_{k-1} - \*y_{k-1})}^2]\\
  \leq &  \qty(1-\beta)^{3/2 } \qty[\qty(1+\alpha) \norm{\hat{\*n}_{k-1} \circ \qty(\*y_{k} - \*y_{k-1})}^2 +  (1+\frac{1}{\alpha})\hat{\beta}^2  \norm{\hat{\*n}_{k-1} \circ(\*x_{k-1} - \*y_{k-1})}^2] \\
  \leq &  \qty(1-\beta)\norm{\hat{\*n}_{k-1} \circ \qty(\*y_{k} - \*y_{k-1})}^2 + \frac{\hat{\beta}^2 (1-\beta)^{3/2}}{1- (1-\beta)^{1/2}} \norm{\hat{\*n}_{k-1} \circ(\*x_{k-1} - \*y_{k-1})}^2\\
  \leq &  (1-\beta) \norm{\hat{\*n}_{k-1} \circ \qty(\*y_{k} - \*y_{k-1})}^2 + \frac{2\hat{\beta}^2}{\beta} \norm{\hat{\*n}_{k-1} \circ(\*x_{k-1} - \*y_{k-1})}^2\\
   \leq &  \qty(1-\beta)\norm{\hat{\*n}_{k-1} \circ \qty(\*y_{k} - \*y_{k-1})}^2 + 4\beta^3 R^2c_\infty^2/\varepsilon^2, 
\end{aligned}  
\end{equation}
where we let $\hat{\beta} \coloneqq {\sqrt{2}\beta^2 c_\infty}/{\varepsilon}$, $\alpha = (1-\beta)^{-1/2}-1$, and the last inequality we use the results in Proposition \ref{prop:x-y}. Summing over $k=2,\cdots,\$K-1$, and note that $\*y_1 = \*x_1$, and hence we have $\E \qty( f(\*y_{2}) - f(\*x_1) ) = \E \qty( f(\*y_{2}) - f(\*y_1) ) \leq  {\eta \beta \sigma c_\infty}/{\sqrt{\varepsilon}}$ due to Eq.~\eqref{eq:fy-1-fx},
then we get:
\[
\begin{aligned}
 &\E\qty(f(\*y_{\$K}) - f(\*y_1))\leq \E\qty(-\frac{1}{4\eta}\sum_{t=1}^{\$K-1}\norm{\hat{\*n}_k \circ\qty(\*y_{t+1}-\*y_t)}^2)  +\frac{4\$K\beta^2 R^2c_\infty^2}{\eta\varepsilon^2}  + \frac{\$K\eta \beta \sigma^2}{{\varepsilon}}.
 \end{aligned}
 \]
 On the other hand, similar to the results given in Eq.\eqref{eq:fy-1-fx}, we have:
\[
\begin{aligned}
 \E\qty(f(\*y_{1}) - f(\*y_0)) =  \E\qty(f(\*x_{1}) - f(\*x_0))
 \leq   \E\qty(-\frac{1}{2\eta }\norm{\H{\*n}_k\circ\qty(\*{y}_{1} - \*{y}_0  )}^2 ) + \frac{\eta \sigma^2}{{\varepsilon}}.
\end{aligned}
\]
Therefore, using $\beta \$K  = \order{1}$ and the restart condition 
$\$K\sum_{t=0}^{\$K-1} \norm{(\sqrt{\*n_t} + \varepsilon)^{1/2}  \circ \qty(\*y_{t+1} - \*y_t)}^2 \geq R^2,$ we can get:
\[
\begin{aligned}
& \E\qty(f(\*y_{\$K}) - f(\*y_0))
\leq  \E\qty(-\frac{1}{4\eta}\sum_{t=0}^{\$K-1}\norm{\hat{\*n}_k \circ\qty(\*y_{k+1}-\*y_k)}^2)  +\frac{4\$K\beta^2 R^2c_\infty^2}{\eta\varepsilon^2}  + \frac{(\$K\beta+1) \eta \sigma^2}{{\varepsilon}}\\
 \leq & -\frac{ R^2}{ 4 \$K \eta }+\frac{4\$K\beta^2 R^2c_\infty^2}{\eta\varepsilon^2}  + \frac{(\$K\beta+1) \eta \sigma^2}{{\varepsilon}}
 = -\order{\frac{R^2}{\$K \eta}  - \frac{\beta R^2}{\eta}- \eta} = -\order{\epsilon^{1.5}}. 
\end{aligned}
\]
Now, we finish the proof of this claim.
\end{proof}
\subsubsection{Gradient in the last Restart Cycle}
Before showing the main results, we first provide several definitions.
Note that, for any $k< \$K$ we already have:
\[
\begin{aligned}
 (\varepsilon)^{1/2}\norm{\*y_{k} - \*y_0} \leq (\varepsilon)^{1/2}\sqrt{k \sum_{t=0}^{k-1}\norm{\*y_{t+1} - \*y_t}^2} \leq R.
\end{aligned}
\]
and we have:
\begin{equation}\label{eq:xk-x0}
 \E\qty(\norm{\*x_{k} - \*x_0}) \leq  
\E\qty(\norm{\*y_{k} - \*x_k} + \norm{\*y_k - \*x_0}) \leq \frac{2R}{\varepsilon^{1/2}},
\end{equation}
where we utilize the results from Proposition \ref{prop:x-y}.
For each epoch, denote $\*H \coloneqq \nabla^2 f(\*x_0)$.
We then define: 
\[
h(\*y) \coloneqq \innerprod{\*g^{full}_0,\*y - \*x_0} + \frac{1}{2}\qty(\*y - \*x_0)^\top\*H\qty(\*y - \*x_0).
\]
Recall the Eq. (\ref{eq:xy-reformulate}):
\begin{equation}\label{eq:name_h}
\left\{ 
\begin{aligned}
 & {\*y}_{k+1} = {\*x}_k -  \beta\bm{\eta}_k \circ  \qty( {\*g}^{full}_k +  {\bm{\xi}}_k) =  {\*x}_k - \beta \bm{\eta}_k  \circ \qty( \nabla h( {\*x}_k) + \bm{\delta}_k  +  {\bm{\xi}}_k)\\
 &  {\*x}_{k+1} -  {\*y}_{k+1} = \qty(1-\beta) \qty[ \qty( {\*y}_{k+1} -  {\*y}_{k}) +  \qty( \frac{\bm{\eta}_{k} - \bm{\eta}_{k-1}}{\bm{\eta}_{k-1}}\circ\qty(\*x_k - \*y_k))],
\end{aligned}
\right.
\end{equation}
where we let $\bm{\delta}_k \coloneqq {\*g}^{full}_k - \nabla h({\*x}_k)$, and we can get that:
\begin{equation}\label{eq:delta-bound}
  \begin{aligned}
 & \E\qty(\norm{\bm{\delta}_k})  = \E\qty(\norm{{\*g}^{full}_k - {\*g}^{full}_0 - \*H\qty({\*x}_k - {\*x}_0)})\\
 = & \E\qty(\norm{\qty(\int_0^1 \nabla^2 h\qty(\*x_0+t\qty(\*x_k-\*x_0))-\*H)\qty(\*x_k-\*x_0)dt})
 \leq  \frac{\rho}{2}\E\qty(\norm{\*x_k-\*x_0}^2) \leq \frac{2 \rho R^2}{\varepsilon}.
\end{aligned}  
\end{equation}
Iterations in Eq.\eqref{eq:name_h} can be viewed as applying the proposed optimizer to the quadratic approximation $h(\*x)$ with the gradient error $\delta_k$, which is in the order of $\order{\rho R^2/\varepsilon}$. 
\begin{lemma}\label{lem:last-cycle}
Suppose that Assumptions \ref{asm:Lsmooth}-\ref{asm:rhosmooth} hold. Let  $B = \order{\epsilon^{0.5}}$, $\beta  = \order{\epsilon^2} $, $\eta =  \order{\epsilon^{1.5}}$, $\$K \leq K = \order{\epsilon^{-2}}$. Then we have:
\[
\E\qty(\norm{\nabla f(\Bar{\*x})}) = \order{\epsilon}, \quad
\text{where } 
\Bar{\*x}\coloneqq \frac{1}{K_0-1}\sum_{k=1}^{K_0}  \*x_k.
\]
\end{lemma} 
\begin{proof}
Since $h(\cdot)$ is quadratic, then we have:
\[
\begin{aligned}
&\E\qty(\norm{\nabla h(\Bar{\*x})}) = \E\qty(\norm{\frac{1}{K_0-1}\sum_{k=1}^{K_0}  \nabla h(\*x_k)}) 
=  \frac{1}{K_0-1}\E\norm{\sum_{k=1}^{K_0}(\beta\bm{\eta}_k)^{-1}\circ \qty(\*y_{k+1}-\*x_k)+\bm{\xi}_k + \bm{\delta}_k}\\
\leq & \frac{1}{(K_0-1)\beta}\E\norm{\sum_{k=1}^{K_0}(\beta\bm{\eta}_k)^{-1}\circ \qty(\*y_{k+1}-\*x_k)} 
+ \frac{1}{(K_0-1)}\E\norm{\sum_{k=1}^{K_0}\bm{\xi}_k } + \frac{1}{(K_0-1)}\E\norm{\sum_{k=1}^{K_0} \bm{\delta}_k} \\
\overset{(a)}{\leq} & \frac{1}{(K_0-1)\beta}\E\norm{\sum_{k=1}^{K_0}(\bm{\eta}_k)^{-1}\circ \qty(\*y_{k+1}-\*x_k)} 
+ \frac{\sigma}{\sqrt{K_0-1}} + \frac{2 \rho R^2}{\varepsilon}\\
= & \frac{1}{(K_0-1)\beta}\E\norm{\sum_{k=1}^{K_0} \frac{\*y_{k+1}- \*y_k-\qty(1-\beta) \qty(\*y_{k} - \*y_{k-1})}{\bm{\eta}_k} \!-\!  (1\!-\!\beta)\frac{\bm{\eta}_{k-1} -  \bm{\eta}_{k-2}}{\bm{\eta}_{k-2}\bm{\eta}_k}\qty(\*x_{k-1} - \*y_{k-1})} \\
& + \frac{\sigma}{\sqrt{K_0-1}} + \frac{2 \rho R^2}{\varepsilon}\\
\overset{(b)}{\leq} & \frac{1}{(K_0-1)\beta}\E\norm{\sum_{k=1}^{K_0} \frac{\*y_{k+1}- \*y_k-\qty(1-\beta) \qty(\*y_{k} - \*y_{k-1})}{\bm{\eta}_k} } + \frac{2\beta c_\infty^{1.5} R}{\eta \varepsilon}
+ \frac{\sigma}{\sqrt{K_0-1}} + \frac{2 \rho R^2}{\varepsilon}\\
\overset{(c)}{\leq} & \frac{1}{(K_0-1)\beta}\E\norm{\sum_{k=1}^{K_0} \qty(\frac{\*y_{k+1}- \*y_k}{\bm{\eta}_k} - \frac{\qty(1-\beta) \qty(\*y_{k} - \*y_{k-1})}{\bm{\eta}_{k-1}})} + \frac{4\beta c_\infty^{1.5} R}{\eta \varepsilon}
+ \frac{\sigma}{\sqrt{K_0-1}} + \frac{2 \rho R^2}{\varepsilon}\\
\leq & 
\frac{1}{(K_0-1)\beta}\E\norm{\frac{\*y_{K_0}-\*y_{K_0-1}}{\bm{\eta}_{K_0}}} + \frac{1}{(K_0-1)}\E\norm{\sum_{k=1}^{K_0-1} \frac{\*y_{k+1}- \*y_k}{\bm{\eta}_k}} + \frac{4\beta c_\infty^{1.5} R}{\eta \varepsilon}
+ \frac{\sigma}{\sqrt{K_0-1}} + \frac{2 \rho R^2}{\varepsilon}\\
\overset{(d)}{\leq} &
 \frac{1}{(K_0-1)}\E\norm{\sum_{k=1}^{K_0} \frac{\*y_{k+1}- \*y_k}{\bm{\eta}_k}} + \frac{4R \sqrt{c_\infty}}{\beta \eta K^2} + \frac{4\beta c_\infty^{1.5} R}{\eta \varepsilon}
+ \frac{\sigma}{\sqrt{K_0-1}} + \frac{2 \rho R^2}{\varepsilon}\\
\leq & 
 \frac{\sqrt{2c_\infty}}{\eta K}\E\norm{\sum_{k=1}^{K_0} \qty(\sqrt{\*n_{k}}+\varepsilon)^{1/2}\circ(\*y_{k+1} - \*y_{k})} + \frac{4R \sqrt{c_\infty}}{\beta \eta K^2} + \frac{4\beta c_\infty^{1.5} B}{\eta \varepsilon}
+ \frac{\sigma}{\sqrt{K_0-1}} + \frac{2 \rho R^2}{\varepsilon}\\
\leq & 
 \frac{\sqrt{2c_\infty}R}{\eta K} + \frac{4R \sqrt{c_\infty}}{\beta \eta K^2} + \frac{4\beta c_\infty^{1.5} R}{\eta \varepsilon}
+ \frac{\sigma}{\sqrt{K_0-1}} + \frac{2 \rho R^2}{\varepsilon} \\
= &  \order{\frac{R}{\eta K} + \frac{\beta R}{\eta}
+ \frac{1}{\sqrt{K}} + R^2} = \order{\epsilon},
\end{aligned}
\]
where (a) is due to the independence of $\bm{\xi}_k$'s and Eq.\eqref{eq:delta-bound}, (b) comes from Propositions \ref{prop:x-y} and \ref{prop:mn-bound}:
\[
\begin{aligned}
    &\norm{\frac{\bm{\eta}_{k-1} - \bm{\eta}_{k-2}}{\bm{\eta}_{k-2}\bm{\eta}_k}\qty(\*x_{k-1} - \*y_{k-1})}
    \leq \frac{\sqrt{\*n_k}+\varepsilon}{\eta\qty(\sqrt{\*n_{k-1}}+\varepsilon)^{1/2}}
    \norm{\frac{\bm{\eta}_{k-1} - \bm{\eta}_{k-2}}{\bm{\eta}_{k-2}}}_\infty \norm{\hat{\*n}_{k-1}\circ(\*x_{k-1} - \*y_{k-1})} \\
    \leq & \frac{\qty(\sqrt{\*n_{k}}+\varepsilon)^{1/2}}{\eta}\frac{\sqrt{2}\beta^2 c_\infty}{\varepsilon} \qty(1-\frac{\sqrt{2}\beta^2 c_\infty}{\varepsilon})^{-1/2} \!\!R 
    \leq   \frac{\qty(c_\infty + \varepsilon)^{1/2}}{\eta}\frac{\sqrt{2}\beta^2 c_\infty}{\varepsilon} 
    \frac{R}{(1-\beta)^{1/4}}
    \leq \qty(\frac{1}{1-\beta})^{1/4}\frac{2\beta^2 c_\infty^{1.5} R}{\eta \varepsilon},
\end{aligned}
\]
we use the following bounds in (c):
\[
\begin{aligned}
  & \norm{ \frac{\qty(\*y_{k} -\*y_{k-1})}{\bm{\eta}_{k-1}} - \frac{\qty(\*y_{k} -\*y_{k-1})}{\bm{\eta}_{k}}} = 
  \norm{\frac{\bm{\eta}_{k} - \bm{\eta}_{k-1}}{\bm{\eta}_{k-1}\bm{\eta}_k}\qty(\*y_{k} - \*y_{k-1})} \\
  \leq & \frac{\qty(\sqrt{\*n_{k-1}}+\varepsilon)^{1/2}}{\eta}
    \norm{\frac{\bm{\eta}_{k} - \bm{\eta}_{k-1}}{\bm{\eta}_{k}}}_\infty \norm{\qty(\sqrt{\*n_{k-1}}+\varepsilon)^{1/2}\circ(\*y_{k} - \*y_{k-1})}\\
    \leq & \frac{\qty(\sqrt{\*n_{k-1}}+\varepsilon)^{1/2}}{\eta}\frac{\sqrt{2}\beta^2 c_\infty}{\varepsilon} \frac{R}{k}
    \leq  \frac{\qty(c_\infty + \varepsilon)^{1/2}}{\eta}\frac{\sqrt{2}\beta^2 c_\infty}{\varepsilon} \frac{R}{k}
    \leq \frac{2\beta^2 c_\infty^{1.5} R}{\eta \varepsilon k},
\end{aligned}
\]
(d) is implied by $K_{0}=\argmin_{\left\lfloor\frac{K}{2}\right\rfloor \leq k \leq K-1}\norm{\qty(
\sqrt{\*n_{k}}+ \varepsilon )^{1/2}  \circ \qty(\*y_{k+1} - \*y_k)}$ and restart condition:
\[
\begin{aligned}
 & \norm{\frac{\*y_{K_0}-\*y_{K_0-1}}{\bm{\eta}_{K_0}}}^2 \leq 
\frac{\sqrt{\*n_{K_0}}+\varepsilon}{\eta^2} \norm{\qty(\sqrt{\*n_{K_0}}+\varepsilon)^{1/2}\circ(\*y_{K_0}-\*y_{K_0-1})}^2\\
& \norm{\qty(\sqrt{\*n_{K_0}}+\varepsilon)^{1/2}\circ(\*y_{K_0}-\*y_{K_0-1})}^2 \leq \frac{1}{K - \left\lfloor K/2\right\rfloor} \sum_{k = \left\lfloor K/2\right\rfloor}^{K-1} \norm{\qty(\sqrt{\*n_{k}}+ \varepsilon )^{1/2}  \circ \qty(\*y_{k+1} - \*y_k)}^2 \\
& \leq \frac{1}{K - \left\lfloor K/2\right\rfloor} \sum_{k = 1}^K \norm{\qty(\sqrt{\*n_{k}}+ \varepsilon )^{1/2}  \circ \qty(\*y_{k+1} - \*y_k)}^2 \leq \frac{1}{K - \left\lfloor K/2\right\rfloor} \frac{R^2}{K} \leq \frac{2R^2}{K^2}.
\end{aligned}
\]
Finally, we have:
\[
 \E\qty(\norm{\nabla f(\Bar{\*x})}) = \E\qty(\norm{\nabla h(\Bar{\*x})}) + \E\qty(\norm{\nabla f(\Bar{\*x}) - \nabla h(\Bar{\*x})})= \order{\epsilon} + \frac{2\rho R^2}{{\varepsilon}} = \order{\epsilon},
\]
where we use the results from Eq.\eqref{eq:delta-bound}, namely:
\[
\E\qty(\norm{\nabla f(\Bar{\*x}) - \nabla h(\Bar{\*x})})  = \E\qty(\norm{\nabla f(\Bar{\*x}) - {\*g}^{full}_0 - \*H\qty(\Bar{\*x} - {\*x}_0)})
 \leq  \frac{\rho}{2}\E\qty(\norm{\Bar{\*x}-\*x_0}^2),
\]
and we also note that, by Eq.\eqref{eq:xk-x0}:
\[
\E\norm{\Bar{\*x}-\*x_0} \leq  \frac{1}{K_0-1}\sum_{k=1}^{K_0}  \E\norm{{\*x_k}-\*x_0} \leq \frac{2R}{\varepsilon^{1/2}}.
\]
\end{proof}
\subsubsection{Proof for Main Theorem}
\begin{theorem}\label{thm:speed3d5}
Suppose that Assumptions \ref{asm:Lsmooth}-\ref{asm:rhosmooth} hold. Let $B = \order{\epsilon^{0.5}}$, $\beta  = \order{\epsilon^2} $, $\eta =  \order{\epsilon^{1.5}}$, $\$K \leq K = \order{\epsilon^{-2}}$. Then Algorithm \ref{alg:Name} find an $\epsilon$-approximate first-order stationary point within at most $\order{\epsilon^{-3.5}}$ iterations. Namely, we have:
\[
 \E\qty(f(\*y_{\$K}) - f(\*x_0)) = - \order{\epsilon^{1.5}}, 
 \quad
 \E\qty(\norm{\nabla f(\Bar{\*x})}) = \order{\epsilon}.
\]
\end{theorem}
\begin{proof}
Note that at the beginning of each restart cycle in Algorithm \ref{alg:Name_ref}, we set $\*x_0$ to be the last iterate $\*x_{\$K}$ in the previous restart cycle. 
Due to Lemma \ref{lem:one-cycle}, we already have:
\[
\E\qty(f(\*y_{\$K}) - f(\*x_0)) = - \order{\epsilon^{1.5}}.
\]
Summing this inequality over all cycles, say $N$ total restart cycles, we have:
\[
\min_{\*x} f(\*x) - f(\*x_{\operatorname{init}}) = - \order{N \epsilon^{1.5}},
\]
Hence, the Algorithm \ref{alg:Name_ref} terminates within at most $\order{\epsilon^{-1.5}\Delta_f }$ restart cycles, where $\Delta_f \coloneqq f(\*x_{\operatorname{init}}) - \min_{\*x} f(\*x)$. 
Note that each cycle contain at most $K = \order{\epsilon^{-2}}$ iteration step, therefore, the total iteration number must be less than $\order{\epsilon^{-3.5}\Delta_f }$.
\par
On the other hand, by Lemma \ref{lem:last-cycle}, in the last restart cycle, we have:
\[
\E\qty(\norm{\nabla f(\Bar{\*x})}) = \order{\epsilon}.
\]
Now, we obtain the final conclusion for the theorem.
\end{proof}

\subsection{Auxiliary Lemmas}\label{sec:auxiliary}

\begin{prop} \label{prop:mn-bound}
If Assumption \ref{asm:boundVar} holds. We have:
\[
\norm{\*m_k}_\infty \leq c_\infty,\quad \norm{\*n_k}_\infty \leq c_\infty^2.
\]
\end{prop}
\begin{proof}
By the definition of $\*m_k$, we can have that:
\[
\*m_k = \sum_{t = 0}^k c_{k,t} {\*g}_t,
\]
where
\[
c_{k,t} = \begin{cases} \beta_1 \qty(1-\beta_1)^{(k-t)}  & \text {when } t>0, \\\\  \qty(1-\beta_1)^{k}  & \text {when } t=0.\end{cases}
\]
Similar, we also have:
\[
\*n_k = \sum_{t = 0}^k c^\prime_{k,t} \qty({\*g}_t + \qty(1-\beta_2) \qty({\*g}_t - {\*g}_{t-1}))^2,
\]
where
\[
c^\prime_{k,t} = \begin{cases} \beta_3 \qty(1-\beta_3)^{(k-t)}  & \text {when } t>0, \\\\  \qty(1-\beta_3)^{k}  & \text {when } t=0 .\end{cases}
\]
If is obvious that:
\[
\sum_{t=0}^k c_{k,t} = 1,\quad \sum_{t=0}^k c^\prime_{k,t} = 1,
\]
hence, we get:
\[
\begin{aligned}
&\norm{\*m_k}_{\infty} \leq \sum_{t = 0}^k c_{k,t}\norm{{\*g}_t}_{\infty},\\ 
&\norm{\*n_k}_{\infty} \leq \sum_{t = 0}^k c^\prime_{k,t} \norm{{\*g}_t + \qty(1-\beta_2) \qty({\*g}_t - {\*g}_{t-1})}^2_{\infty} \leq c_\infty^2.
\end{aligned}
\]
\end{proof}
\begin{prop}\label{prop:eta_diff}
If Assumption \ref{asm:boundVar} holds, we have:
\[
\norm{\frac{\bm{\eta}_k - \bm{\eta}_{k-1}}{\bm{\eta}_{k-1}}}_\infty \leq \frac{\sqrt{2\beta_3} c_\infty}{\varepsilon}.
\]
\end{prop}
\begin{proof}
Give any index $i \in [d]$ and the definitions of $\bm{\eta}_k$, we have:
\[
\abs{\qty(\frac{\bm{\eta}_k - \bm{\eta}_{k-1}}{\bm{\eta}_{k-1}})_i} =
\abs{\qty(\frac{\sqrt{\*n_{k-1}}+ \varepsilon }{\sqrt{\*n_{k}}+ \varepsilon })_i - 1} 
=
\abs{\qty(\frac{\sqrt{\*n_{k-1}} - \sqrt{\*n_{k}}}{\sqrt{\*n_{k}}+ \varepsilon })_i} .
\]
Note that, by the definition of $\*n_k$, we have:
\[
\begin{aligned}
 & \abs{\qty(\frac{\sqrt{\*n_{k-1}} - \sqrt{\*n_{k}}}{\sqrt{\*n_{k}}+ \varepsilon })_i} \leq 
 \abs{\qty(\frac{\sqrt{\abs{\*n_{k-1} - \*n_{k}}}}{\sqrt{\*n_{k}}+ \varepsilon })_i}\\
 = & \sqrt{\beta_3} \qty(\frac{\sqrt{\abs{\*n_{k-1} - \qty({\*g}_k+ \qty(1-\beta_2)\qty({\*g}_k - {\*g}_{k-1}))^2}}}{\sqrt{\*n_{k}} + \varepsilon})_i   \leq  \frac{\sqrt{2\beta_3} c_\infty}{\varepsilon},
\end{aligned}
\]
hence, we have:
\[
\abs{\qty(\frac{\bm{\eta}_k - \bm{\eta}_{k-1}}{\bm{\eta}_{k-1}})_i} \in \qty[0,  \frac{\sqrt{2\beta_3} c_\infty}{\varepsilon}].
\]
We finish the proof.
\end{proof}
\begin{lemma}\label{lem:mk}
Consider a moving average sequence:
\[
\*m_k = \qty(1-\beta)\*m_{k-1} + \beta {\*g}_k,
\]
where we note that:
\[
{\*g}_{k} = \E_{\bm{\zeta}}[\nabla f(\bm{\theta}_k,\bm{\zeta})] + \bm{\xi}_k,
\]
and we denote ${\*g}^{full}_{k}\coloneqq E_{\bm{\zeta}}[\nabla f(\bm{\theta}_k,\bm{\zeta})]$ for convenience.
Then we have:
\[
\^E\qty(\norm{\*m_k - \*g^{full}_k}^2) \leq \qty(1-\beta)\^E\qty(\norm{\*m_{k-1} - \*g^{full}_{k-1}}^2) + \frac{\qty(1-\beta)^2L^2}{\beta} \^E\qty(\norm{\bm{\theta}_{k-1} - \bm{\theta}_{k}}^2) + {\beta^2 \sigma^2}.
\]
\end{lemma}
\begin{proof}
Note that, we have:
\[
\begin{aligned}
  \*m_k - \*g^{full}_k = &  \qty(1-\beta)\qty(\*m_{k-1} - \*g^{full}_{k-1}) + (1-\beta)\*g^{full}_{k-1} - \*g^{full}_k +   \beta {\*g}_k\\
  =& \qty(1-\beta)\qty(\*m_{k-1} - \*g^{full}_{k-1}) + (1-\beta)\qty(\*g^{full}_{k-1} - \*g^{full}_k) + \beta \qty( {\*g}_k - \*g^{full}_k).
\end{aligned}
\]
Then, take expectation on both sides:
\[
\begin{aligned}
 & \^E\qty(\norm{\*m_k - \*g^{full}_k}^2)\\
= & \qty(1-\beta)^2\^E\qty(\norm{\*m_{k-1} - \*g^{full}_{k-1}}^2)+\qty(1-\beta)^2\^E\qty(\norm{\*g^{full}_{k-1} - \*g^{full}_{k}}^2) + {\beta^2 \sigma^2}+\\
&  2 \qty(1-\beta)^2\^E\qty(\innerprod{\*m_{k-1} - \*g^{full}_{k-1}, \*g^{full}_{k-1} - \*g^{full}_{k}}) \\
\leq & \qty(\qty(1-\beta)^2 + \qty(1-\beta)^2 a)\^E\qty(\norm{\*m_{k-1} - \*g^{full}_{k-1}}^2) + \\
& \qty( 1+\frac{1}{a})\qty(1-\beta)^2\^E\qty(\norm{\*g^{full}_{k-1} - \*g^{full}_{k}}^2) + {\beta^2 \sigma^2}\\
\overset{(a)}{\leq} & \qty(1-\beta)\^E\qty(\norm{\*m_{k-1} - \*g^{full}_{k-1}}^2) + \frac{\qty(1-\beta)^2}{\beta} \^E\qty(\norm{\*g^{full}_{k-1} - \*g^{full}_{k}}^2) + {\beta^2 \sigma^2} \\ 
\leq & \qty(1-\beta)\^E\qty(\norm{\*m_{k-1} - \*g^{full}_{k-1}}^2) + \frac{\qty(1-\beta)^2L^2}{\beta} \^E\qty(\norm{\bm{\theta}_{k-1} - \bm{\theta}_{k}}^2) + {\beta^2 \sigma^2},
\end{aligned}
\]
where for (a), we set $a = \frac{\beta}{1-\beta}$.
\end{proof}

\begin{lemma}\label{lem:vk}
Consider a moving average sequence:
\[
\*v_k = \qty(1-\beta)\*v_{k-1} + \beta \qty({\*g}_k - {\*g}_{k-1}),
\]
where we note that:
\[
{\*g}_{k} = \E_{\bm{\zeta}}[\nabla f(\bm{\theta}_k,\bm{\zeta})] + \bm{\xi}_k,
\]
and we denote ${\*g}^{full}_{k}\coloneqq E_{\bm{\zeta}}[ f(\bm{\theta}_k,\bm{\zeta})]$ for convenience.
 Then we have:
\[
\^E\qty(\norm{\*v_{k}}^2) \leq \qty(1-\beta)\^E\qty(\norm{\*v_{k-1}}^2) + 2\beta \^E\qty(\norm{{\*g}^{full}_k - {\*g}^{full}_{k-1}}^2)
+  {3\beta^2 \sigma^2}.
\]
\end{lemma}
\begin{proof}
Take expectation on both sides:
\[
\begin{aligned}
& \^E\qty(\norm{\*v_{k}}^2) =  \qty(1-\beta)^2\^E\qty(\norm{\*v_{k-1}}^2) + \beta^2 \^E\qty(\norm{{\*g}_k - {\*g}_{k-1}}^2)
+  2 \beta\qty(1-\beta)\^E\qty(\innerprod{\*v_{k-1}, {\*g}_k - {\*g}_{k-1}}) \\
\overset{(a)}{=} & \qty(1-\beta)^2\^E\qty(\norm{\*v_{k-1}}^2) + \beta^2 \^E\qty(\norm{{\*g}_k - {\*g}_{k-1}}^2)
+  2 \beta\qty(1-\beta)\^E\qty(\innerprod{\*v_{k-1}, {\*g}^{full}_k - {\*g}_{k-1}})\\
\overset{(b)}{\leq}  & \qty(1-\beta)^2\^E\qty(\norm{\*v_{k-1}}^2) + 2\beta^2 \^E\qty(\norm{{\*g}^{full}_k - {\*g}^{full}_{k-1}}^2)
+  2 \beta\qty(1-\beta)\^E\qty(\innerprod{\*v_{k-1}, {\*g}^{full}_k - {\*g}_{k-1}}) + {3\beta^2 \sigma^2}\\
\overset{(c)}{\leq} & \qty(1-\beta)^2\^E\qty(\norm{\*v_{k-1}}^2) + 2\beta^2 \^E\qty(\norm{{\*g}^{full}_k - {\*g}^{full}_{k-1}}^2)
+  2 \beta\qty(1-\beta)\^E\qty(\innerprod{\*v_{k-1}, {\*g}^{full}_k - {\*g}^{full}_{k-1}}) + {3\beta^2 \sigma^2} \\ 
\overset{(d)}{\leq} & \qty(1-\beta)\^E\qty(\norm{\*v_{k-1}}^2) + 2\beta \^E\qty(\norm{{\*g}^{full}_k - {\*g}^{full}_{k-1}}^2)
+  {3\beta^2 \sigma^2},
\end{aligned}
\]
where for (a), we utilize the independence between $
{\*g}_k$ and $\*v_{k-1}$, while for (b):
\[
\^E\qty(\norm{{\*g}_k - {\*g}_{k-1}}^2) \leq  
\^E\qty(\norm{{\*g}_k - {\*g}^{full}_{k}}^2) + 2\^E\qty(\norm{{\*g}^{full}_{k-1} - {\*g}_{k-1}}^2) + 2\^E\qty(\norm{{\*g}^{full}_k - {\*g}^{full}_{k-1}}^2),
\]
for (c), we know:
\begin{equation*}
\begin{split}
&\^E\qty(\innerprod{\*v_{k-1}, {\*g}^{full}_{k-1} - {\*g}_{k-1}}) =  \^E\qty\Big(\innerprod{\qty(1-\beta)\*v_{k-2} + \beta \qty({\*g}_{k-1} - {\*g}_{k-2}), {\*g}^{full}_{k-1} - {\*g}_{k-1}}) \\
= & \^E\qty\Big(\innerprod{\qty(1-\beta)\*v_{k-2} - \beta  {\*g}_{k-2}, {\*g}^{full}_{k-1} - {\*g}_{k-1}}) + \beta  \^E\qty\Big(\innerprod{ {\*g}_{k-1}- {\*g}^{full}_{k-1} + {\*g}^{full}_{k-1}, {\*g}^{full}_{k-1} - {\*g}_{k-1}}) \\
= & - \beta \^E\qty\Big(\norm{{\*g}^{full}_{k-1} - {\*g}_{k-1}}^2), 
\end{split}
\end{equation*}
and  thus $\^E\qty(\innerprod{\*v_{k-1}, {\*g}^{full}_{k} - {\*g}_{k-1}}) = \^E\qty(\innerprod{\*v_{k-1}, {\*g}^{full}_{k} - {\*g}^{full}_{k-1}}) - \beta \^E\qty\Big(\norm{{\*g}^{full}_{k-1} - {\*g}_{k-1}}^2)$.  
Finally, for (d), we use:
\[
2\^E\qty(\innerprod{\*v_{k-1}, {\*g}^{full}_k - {\*g}^{full}_{k-1}}) \leq \^E\qty(\norm{\*v_{k-1}}^2) + \^E\qty(\norm{{\*g}^{full}_k - {\*g}^{full}_{k-1}}^2).
\]
\end{proof}

\end{document}